\documentclass[11pt]{article}
\usepackage{amsmath}
\usepackage{amsthm}
\usepackage{multirow}
\usepackage{amssymb}
\usepackage{color}
\usepackage[english]{babel}
\usepackage{graphicx}
\usepackage{wrapfig,epsfig}
\usepackage{epstopdf}
\usepackage{url}
\usepackage{graphicx}
\usepackage{color}
\usepackage{epstopdf}
\usepackage{scrextend}
\usepackage[T1]{fontenc}
\usepackage{bbm}
\usepackage{comment}

\usepackage{babel}
\usepackage[font=small,labelfont=bf]{caption}

\usepackage{tikz}
\usepackage{hyperref}  
\hypersetup{colorlinks=true,citecolor=blue,linkcolor=blue} 
\usetikzlibrary{arrows}
\usepackage[margin=1in]{geometry}
\graphicspath{{./figs/}}

\newtheorem{theorem}{Theorem}[section]
\newtheorem{lemma}[theorem]{Lemma}
\newtheorem{definition}[theorem]{Definition}

\newtheorem{fact}[theorem]{Fact}

\DeclareMathOperator*{\E}{{\mathbb{E}}}
\newcommand{\dR}{\mathbb{R}}
\renewcommand{\hat}{\widehat}

\DeclareMathOperator{\diag}{diag}

\DeclareMathOperator{\tr}{tr}

\DeclareMathOperator{\poly}{poly}

\newcommand{\wh}{\widehat}

\renewcommand{\bar}{\overline}
\newcommand{\ov}{\overline}
\renewcommand{\hat}{\widehat}

\DeclareMathOperator{\R}{{\mathbb{R}}}
\newcommand{\bone}{{\mathbf{1}}}
\DeclareMathOperator*{\D}{{\mathcal{D}}}
\DeclareMathOperator*{\N}{{\mathcal{N}}}

\date{}
\title{Nonlinear Inductive Matrix Completion based on One-layer Neural Networks}
\author{
  Kai Zhong\thanks{Supported in part by NSF grants CCF-1320746, IIS-1546452 and CCF-1564000.} \\
  \texttt{zhongkai@ices.utexas.edu}\\
  UT-Austin
  \and
  Zhao Song\thanks{Work done while hosted by Jelani Nelson.} \\
  \texttt{zhaos@g.harvard.edu}\\
  Harvard University \& UT-Austin
  \and
  Prateek Jain \\
  \texttt{prajain@microsoft.com}\\
  Microsoft Research, India
  \and
  Inderjit S. Dhillon\thanks{Supported in part by NSF grants CCF-1320746, IIS-1546452 and CCF-1564000.} \\
  \texttt{inderjit@cs.utexas.edu}\\
  UT-Austin \& Amazon/A9
}

\makeatletter
\newcommand*{\RN}[1]{\expandafter\@slowromancap\romannumeral #1@}
\makeatother

\usepackage{lineno}

\usepackage{etoolbox}

\makeatletter

\newcommand{\define}[4][ignore]{%
  \ifstrequal{#1}{ignore}{}{
  \@namedef{thmtitle@#2}{#1}}%
  \@namedef{thm@#2}{#4}%
  \@namedef{thmtypen@#2}{lemma}%
  \newtheorem{thmtype@#2}[theorem]{#3}%
  \newtheorem*{thmtypealt@#2}{#3~\ref{#2}}%
}

\newcommand{\state}[1]{%
  \@namedef{curthm}{#1}
  \@ifundefined{thmtitle@#1}{
  \begin{thmtype@#1}
    }{
  \begin{thmtype@#1}[\@nameuse{thmtitle@#1}]
  }
    \label{#1}
    \@nameuse{thm@#1}
  \end{thmtype@#1}
  \@ifundefined{thmdone@#1}{
  \@namedef{thmdone@#1}{stated}%
  }{}
}

\newcommand{\restate}[1]{%
  \@namedef{curthm}{#1}
  \@ifundefined{thmtitle@#1}{
    \begin{thmtypealt@#1}
    }{
  \begin{thmtypealt@#1}[\@nameuse{thmtitle@#1}]
  }
    \@nameuse{thm@#1}
  \end{thmtypealt@#1}
  \@ifundefined{thmdone@#1}{
  \@namedef{thmdone@#1}{stated}%
  }{}
}

\newcommand{\thmlabel}[1]{
  \@ifundefined{thmdone@\@nameuse{curthm}}{\label{#1}
    }{\tag*{\eqref{#1}}}
}
\makeatother

\begin{document}

\begin{titlepage}
  \maketitle
  \begin{abstract}
The goal of a recommendation system is to predict the interest of a user in a given item by exploiting the existing set of ratings as well as certain user/item features. A standard approach to modeling this problem is Inductive Matrix Completion where the predicted rating is modeled as an inner product of the user and the item  features projected onto a latent space. In order to learn the parameters effectively from a small number of observed ratings, the latent space is constrained to be low-dimensional which implies that the parameter matrix is constrained to be low-rank. 
However, such bilinear modeling of the ratings can be limiting in practice and non-linear prediction functions can lead to significant improvements. A natural approach to introducing non-linearity in the prediction function is to apply a non-linear activation function on top of the projected user/item features.  
Imposition of non-linearities further complicates an already challenging problem that has two sources of non-convexity: a) low-rank structure of the parameter matrix, and b) non-linear activation function. We show that one can still solve the non-linear Inductive Matrix Completion problem using gradient descent type methods as long as the solution is initialized well. That is, close to the optima, the optimization function is strongly convex and hence admits standard optimization techniques, at least for certain activation functions, such as Sigmoid and tanh. We also highlight the importance of the activation function and show how ReLU can behave significantly differently than say a sigmoid function. Finally, we apply our proposed technique to recommendation systems and semi-supervised clustering, and show that our method can lead to much better performance than standard linear Inductive Matrix Completion methods. 

  \end{abstract}
 \thispagestyle{empty}
 \end{titlepage}


\section{Introduction}
Matrix Completion (MC) or Collaborative filtering \cite{candesr2007, gomez2016netflix} is by now a standard technique to model recommendation systems problems where a few user-item ratings are available and the goal is to predict ratings for any user-item pair. However, standard collaborative filtering suffers from two drawbacks: 1) Cold-start problem: MC can't give prediction for new users or items, 2) Missing side-information: MC cannot leverage side-information that is typically present in recommendation systems such as features for users/items. Consequently, several methods \cite{abernethy2006low,rendle2010factorization,xu2013speedup,jain2013provable} have been proposed to leverage the side information together with the ratings. Inductive matrix completion (IMC) \cite{abernethy2006low,jain2013provable} is one of the most popular methods in this class. 

IMC models the ratings as the inner product between certain linear mapping of the user/items' features, i.e., $A(x,y)=\langle U^\top x,  V^\top y\rangle$, where $A(x,y)$ is the predicted rating of user $x$ for item $y$, $x\in \dR^{d_1},y\in \dR^{d_2}$ are the feature vectors. Parameters $U\in \dR^{d_1\times k}, V\in \dR^{d_2\times k}$ ($k\leq d_1,k\leq d_2$) can typically be learned using a small number of observed ratings \cite{jain2013provable}. 

However, the bilinear structure of IMC is fairly simplistic and limiting in practice and might lead to fairly poor accuracy on real-world recommendation problems. For example, consider the Youtube recommendation system \cite{covington2016deep} that requires predictions over videos. Naturally, a linear function over the pixels of videos will lead to fairly inaccurate predictions and hence one needs to model the videos using non-linear networks. 
The survey paper by \cite{zhang2017deep} presents many more such examples, where we need to design a non-linear ratings prediction function for the input features, including \cite{lei2016comparative} for image recommendation, \cite{wang2014improving} for music recommendation and \cite{zhang2016collaborative} for recommendation systems with multiple types of inputs. 


We can introduce non-linearity in the prediction function using several standard techniques, however, if our parameterization admits too many free parameters then learning them might be challenging as the number of available user-item ratings tend to be fairly small. Instead, we use a simple non-linear extension of IMC that can control the number of parameters to be estimated. Note that IMC based prediction function can be viewed as an inner product between certain latent user-item features where the latent features are a {\em linear} map of the raw user-item features. To introduce non-linearity, we can use a non-linear mapping of the raw user-item features rather than the linear mapping used by IMC. This leads to the following general framework that we call non-linear inductive matrix completion (NIMC), 
\begin{equation}\label{eq:general_nimc}
A(x,y) = \langle \mathcal{U}(x), \mathcal{V}(y) \rangle,
\end{equation}
where $x\in \mathcal{X},y\in \mathcal{Y}$ are the feature vectors, $A(x,y)$ is their rating and $ \mathcal{U}: \mathcal{X} \rightarrow \mathcal{S}, \mathcal{V}: \mathcal{Y} \rightarrow \mathcal{S}$ are non-linear mappings from the raw feature space to the latent space. 

The above general framework reduces to standard inductive matrix completion when $\mathcal{U}, \mathcal{V}$ are linear mappings and further reduces to matrix completion when $x_i,y_j$ are unit vectors $e_i,e_j$ for $i$-th item and $j$-th user respectively.  When $[x_i,e_i]$ is used as the feature vector and $\mathcal{U}$ is restricted to be a two-block (one for $x_i$ and the other for $e_i$) diagonal matrix, then the above framework reduces to the dirtyIMC  model \cite{chiang2015matrix}. Similarly, $\mathcal{U}/\mathcal{V}$ can also be neural networks (NNs),
such as feedforward NNs \cite{si2016goal,covington2016deep}, convolutional NNs for images and recurrent NNs for speech/text. 

In this paper, we focus on a simple nonlinear activation based mapping for the user-item features. That is, we set $\mathcal{U}(x) = \phi(U^{*\top}x)$  and $\mathcal{V}(x) = \phi(V^{*\top}x)$ where $\phi$ is a nonlinear activation function $\phi$. Note that if $\phi$ is ReLU then the latent space is guaranteed to be in non-negative orthant which in itself can be a desirable property for certain recommendation problems.

Note that parameter estimation in both IMC and NIMC models is hard due to non-convexity of the corresponding optimization problem. However, for "nice" data, several strong results are known for the linear  models, such as \cite{candesr2007,jns13,ge2017no} for MC and \cite{jain2013provable, xu2013speedup, chiang2015matrix} for IMC. However, non-linearity in NIMC models adds to the complexity of an already challenging problem and has not been studied extensively, despite its popularity in practice.

In this paper, we study a simple one-layer neural network style NIMC model mentioned above. In particular, we formulate a squared-loss based optimization problem for estimating parameters $U^*$ and $V^*$. We show that under a realizable model and Gaussian input assumption, the objective function is locally strongly convex within a "reasonably large" neighborhood of the ground truth. Moreover, we show that the above strong convexity claim holds even if the number of observed ratings is nearly-linear in dimension and polynomial in the conditioning of the weight matrices. In particular, for well-conditioned matrices, we can recover the underlying parameters using only $\poly\log(d_1+d_2)$ user-item ratings, which is critical for practical recommendation systems as they tend to have very few ratings available per user.  Our analysis covers popular activation functions, e.g., sigmoid and ReLU, and discuss various subtleties that arise due to the activation function. Finally we discuss how we can leverage standard tensor decomposition techniques to initialize our parameters well. We would like to stress that practitioners typically use random initialization itself, and hence results studying random initialization for NIMC model would be of significant interest. 

As mentioned above, due to non-linearity of activation function along with non-convexity of the parameter space, the existing proof techniques do not apply directly to the problem. Moreover, we have to carefully argue about both the optimization landscape as well as the sample complexity of the algorithm which is not carefully studied for neural networks. Our proof establishes some new techniques that might be of independent interest, e.g., how to handle the redundancy in the parameters for ReLU activation. 
To the best of our knowledge, this is one of the first theoretically rigorous study of neural-network based recommendation systems and will hopefully be a stepping stone for similar analysis for "deeper" neural networks based recommendation systems. We would also like to highlight that our model can be viewed as a strict generalization of a one-hidden layer neural network,  hence our result represents one of the few rigorous guarantees for models that are more powerful than one-hidden layer neural networks \cite{li2017convergence,brutzkus2017sgd,zsjbd17}. 

Finally, we apply our model on synthetic datasets and verify our theoretical analysis. Further, we compare our NIMC model with standard linear IMC on several real-world recommendation-type problems, including user-movie rating prediction, gene-disease association prediction and semi-supervised clustering. NIMC demonstrates significantly superior performance over IMC. 
\vspace{-2mm}
\subsection{Related work}
\vspace{-1mm}
{\em Collaborative filtering}: 
Our model is a non-linear version of the standard inductive matrix completion model \cite{jain2013provable}. Practically, IMC has been applied to gene-disease prediction \cite{natarajan2014inductive}, matrix sensing \cite{zhong2015efficient}, multi-label classification\cite{yu2014large}, blog recommender system \cite{shin2015tumblr}, link prediction \cite{chiang2015matrix} and semi-supervised clustering \cite{chiang2015matrix,si2016goal}. However, IMC restricts the latent space of users/items to be a linear transformation of the user/item's feature space. \cite{si2016goal} extended the model to a three-layer neural network and showed significantly better empirical performance for multi-label/multi-class classification problem and semi-supervised problems. 

Although standard IMC has linear mappings, it is still a non-convex problem due to the bilinearity $UV^\top$. To deal with this non-convex problem, \cite{jain2013provable,h14} provided recovery guarantees using alternating minimization with sample complexity linear in dimension. \cite{xu2013speedup} relaxed this problem to a nuclear-norm problem and also provided recovery guarantees. More general norms have been studied \cite{rsw16,swz17,swz17b,swz18}, e.g. weighted Frobenius norm, entry-wise $\ell_1$ norm. More recently, \cite{zhang2018fast} uses gradient-based non-convex optimization and proves a better sample complexity. \cite{chiang2015matrix} studied dirtyIMC models and showed that the sample complexity can be improved if the features are informative when compared to matrix completion. Several low-rank matrix sensing problems \cite{zhong2015efficient,ge2017no} are also closely related to IMC models where the observations are sampled only from the diagonal elements of the rating matrix. \cite{rendle2010factorization, lin2016non} introduced and studied an alternate framework for ratings prediction with side-information but the prediction function is linear in their case as well. 

{\em Neural networks}: Nonlinear activation functions play an important role in neural networks. Recently, several powerful results have been discovered for learning one-hidden-layer feedforward neural networks \cite{t17a,zsjbd17,jsa15,li2017convergence,brutzkus2017sgd,gkkt16}, convolutional neural networks \cite{bg17,zhong2017learning,du2017convolutional,du2017gradient,gkm18}. However, our problem is a strict generalization of the one-hidden layer neural network and is not covered by the above mentioned results.

{\bf Notations.}
For any function $f$, we define $\widetilde{O}(f)$ to be $f\cdot \log^{O(1)}(f)$. 
For two functions $f,g$, we use the shorthand $f\lesssim g$ (resp. $\gtrsim$) to indicate that $f\leq C g$ (resp. $\geq$) for an absolute constant $C$. We use $f\eqsim g$ to mean $cf\leq g\leq Cf$ for constants $c,C$. We use $\poly(f)$ to denote $f^{O(1)}$.

{\bf Roadmap.} We first present the formal model and the corresponding optimization problem in Section~\ref{sec:prob_form}. 
 We then present the local strong convexity and local linear convergence results in Section~\ref{sec:main_result}. 
 Finally, we demonstrate the empirical superiority of NIMC when compared to linear IMC (Section~\ref{sec:exp}).

\section{Problem Formulation}\label{sec:prob_form}
Consider a user-item recommender system, where we have $n_1$ users with feature vectors $X := \{x_i\}_{i\in[n_1]} \subseteq \mathbb{R}^{d_1}$, $n_2$ items with feature vectors $Y := \{y_j\}_{j\in [n_2]}\subseteq \mathbb{R}^{d_2}$ and a collection of partially-observed user-item ratings, ${\cal A}_{\text{obs}}=\{A(x,y)|(x,y)\in \Omega\subseteq X \times Y \}$. That is $A(x_i,y_j)$ is the rating that user $x_i$ gave for item $y_j$. 
For simplicity, we assume $x_i$'s and $y_j$'s are sampled i.i.d. from distribution $\mathcal{X}$ and $\mathcal{Y}$, respectively. Each element of the index set $\Omega$ is also sampled independently and uniformly with replacement from $S:=X\times Y$. 

In this paper, our goal is to predict the rating for {\em any} user-item pair with feature vectors $x$ and $y$, respectively.
We model the user-item ratings as: 
\begin{equation}\label{eq:base_model}
 A(x,y) =  \phi( U^{*\top} x)^\top \phi (V^{*\top} y), 
 \end{equation}
where $U^*\in \mathbb{R}^{d_1\times k}$, $V^*\in \mathbb{R}^{d_2\times k}$ and $\phi$ is a {\em non-linear} activation function. 
Under this realizable model, our goal is to {\em recover} $U^*,V^*$ from a collection of observed entries, $\{A(x,y)|(x,y)\in \Omega\}$. Without loss of generality, we set $d_1 = d_2$. Also we treat $k$ as a constant throughout the paper. Our analysis requires $U^*,V^*$ to be full column rank, so we require $k\leq d$. And w.l.o.g., we assume $\sigma_k(U^*)=\sigma_k(V^*) = 1$, i.e., the smallest singular value of both $U^*$ and $V^*$ is $1$. 

Note that this model is similar to one-hidden layer feed-forward network popular in standard classification/regression tasks. However, as there is an inner product between the output of two non-linear layers, $\phi(U^*x)$ and $\phi(V^* y)$, it cannot be modeled by a single hidden layer neural network (with same number of nodes). Also, for linear activation function, the problem reduces to inductive matrix completion \cite{abernethy2006low, jain2013provable}. 

Now, to solve for $U^*$, $V^*$, we optimize a simple squared-loss based optimization problem, i.e., 
\begin{equation}\label{eq:emp_risk}
\min_{U \in \R^{d_1 \times k}, V \in \R^{d_2 \times k} } f_{\Omega}(U, V)=\sum_{(x,y)\in \Omega} (\phi( U^{\top} x)^\top \phi (V^{\top} y) - A(x, y))^2.
\end{equation}

Naturally, the above problem is a challenging non-convex optimization problem that is strictly harder than two non-convex optimization problems which are challenging in their own right: a) the linear inductive matrix completion where non-convexity arises due to bilinearity of $U^{\top} V$, and b) the standard one-hidden layer neural network (NN).  In fact, recently a lot of research has focused on understanding various properties of both the linear inductive matrix completion problem \cite{ge2017no,jain2013provable} as well as one-hidden layer NN \cite{ge2017learning,zsjbd17}. 

In this paper, we show that despite the non-convexity of Problem~\eqref{eq:emp_risk}, it behaves as a convex optimization problem close to the optima if the data is sampled stochastically from a Gaussian distribution. This result combined with standard tensor decomposition based initialization \cite{zsjbd17, kuleshov2015tensor,jsa15} leads to a polynomial time algorithm for solving \eqref{eq:emp_risk} optimally if the data satisfies certain sampling assumptions in Theorem~\ref{thm:recovery}. Moreover, we also discuss the effect of various activation functions, especially the difference between a sigmoid activation function vs RELU activation (see Theorem~\ref{thm:sigmoid_main} and  Theorem~\ref{thm:relu_main}). 

Informally, our recovery guarantee can be stated as follows,
\begin{theorem}[Informal Recovery Guarantee]\label{thm:recovery}
Consider a recommender system with a realizable model Eq.~\eqref{eq:base_model} with sigmoid activation,
Assume the features $ \{x_i\}_{i\in[n_1]}$ and $\{y_j\}_{j\in [n_2]}$ are sampled i.i.d. from the normal distribution and the observed pairs $\Omega$ are i.i.d. sampled from $ \{x_i\}_{i\in[n_1]} \times \{y_j\}_{j\in [n_2]}$ uniformly at random. Then there exists an algorithm such that $U^*,V^*$ can be recovered to any precision $\epsilon$ with time complexity and sample complexity (refers to $n_1,n_2,|\Omega|$) polynomial in the dimension and the condition number of $U^*,V^*$, and logarithmic in $1/\epsilon$. 
\end{theorem}


\section{Main Results}\label{sec:main_result}
Our main result shows that when initialized properly, gradient-based algorithms will be guaranteed to converge to the ground truth. 
We first study the Hessian of empirical risk for different activation functions, then based on the positive-definiteness of the Hessian for smooth activations, we show local linear convergence of gradient descent. The proof sketch is provided in Appendix~\ref{sec:proof_sketch}. 
 

The positive definiteness of the Hessian {\it does not} hold for several activation functions. Here we provide some examples. {\bf Counter Example 1)} The Hessian at the ground truth for linear activation is not positive definite because for any full-rank matrix $R\in \dR^{k\times k}$, $(U^*R, V^*R^{-1})$ is also a global optimal.
{\bf Counter Example 2) }The Hessian at the ground truth for ReLU activation is not  positive definite because for any diagonal matrix $D\in \dR^{k\times k}$ with positive diagonal elements, $U^*D, V^*D^{-1}$ is also a global optimal. 
These counter examples have a common property: there is redundancy in the parameters. Surprisingly, for sigmoid and tanh, the Hessian around the ground truth is positive definite. More surprisingly, we will later show that for ReLU, if the parameter space is constrained properly, its Hessian at a given point near the ground truth can also be proved to be positive definite with high probability. 

\subsection{Local Geometry and Local Linear Convergence for Sigmoid and Tanh}
We define two natural condition numbers for the problem that captures the "hardness" of the problem: 
\begin{definition}\label{defn:cond}
Define $\lambda := \max\{\lambda(U^*), \lambda(V^*)\}$ and $\kappa := \max\{\kappa(U^*),\kappa(V^*)\}$, where 
$
\lambda(U) = \sigma_1^k(U)/(\Pi_{i=1}^k \sigma_i(U))$, $\kappa(U) = \sigma_1(U)/\sigma_k(U)$, and $\sigma_i(U)$ denotes the $i$-th singular value of $U$ with the ordering $\sigma_i \ge \sigma_{i+1}$.
\end{definition}

First we show the result for sigmoid and tanh activations.
\begin{theorem}[Positive Definiteness of Hessian for Sigmoid and Tanh] \label{thm:sigmoid_main}
Let the activation function $\phi$ in the NIMC model \eqref{eq:base_model} be sigmoid or tanh and let $\kappa, \lambda$ be as defined in Definition~\ref{defn:cond}. Then for any $t>1$ and any given $U,V$, if 
\begin{align*}
& \quad n_1 \gtrsim t \lambda^4 \kappa^2 d \log^2 d,  \quad n_2 \gtrsim  t \lambda^4 \kappa^2  d\log^2 d, \quad |\Omega| \gtrsim t \lambda^4 \kappa^2  d \log^2 d , \\
& \quad  \text{and }  \quad  \|U - U^*\| + \|V - V^*\| \lesssim  1 / ( \lambda^2\kappa ),
\end{align*}
then with probability at least $1-d^{-t}$, the smallest eigenvalue of the Hessian of Eq.~\eqref{eq:emp_risk} is lower bounded by: 
\begin{align*}
\lambda_{\min}( \nabla^2 f_{\Omega}(U,V)) \gtrsim  1 / ( \lambda^2\kappa ).
\end{align*}
\end{theorem}

{\bf Remark.} Theorem~\ref{thm:sigmoid_main} shows that, given sufficiently large number of user-items ratings and a sufficiently large number of users/items themselves, the Hessian at a point close enough to the true parameters $U^*$, $V^*$, is positive definite with high probability. The sample complexity, including $n_1,n_2$ and $|\Omega|$, have a near-linear dependency on the dimension, which matches the linear IMC analysis \cite{jain2013provable}. Strong convexity parameter as well as the sample complexity depend on the condition number of $U^*,V^*$ as defined in Definition~\ref{defn:cond}. Although we don't explicitly show the dependence on $k$, both sample complexity and the minimal eigenvalue scale as a polynomial of $k$. The proofs can be found in Appendix~\ref{sec:proof_sketch}.

As the above theorem shows the Hessian is positive definite w.h.p. for a given $U,V$ that is close to the optima. This result along with smoothness of the activation function implies linear convergence of gradient descent that samples a fresh batch of samples in each iteration as shown in the following, whose proof is postponed to Appendix~\ref{app:linear_gd}.
\begin{theorem}\label{thm:grad_converge}
Let $[U^c,V^c]$ be the parameters in the $c$-th iteration. Assuming $ \|U^c - U^*\| + \|V^c - V^*\| \lesssim  1 / ( \lambda^2\kappa )$,
then given a fresh sample set, $\Omega$, that is independent of $[U^c,V^c]$ and satisfies the conditions in Theorem~\ref{thm:sigmoid_main}, the next iterate using one step of gradient descent, i.e.,
$ [U^{c+1},V^{c+1}] = [U^c,V^c] - \eta \nabla f_{\Omega} (U^c,V^c),$
satisfies
\begin{align*}
\|U^{c+1} - U^*\|_F^2 + \|V^{c+1}-V^*\|_F^2 \leq (1-M_l/M_u) (\|U^c - U^*\|_F^2 + \|V^c-V^*\|_F^2)
\end{align*}
with probability $1-d^{-t}$, where $\eta = \Theta(1/M_u)$ is the step size and $M_l \gtrsim 1 / ( \lambda^2\kappa ) $ is the lower bound on the eigenvalues of the Hessian and $M_u \lesssim 1$ is the upper bound on the eigenvalues of the Hessian. 
\end{theorem}
{\bf Remark.} The linear convergence requires each iteration has a set of fresh samples. However, since it converges linearly to the ground-truth, we only need $\log(1/\epsilon)$ iterations, therefore the sample complexity is only logarithmic in $1/\epsilon$. This dependency is better than directly using Tensor decomposition method \cite{jsa15}, which requires $O(1/\epsilon^2)$ samples. Note that we only use Tensor decomposition to initialize the parameters. Therefore the sample complexity required in our tensor initialization does not depend on $\epsilon$. 
\subsection{Empirical Hessian around the Ground Truth for ReLU}
We now present our result for ReLU activation. As we see in Counter Example 2, without any further modification, the Hessian for ReLU is not locally strongly convex due to the redundancy in parameters. Therefore, we reduce the parameter space by fixing one parameter for each $(u_i,v_i)$ pair, $i\in[k]$. In particular, we fix $u_{1,i} = u_{1,i}^*, \forall i\in[k]$ when minimizing the objective function, Eq.~\eqref{eq:emp_risk}, where $u_{1,i}$ is $i$-th element in the first row of $U$. Note that as long as $u_{1,i}^*\neq 0$, $u_{1,i}$ can be fixed to any other non-zero values. We set $u_{1,i} = u_{1,i}^*$ just for simplicity of the proof. The new objective function can be represented as {
\begin{equation}\label{eq:emp_risk_relu}
\begin{aligned}
 f_{\Omega}^{\mathrm{ReLU}} (W,V)=  \frac{1}{2|\Omega|}\sum_{(x,y)\in \Omega }  ( \phi( W^\top x_{2:d} + x_1 (u^{*(1)})^\top)^\top \phi (V^\top y) - A(x,y) )^2.
\end{aligned}
\end{equation}}
where $u^{*(1)}$ is the first row of $U^*$ and $W \in \dR^{(d-1)\times k}$.

Surprisingly, after fixing one parameter for each $(u_i,v_i)$ pair, the Hessian using ReLU is also positive definite w.h.p. for a given $(U,V) $ around the ground truth. 
\begin{theorem}[Positive Definiteness of Hessian for ReLU]\label{thm:relu_main}
Define $u_0 := \min_{i\in [k]}\{|u_{1,i}^{*}|\}$. 
For any $t>1$ and any given $U,V$, if 
\begin{align*}
& \quad n_1 \gtrsim u_0^{-4} t \lambda^4 \kappa^{12} d \log^2 d ,  \quad n_2 \gtrsim u_0^{-4} t\lambda^4  \kappa^{12}  d\log^2 d,  \quad |\Omega| \gtrsim u_0^{-4} t \lambda^4  \kappa^{12} d \log^2 d , \\
& \quad \|W - W^*\| + \|V - V^*\| \lesssim  u_0^4/ \lambda^4\kappa^{12},
\end{align*}
then with probability $1-d^{-t}$, the minimal eigenvalue of the objective for ReLU activation function, Eq.~\eqref{eq:emp_risk_relu}, is lower bounded,
\begin{align*}
\lambda_{\min}( \nabla^2 f_{\Omega}^{\mathrm{ReLU}}(W,V)) \gtrsim  u_0^2 / \lambda^2\kappa^4.
\end{align*}
\end{theorem}
{\bf Remark.} Similar to the sigmoid/tanh case, the sample complexity for ReLU case also has a linear dependency on the dimension. However, here we have a worse dependency on the condition number of the weight matrices.  The scale of $u_0$ can also be important and in practice one needs to set it carefully. Note that although the activation function is not smooth, the Hessian at a given point can still exist with probability $1$, since ReLU is smooth almost everywhere and there are only a finite number of samples. However, owing to the non-smoothness, a proof of convergence of gradient descent method for ReLU is still an open problem.

\subsection{Initialization}\label{sec:initialization}
To achieve the ground truth, our algorithm needs a good initialization method that can initialize the parameters to fall into the neighborhood of the ground truth. Here we show that this is possible by using tensor method under the Gaussian assumption. 

 In the following, we consider estimating $U^*$. Estimating $V^*$ is similar. 
 
Define $M_3 :=  ~ \mathbb{E} [A(x,y) \cdot ( x^{\otimes 3} - x \widetilde{\otimes} I)]$, where $ x \widetilde{\otimes} I := \sum_{j=1}^d [ x\otimes e_j\otimes e_j+ e_j\otimes x\otimes e_j+ e_j \otimes e_j \otimes x]$. 
Define $\gamma_j(\sigma) :=  \mathbb{E} [ \phi(\sigma\cdot z)z^j], \; \forall j=0,1,2,3$. Then 
$ M_3 = \sum_{i=1}^k \alpha_i \bar{u}_i^{*\otimes 3},$
where $\bar{u}_i^* = u_i^*/\|u_i^*\|$ and $\alpha_i = \gamma_0(\|v_i^*\|)  \left( \gamma_3 ( \| u_i^*\|)- 3\gamma_1( \|u_i^*\|)\right)$. When $\alpha_i \neq 0$, we can approximately recover $\alpha_i$ and $\bar{u}_i^*$ from the empirical version of $M_3$ using non-orthogonal tensor decomposition \cite{kuleshov2015tensor}. When $\phi$ is sigmoid, $\gamma_0(\|v_i^*\|) = 0.5$. Given $ \alpha_i$, we can estimate $\|u_i^*\|$, since $\alpha_i$ is a monotonic function w.r.t. $\|u_i^*\|$. Applying Lemma~{B.7} in \cite{zsjbd17}, we can bound the approximation error of empirical $M_3$ and population $M_3$ using polynomial number of samples. By \cite{kuleshov2015tensor}, we can bound the estimation error of $\|u_i^*\|$ and $\bar{u}_i^*$. Finally combining Theorem~\ref{thm:sigmoid_main}, we are able to show the recovery guarantee for sigmoid activation, i.e., Theorem~\ref{thm:recovery}.
%

Although tensor initialization has nice theoretical guarantees and sample complexity, it heavily depends on Gaussian assumption and realizable model assumption. In contrast, practitioners typically use  random initialization. 


\section{Experiments}\label{sec:exp}
\begin{figure*}[t]
    \hspace{-0.3cm}
    \begin{minipage}{.3\textwidth}
        \centering
\begin{tabular}{cccc}
\includegraphics[width=0.87\textwidth]{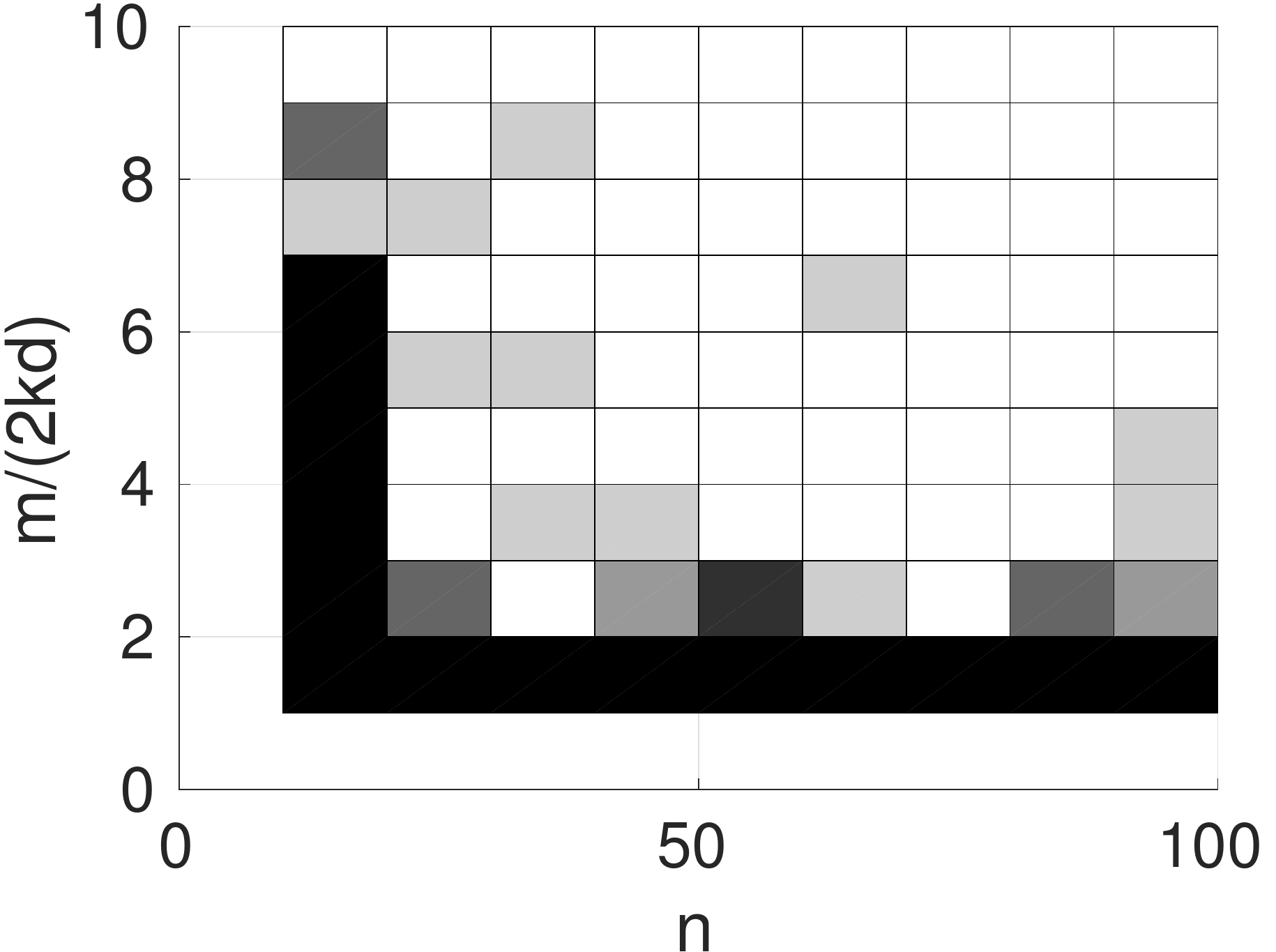}\vspace{-0.3cm}
\\
 \includegraphics[width=0.87\textwidth]{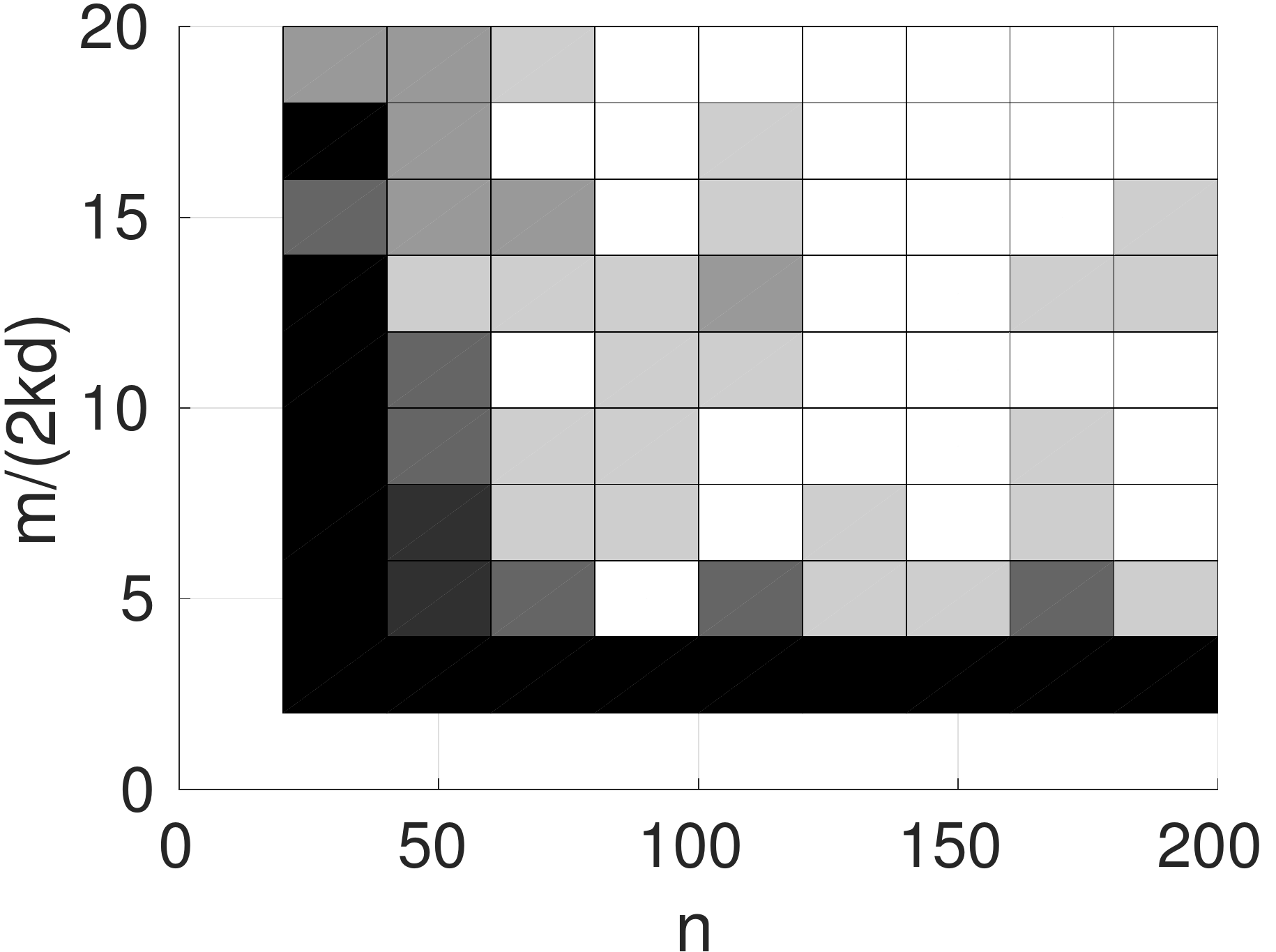}  
 \vspace{-0.3cm}
\end{tabular}
    \end{minipage}%
    \hspace{-0.3cm}
    \begin{minipage}{0.6\textwidth}
    \scalebox{0.85}{
         \begin{tabular}{|c |c | c | c | c || r| r || r |r|r| c c} 
 \hline
Dataset       &$n$ &$d$&$k$             & $|\Omega|$ & NIMC   & IMC &  \begin{tabular}{@{}c@{}}NIMC- \\ RFF\end{tabular}   &  \begin{tabular}{@{}c@{}}IMC- \\ RFF\end{tabular}  \\ \hline
\multirow{2}{*}{mushroom }& \multirow{2}{*}{8124} & \multirow{2}{*}{112} & \multirow{2}{*}{2} 
& 5$n$           &	0	&0.0049	&0	&0 \\
&&&
& 20$n$           &	0	&	0.0010	&0&	0 \\ \hline
\multirow{2}{*}{segment}& \multirow{2}{*}{2310} & \multirow{2}{*}{19} & \multirow{2}{*}{7}
 & 5$n$           &	    0.0543   & 0.0694   & 0.0197   & 0.0257  \\
&&&  
 & 20$n$           &    0.0655   & 0.0768   & 0.0092   & 0.0183   \\ \hline
\multirow{2}{*}{covtype}& \multirow{2}{*}{1708} & \multirow{2}{*}{54} & \multirow{2}{*}{7}
 & 5$n$           &	       0.1671  &  0.1733  &  0.1548  &  0.1529   \\ 
&&& 
 & 20$n$           &        0.1555   & 0.1600  &  0.1200 &   0.1307  \\
 \hline
\multirow{2}{*}{letter}& \multirow{2}{*}{15000} & \multirow{2}{*}{16} & \multirow{2}{*}{26}
 & 5$n$           &	    0.0590  &  0.0704  &  0.0422 &   0.0430 \\
&&&  
& 20$n$           &  0.0664  &  0.0760  &  0.0321  &  0.0356 \\    
 \hline 
\multirow{2}{*}{yalefaces}& \multirow{2}{*}{2452} & \multirow{2}{*}{100} & \multirow{2}{*}{38}
& 5$n$           &	    0.0315  &  0.0329 &   0.0266 &   0.0273 \\ 
&&&  
 & 20$n$           & 0.0212   & 0.0277   & 0.0064  &  0.0142 \\
\hline
\multirow{2}{*}{usps}& \multirow{2}{*}{7291} & \multirow{2}{*}{256} & \multirow{2}{*}{10}
& 5$n$          &	    0.0211  &  0.0361  &  0.0301   & 0.0185 \\
&&&  
 & 20$n$           &    0.0184  &  0.0320  &  0.0199   & 0.0152 \\
   \hline
 \end{tabular}
}

   \end{minipage}
 \caption{The left two figures (top: sigmoid, bottom: ReLU) plot rate of success of GD over synthetic data. White blocks denote $100\%$ success rate. The right table presents error in semi-supervised clustering using NIMC and IMC.}
 \label{fig:results}
\end{figure*}

In this section, we show experimental results on both synthetic data and real-world data. Our experiments on synthetic data are intended to verify our theoretical analysis, while the real-world data shows the superior performance of NIMC over IMC. We apply gradient descent with random initialization to both NIMC and IMC.  

\subsection{Synthetic Data}
We first generate some synthetic datasets to verify the sample complexity and the convergence of gradient descent using random initialization. We fix $k=5,d=10$. For sigmoid, set the number of samples $n_1=n_2=n = \{10 \cdot i\}_{i=1,2\cdots,10}$ and the number of observations $|\Omega|=m = \{2kd\cdot i\}_{i=1,2,\cdots,10}$. For ReLU,  set $n = \{20 \cdot i\}_{i=1,2\cdots,10}$ and $m = \{4kd\cdot i\}_{i=1,2,\cdots,10}$. The sampling rule follows our previous assumptions. For each $n,m$ pair, we make 5 trials and take the average of the successful recovery times.  We say a solution ($U,V$) successfully recovers the ground truth parameters when the solution achieves 0.001 relative testing error, i.e., $\|\phi(X_tU)\phi(X_tU)^\top - \phi(X_tU^*)\phi(X_tU^*)^\top \|_F \leq  0.001\cdot \|\phi(X_tU^*)\phi(X_tU^*)^\top \|_F ,$ where $X_t \in \dR^{n\times d}$ is a newly sampled testing dataset. For both ReLU and sigmoid, we minimize the original objective function~\eqref{eq:emp_risk}. 
We illustrate the recovery rate in left figures in Figure~\ref{fig:results}. As we can see, ReLU requires more samples/observations than that for sigmoid for exact recovery (note the scales of $n$ and $m/2kd$ are different in the two figures). This is consistent with our theoretical results. Comparing Theorem~\ref{thm:sigmoid_main} and Theorem~\ref{thm:relu_main}, we can see the sample complexity for ReLU has a worse dependency on the conditioning of $U^*,V^*$ than sigmoid. We can also see that when $n$ is sufficiently large, the number of observed ratings required remains the same for both methods. This is also consistent with the theorems, where $|\Omega|$ is near-linear in $d$ and is independent of $n$. 

\subsection{Semi-supervised Clustering}
We apply NIMC to semi-supervised clustering and follow the experimental setting in GIMC \cite{si2016goal}. In this problem, we are given a set of items with their features, $X\in \dR^{n\times d}$, where $n$ is the number of items and $d$ is the feature dimension, and an incomplete similarity matrix $A$, where $A_{i,j} = 1$ if $i$-th item and $j$-th item are similar and $A_{i,j} = 0$ if $i$-th item and $j$-th item are dissimilar. The goal is to do clustering using both existing features and the partially observed similarity matrix. We build the dataset from a classification dataset where the label of each item is known and will be used as the ground truth cluster. We first compute the similarity matrix from the labels and sample $|\Omega|$ entries uniformly as the observed entries. Since there is only one features we set $y_j = x_j$ in the objective function Eq.~\eqref{eq:emp_risk}. 

We initialize $U$ and $V$ to be the same Gaussian random matrix, then apply gradient descent. This guarantees $U$ and $V$ to keep identical during the optimization process. Once $U$ converges, we take the top $k$ left singular vectors of $\phi(XU)$ to do k-means clustering. The clustering error is defined as in  \cite{si2016goal}. 
Like \cite{si2016goal}, we define the clustering error as follows, 
\begin{equation*}
\text{error} = \frac{2}{n(n-1)} \left(\sum_{(i,j):\pi^*_i = \pi^*_j} 1_{\pi_i \neq \pi_j}+\sum_{(i,j):\pi^*_i \neq \pi^*_j} 1_{\pi_i = \pi_j} \right),
\end{equation*}
where $\pi^*$ is the ground-truth clustering and $\pi$ is the predicted clustering. 
We compare NIMC of a ReLU activation function with IMC on six datasets using raw features and random Fourier features (RFF). The random Fourier feature is 
$ r(x) = \frac{1}{\sqrt{q}} \cdot \begin{bmatrix} \sin(Qx)^\top & \cos(Qx)^\top \end{bmatrix}^\top \in \R^{2q}
$
and each entry of $Q \in \dR^{q\times d}$ is i.i.d. sampled from $\mathcal{N}(0,\sigma)$. We use Random Fourier features in order to see how increasing the depth of the neural network changes the performance. However, our analysis only works for one-layer neural networks, therefore, we use Random Fourier features, which can be viewed as using two-layer neural networks but with the first-layer parameters fixed.  

$\sigma$ is chosen such that a linear classifier using these random features achieves the best classification accuracy. $q$ is set as $100$ for all datasets. Datasets {\it mushroom, segment, letter,usps,covtype} are downloaded from libsvm website. We subsample covtype dataset to balance the samples from different classes. We preprocess {\it yalefaces} dataset as described in \cite{kusner2014stochastic}. As shown in the right table in Figure~\ref{fig:results}, when using raw features, NIMC achieves better clustering results than IMC for all the cases. This is also true for most cases when using Random Fourier features.
\subsection{Recommendation Systems}
Recommender systems are used in many real situations. Here we consider two tasks. 

{\bf Movie recommendation for users. }We use Movielens\cite{movielens} dataset, which has not only the ratings users give movies but also the users' demographic information and movies' genre information. Our goal is to predict ratings that new users will give the existing movies. We randomly split the users into existing users (training data) and new users (testing data) with ratio 4:1. The user features include 21 types of occupations, 7 different age ranges and one gender information; the movie features include 18-19 (18 for ml-1m and 19 for ml-100k) genre features and 20 features from the top 20 right singular values of the training rating matrix (which has size \#training users -by- \#movies). In our experiments, we set $k$ to be 50. Here are our results on datasets ml-1m and ml-100k. For NIMC, we use ReLU activation. As shown in Table~\ref{table:movielens}, NIMC achieves much smaller RMSE than IMC for both ml-100k and ml-1m datasets. 
\begin{table}
\centering
 \begin{tabular}{|c |c | c | c | c | c| c | c |c|c| c c} 
 \hline
Dataset       &\#movies  &  \#users &  \# ratings   &\# movie feat. & \# user feat.&  \begin{tabular}{@{}c@{}}RMSE \\ NIMC \end{tabular}  & \begin{tabular}{@{}c@{}}RMSE \\ IMC \end{tabular}  \\ \hline
ml-100k & 1682 & 943 & 100,000 & 39 & 29 & {\bf 1.034} & 1.321 \\   \hline
ml-1m & 3883 &  6040 & 1,000,000 & 38 & 29 & {\bf 1.021} &  1.320  \\   \hline
 \end{tabular}
 \caption{Test RMSE for recommending new users with movies on Movielens dataset. }
 \label{table:movielens}
\end{table}

\begin{figure}[t]
\begin{tabular}{cccc} \hspace{-0.32cm}
\includegraphics[width=0.24\textwidth]{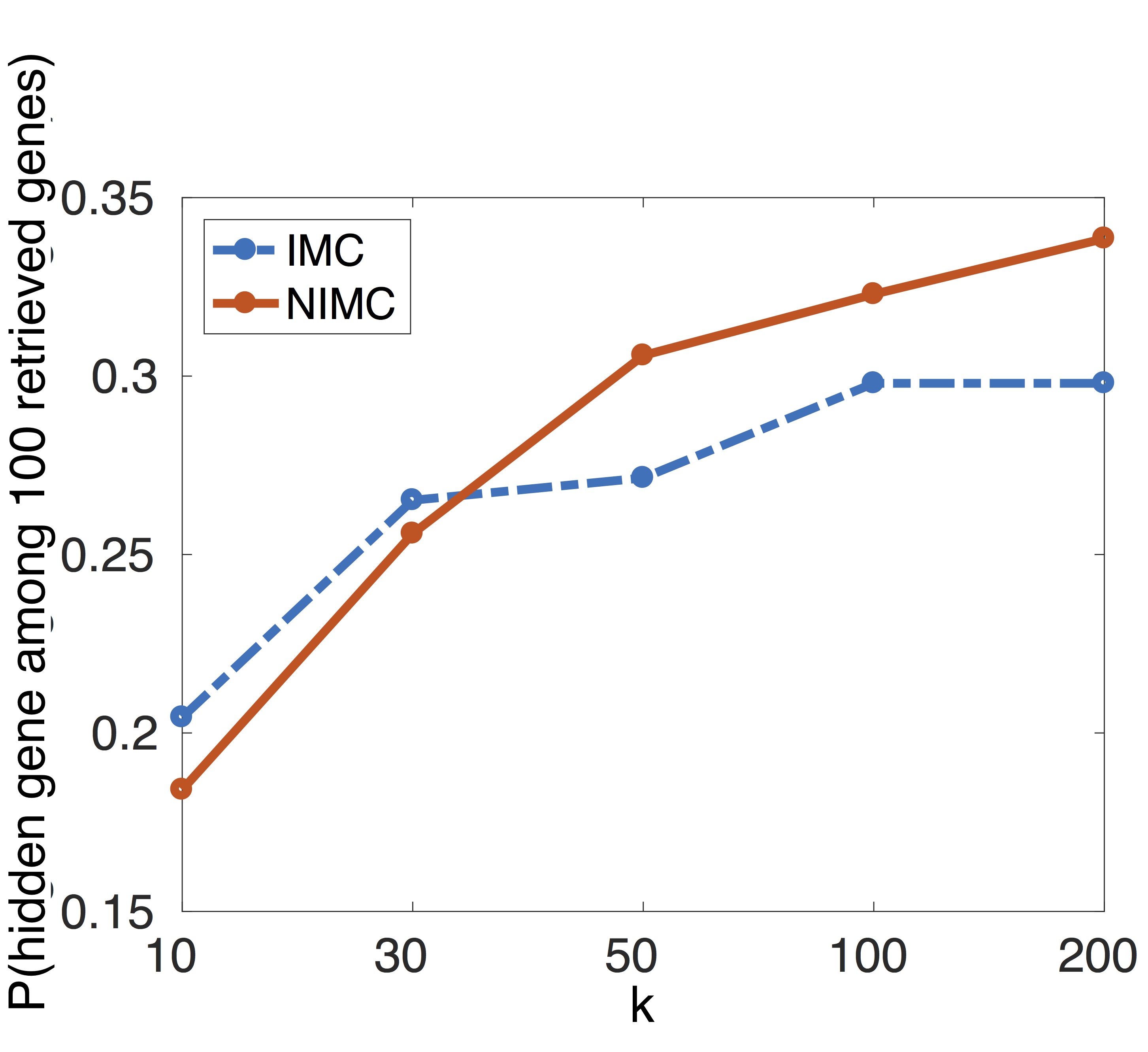} &   \hspace{-0.32cm}\includegraphics[width=0.24\textwidth]{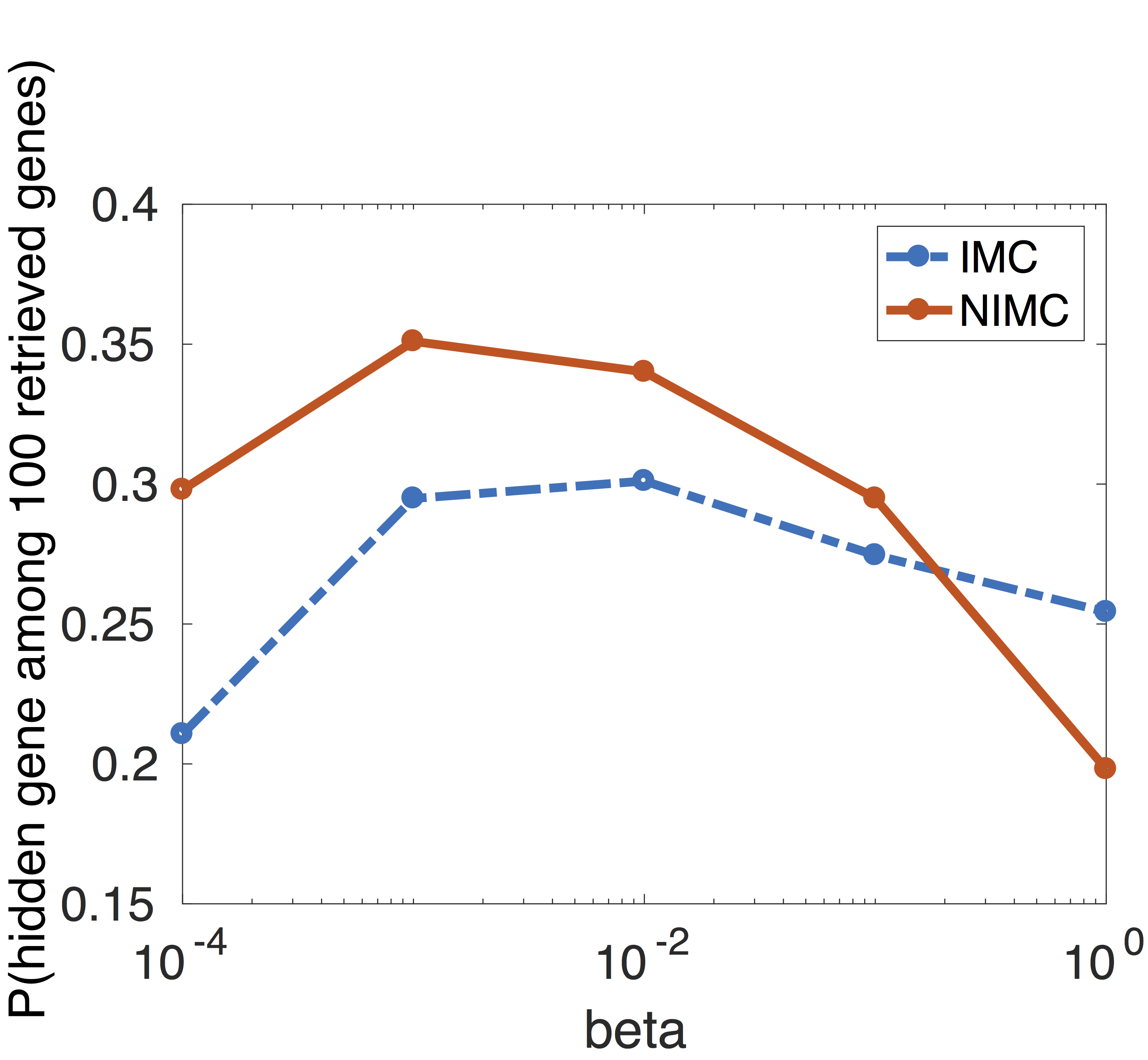}  &  \hspace{-0.32cm}  \includegraphics[width=0.24\textwidth]{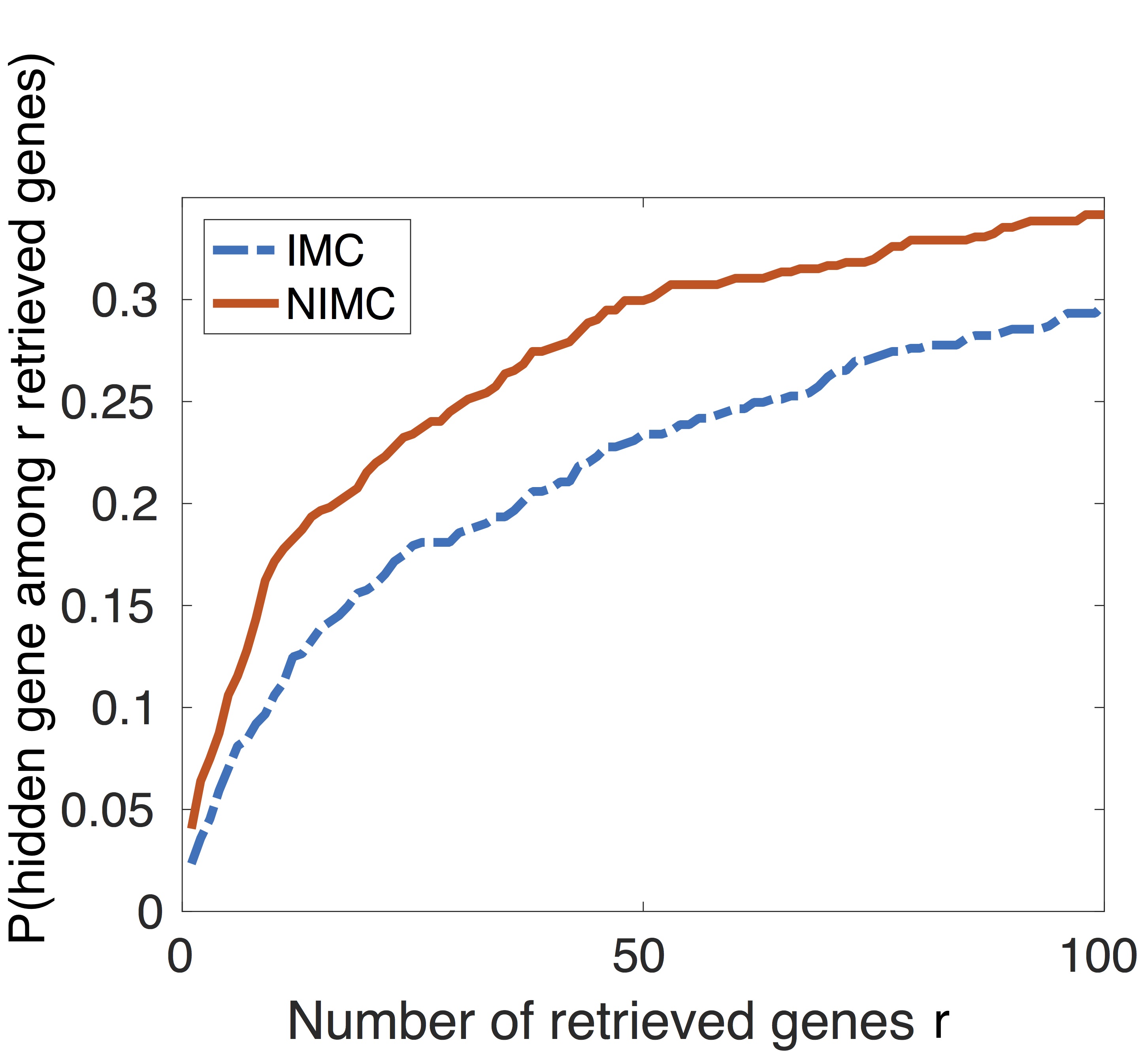}  &  \hspace{-0.32cm}\includegraphics[width=0.24\textwidth]{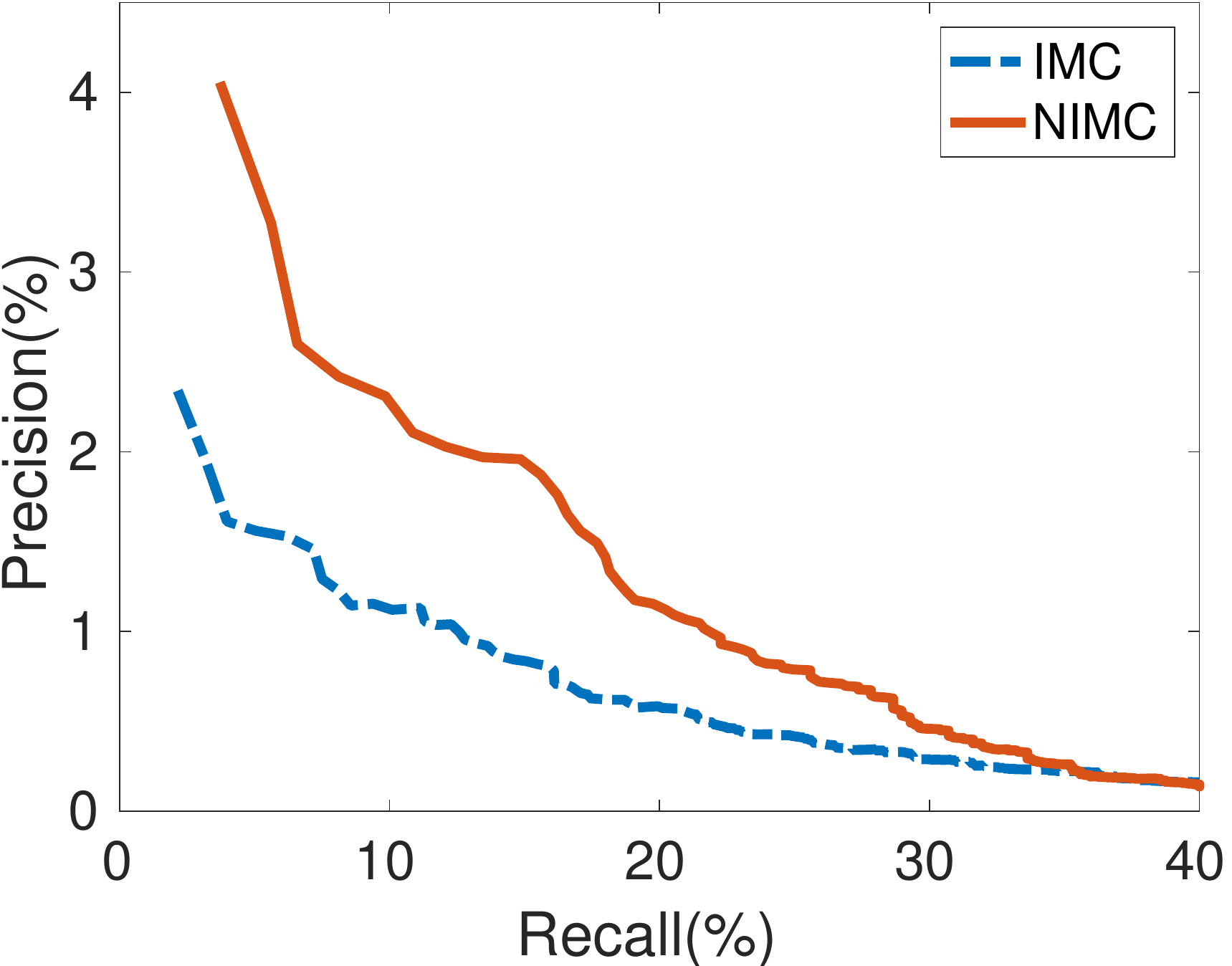}  \\
(a) & (b) & (c) & (d) 
\end{tabular}
\caption{NIMC v.s. IMC on gene-disease association prediction task. }
\label{fig:bio}
\end{figure}

{\bf Gene-Disease association prediction. }
We use the dataset collected by \cite{natarajan2014inductive}, which has 300 gene features and 200 disease features. Our goal is to predict associated genes for a new disease given its features. Since the dataset only contains positive labels, this is a problem called positive-unlabeled learning \cite{hsieh2015pu} or one-class matrix factorization \cite{yu2017unified}. We adapt our objective function to the following objective,
\begin{equation}\label{eq:pu_obj}
f(U,V) =\frac{1}{2}\left( \sum_{(i,j)\in\Omega}(\phi(U^\top x_i)^\top \phi(V^\top y_j) - A_{ij})^2 + \beta \sum_{(i,j)\in\Omega^c}(\phi(U^\top x_i)^\top \phi(V^\top y_j))^2 \right),
\end{equation}
where $A$ is the association matrix, $\Omega$ is the set of indices for observed associations, $\Omega^c$ is the complementary set of $\Omega$ and $\beta$ is the penalty weight for unobserved associations. There are totally 12331 genes and 3209 diseases in the dataset. We randomly split the diseases into training diseases and testing diseases with ratio 4:1. The results are presented in Fig~\ref{fig:bio}. We follow \cite{natarajan2014inductive} and use the cumulative distribution of the ranks as a measure for comparing the performances of different methods, i.e., the probability that any ground-truth associated gene of a disease appears in the retrieved  top-$r$ genes for this disease. 

In Fig~\ref{fig:bio}(a), we show how $k$ changes the performance of NIMC and IMC. In general, the higher $k$, the better the performance. The performance of IMC becomes stable when $k$ is larger than 100, while the performance of NIMC is still increasing. Although IMC performs better than NIMC when $k$ is small, the performance of NIMC increases much faster than IMC when $k$ increases. $\beta$ is fixed as $0.01$ and $r=100$ in the experiment for Fig~\ref{fig:bio}(a).  In Fig.~\ref{fig:bio}(b), we present how $\beta$ in Eq.~\eqref{eq:pu_obj} affects the performance. We tried over $\beta = [10^{-4},10^{-3},10^{-2},10^{-1},1]$ to check how the value of $\beta$ changes the performance. As we can see, $\beta=10^{-3}$ and $10^{-2}$ give the best results. Fig.~\ref{fig:bio}(c) shows the probability that any ground-truth associated gene of a disease appears in the retrieved  top-$r$ genes for this disease  w.r.t. different $r$'s. Here we fix $k = 200$, and $\beta=0.01$.  Fig.~\ref{fig:bio}(d) shows the precision-recall curves for different methods when $k = 200$, and $\beta=0.01$. 

\section{Conclusion}
In this paper, we studied a nonlinear IMC model that represents one of the simplest inductive model for neural-network-based recommender systems. We study  local geometry of the empirical risk function and show that, close to the optima, the function is strongly convex for both ReLU and sigmoid activations. Therefore, using a smooth activation function like sigmoid activation along with standard tensor initialization, gradient descent recovers the underlying model with polynomial sample complexity and time complexity. Thus we provide the first theoretically rigorous result for the non-linear recommendation system problem, which we hope will spur further progress in the area of deep-learning based recommendation systems. Our experimental results on synthetic data matches our analysis and the results on real-world benchmarks for semi-supervised clustering and recommendation systems show a superior performance over linear IMC. 


\clearpage
\newpage


\bibliographystyle{alpha}

\bibliography{ref}
\addcontentsline{toc}{section}{References}

\newpage
\appendix
\section*{Appendix}

\section{Notation}
For any positive integer $n$, we use $[n]$ to denote the set $\{1,2,\cdots,n\}$.
For random variable $X$, let $\mathbb{E}[X]$ denote the expectation of $X$ (if this quantity exists).

For any vector $ x\in \mathbb{R}^n$, we use $\|  x\|$ to denote its $\ell_2$ norm.

We provide several definitions related to matrix $A$.
Let $\det(A)$ denote the determinant of a square matrix $A$. Let $A^\top$ denote the transpose of $A$. Let $A^\dagger$ denote the Moore-Penrose pseudoinverse of $A$. Let $A^{-1}$ denote the inverse of a full rank square matrix. Let $\| A\|_F$ denote the Frobenius norm of matrix $A$. Let $\| A\|$ denote the spectral norm of matrix $A$. Let $\sigma_i(A)$ to denote the $i$-th largest singular value of $A$.

We use $\bone_{f}$ to denote the indicator function, which is $1$ if
$f$ holds and $0$ otherwise. Let $I_d \in \mathbb{R}^{d\times d}$
denote the identity matrix. We use $\phi(z)$ to denote an activation
function. 
We use ${\cal D}$ to denote a Gaussian distribution ${\cal N}(0,I_d)$. 
For integer $k$, we use $\D_k$ to denote $\N(0,I_k)$.

For any function $f$, we define $\widetilde{O}(f)$ to be $f\cdot \log^{O(1)}(f)$. In addition to $O(\cdot)$ notation, for two functions $f,g$, we use the shorthand $f\lesssim g$ (resp. $\gtrsim$) to indicate that $f\leq C g$ (resp. $\geq$) for an absolute constant $C$. We use $f\eqsim g$ to mean $cf\leq g\leq Cf$ for constants $c,C$.

\section{Preliminaries}
We state some useful facts in this section.
\begin{fact}\label{fac:basic_1}
Let $A = \begin{bmatrix} a_1 & a_2 & \cdots & a_k \end{bmatrix}$. Let $\diag(A) \in \R^k$ denote the vector where the $i$-th entry is $A_{i,i}$, $\forall i \in [k]$. Let ${\bf 1} \in \R^k$ denote the vector that the $i$-th entry is $1$, $\forall i \in [k]$. We have the following properties,
\begin{align*}
\mathrm{(\RN{1})} & ~ \sum_{i=1}^k (a_i^\top e_i)^2  = \| \diag(A) \|_2^2, \\
\mathrm{(\RN{2})} & ~ \sum_{i=1}^k (a_i^\top a_i)^2  = \| A \|_F^2, \\
\mathrm{(\RN{3})} & ~ \sum_{i=1}^k \sum_{j=1}^k (a_i^\top a_j) = \| A \cdot {\bf 1} \|_2^2, \\
\mathrm{(\RN{4})} & ~\sum_{i\neq j} a_i^\top a_j = \| A \cdot {\bf 1} \|_2^2 - \| A \|_F^2.
\end{align*}
\end{fact}

\begin{proof}
Using the definition, it is easy to see that (\RN{1}), (\RN{2}) and (\RN{3}) are holding.

Proof of (\RN{4}), we have
\begin{align*}
\sum_{i\neq j} a_i^\top a_j 
= \sum_{i,j} a_i^\top a_j  - \sum_{i=1}^k a_i^\top a_i 
= \| A \cdot {\bf 1} \|_2^2 - \| A \|_F^2.
\end{align*}
where the last step follows by (\RN{2}) and (\RN{3}).
\end{proof}

\begin{fact}\label{fac:basic_2}
Let $A = \begin{bmatrix} a_1 & a_2 & \cdots & a_k \end{bmatrix}$. Let $\diag(A) \in \R^k$ denote the vector where the $i$-th entry is $A_{i,i}$, $\forall i \in [k]$. Let ${\bf 1} \in \R^k$ denote the vector that the $i$-th entry is $1$, $\forall i \in [k]$. We have the following properties,
\begin{align*}
\mathrm{(\RN{1})} & ~ \sum_{i \neq j} a_i^\top e_i e_i^\top a_j = (\diag(A)^\top \cdot (A \cdot {\bf 1}) ) -  \| \diag(A) \|_2^2, \\
\mathrm{(\RN{2})} & ~  \sum_{i \neq j} a_i^\top e_j e_j^\top a_j = (\diag(A)^\top \cdot (A \cdot {\bf 1}) ) -  \| \diag(A) \|_2^2, \\
\mathrm{(\RN{3})} & ~ \sum_{i \neq j} a_i^\top e_i a_j^\top e_j = (\diag(A)^\top \cdot {\bf 1} )^2 - \| \diag(A) \|_2^2, \\
\mathrm{(\RN{4})} & ~  \sum_{i \neq j} a_i^\top e_j a_j^\top e_i = \langle A^\top, A \rangle - \| \diag(A) \|_2^2.
\end{align*}
\end{fact}
\begin{proof}
Proof of (\RN{1}). We have
\begin{align*}
\sum_{i \neq j} a_i^\top e_i e_i^\top a_j  
= & ~ \sum_{i,j} a_i^\top e_i e_i^\top a_j - \sum_{i=1}^k a_i^\top e_i e_i^\top a_i \\
= & ~ \sum_{i,j} a_{i,i} e_i^\top a_j - \| \diag(A) \|_2^2 \\
= & ~ \sum_{i=1}^k a_{i,i} e_i^\top  \sum_{j=1}^k a_j -  \| \diag(A) \|_2^2 \\
= & ~ (\diag(A)^\top \cdot (A \cdot {\bf 1}) ) -  \| \diag(A) \|_2^2
\end{align*}
Proof of (\RN{2}). It is similar to (\RN{1}).

Proof of (\RN{3}). We have
\begin{align*}
\sum_{i \neq j} a_i^\top e_i a_j^\top e_j 
= & ~ \sum_{i , j} a_i^\top e_i a_j^\top e_j - \sum_{i =1} a_i^\top e_i a_i^\top e_i \\
= & ~ \sum_{i=1}^k a_i^\top e_i \cdot \sum_{j=1}^k a_j^\top e_j - \sum_{i = 1}^k a_i^\top e_i a_i^\top e_i \\
= & ~ \sum_{i=1}^k a_{i,i} \cdot \sum_{j=1}^k a_{j,j} - \sum_{i=1}^k a_{i,i} a_{i,i} \\
= & ~ (\diag(A)^\top \cdot {\bf 1} )^2 - \| \diag(A) \|_2^2
\end{align*}

Proof of (\RN{4}). We have
\begin{align*}
\sum_{i \neq j} a_i^\top e_j a_j^\top e_i  
= & ~ \sum_{i \neq j} \tr[a_i^\top e_j a_j^\top e_i ] \\
= & ~ \sum_{i \neq j} \tr[e_j a_j^\top e_i a_i^\top ] \\
= & ~ \sum_{i \neq j} \langle e_j a_j^\top, a_i e_i^\top \rangle \\
= & ~ \sum_{i , j} \langle e_j a_j^\top, a_i e_i^\top \rangle - \sum_{i=1}^k \langle e_i a_i^\top, a_i e_i^\top \rangle  \\
= & ~ \langle A^\top, A \rangle - \| \diag(A) \|_2^2.
\end{align*}
where the second step follows by $\tr[ABCD] = \tr[BCDA]$, the third step follows by $\tr[AB] = \langle A , B^\top \rangle$.
\end{proof}

\section{Proof Sketch}\label{sec:proof_sketch}

At high level the proofs for Theorem~\ref{thm:sigmoid_main} and Theorem \ref{thm:relu_main} include the following steps. 1) Show that the population Hessian at the ground truth is positive definite. 2) Show that population Hessians near the ground truth are also positive definite. 3) Employ matrix Bernstein inequality to bound the population Hessian and the empirical Hessian. 

We now formulate the Hessian. The Hessian of Eq.~\eqref{eq:emp_risk}, $\nabla^2 f_{\Omega}(U,V) \in \dR^{(2kd)\times (2kd)}$, can be decomposed into two types of blocks, ($i\in [k], j\in[k]$), 
 $$\frac{ \partial^2 f_\Omega(U,V) }{\partial u_i \partial v_j}, \frac{ \partial^2 f_\Omega (U,V) }{\partial u_i \partial u_j},$$
where $u_i$($v_j$, resp.) is the $i$-th column of $U$ ($j$-th column of $V$, resp.).  Note that each of the above second-order derivatives is a $d\times d$ matrix.

The first type of blocks are given by: {
\begin{align*}
\frac{ \partial^2 f_\Omega(U,V) }{\partial u_i \partial v_j}  =  \underset{\Omega}{ \widehat{\E}} \left[ \phi'(u_i^{\top} x) \phi'(v_j^{\top} y) xy^\top \phi(v_i^{\top} y )\phi(u_j^{\top} x) \right] +  \delta_{ij} \underset{\Omega}{ \widehat{\E}} \left[h_{x,y}(U,V) \phi'(u_i^\top x) \phi'(v_i^\top y)xy^\top   \right], 
\end{align*}}
where $\widehat \E_{\Omega}[\cdot] = \frac{1}{|\Omega|} \sum_{(x,y)\in \Omega}[\cdot]$, $\delta_{ij}=1_{i=j}$, and
$$h_{x,y}(U,V) = \phi(U^\top x)^\top \phi(V^\top y) - \phi(U^{*\top} x)^\top \phi(V^{*\top} y)  . $$

For sigmoid/tanh activation function, the second type of blocks are given by: 
{
\begin{align}\label{eq:xx_offdiag_sigmoid}
 \frac{ \partial^2 f_\Omega (U,V) }{\partial u_i \partial u_j}  
=  \underset{\Omega}{ \widehat{\E}} \left[ \phi'(u_i^{\top} x) \phi'(u_j^{\top} x) xx^\top \phi(v_i^{\top} y )\phi(v_j^{\top} y) \right] + \delta_{ij} \underset{\Omega}{ \widehat{\E}} \left[h_{x,y}(U,V) \phi''(u_i^\top x) \phi(v_i^\top y)xx^\top   \right] .
\end{align}
}
For ReLU/leaky ReLU activation function, the second type of blocks are given by: 
\begin{align*}
\frac{ \partial^2 f_\Omega (U,V) }{\partial u_i \partial u_j}  =  \underset{\Omega}{ \widehat{\E}} \left[ \phi'(u_i^{\top} x) \phi'(u_j^{\top} x) xx^\top \phi(v_i^{\top} y )\phi(v_j^{\top} y) \right] .
\end{align*}
Note that the second term of Eq.~\eqref{eq:xx_offdiag_sigmoid} is missing here as $(U,V)$ are fixed, the number of samples is finite and $\phi''(z) = 0$ almost everywhere.

In this section, we will discuss important lemmas/theorems for Step 1 in Appendix~\ref{sec:pd_pop_hessian} and Step 2,3 in Appendix~\ref{sec:emp_hessian}. 

\subsection{Positive definiteness of the population hessian}\label{sec:pd_pop_hessian}
The corresponding population risk for  Eq.~\eqref{eq:emp_risk} is given by: 
{
\begin{equation}\label{eq:pop_risk}
f_{\mathcal{D}} (U,V)=  \frac{1}{2}   \E_{(x,y) \sim \mathcal{D}}  [( \phi( U^\top x)^\top \phi (V^\top y) - A(x, y) )^2], 
\end{equation}}
where $\mathcal{D} :=  \mathcal{X}\times \mathcal{Y}$. 
For simplicity, we also assume $\mathcal{X}$ and $\mathcal{Y}$ are normal distributions. 

Now we study the Hessian of the population risk at the ground truth. 
Let the Hessian of $ f_\mathcal{D} (U,V)$ at the ground-truth $(U,V) = (U^*,V^*)$ be $H^* \in \dR^{(2dk) \times (2dk)}$, which can be decomposed into the following two types of blocks ($i\in [k], j\in[k]$),
{
\begin{align*}
\frac{ \partial^2 f_\mathcal{D} (U^*,V^*) }{\partial u_i \partial u_j}  = &   \E_{x,y} \left[ \phi'(u_i^{*\top} x) \phi'(u_j^{*\top} x) xx^\top \phi(v_i^{*\top} y )\phi(v_j^{^*\top} y) \right], \\ 
\frac{ \partial^2 f_\mathcal{D} (U^*,V^*) }{\partial u_i \partial v_j}  = &   \E_{x,y} \left[ \phi'(u_i^{*\top} x) \phi'(v_j^{*\top} y) xy^\top \phi(v_i^{*\top} y )\phi(u_j^{^*\top} x) \right]. 
\end{align*}}

To study the positive definiteness of $H^*$, we characterize the minimal eigenvalue of $H^*$ by a constrained optimization problem,
\begin{align}\label{eq:hessian_min_eig}
\lambda_{\min}(H^*) =  \min_{ (a,b) \in \mathbb{B}}  \E_{x,y} \left[ \left( \sum_{i=1}^k \phi'(u_i^{*\top} x ) \phi(v_i^{*\top} y ) x^\top a_i + \phi'(v_i^{*\top} y )\phi(u_i^{*\top} x ) y^\top b_i \right)^2 \right],
\end{align}
where $(a,b) \in \mathbb{B}$ denotes that $\sum_{i=1}^k \|a_i\|^2 + \|b_i\|^2 = 1$.
Obviously, $\lambda_{\min}(H^*) \geq 0$ due to the squared loss and the realizable assumption. However, this is not sufficient for the local convexity around the ground truth, which requires the positive (semi-)definiteness for the {\it neighborhood} around the ground truth. In other words, we need to show that $\lambda_{\min}(H^*)$ is strictly greater than $0$, so that we can characterize an area in which the Hessian still preserves positive definiteness (PD) despite the deviation from the ground truth. 

{\bf Challenges.} As we mentioned previously there are activation functions that lead to redundancy in parameters. Hence one challenge is to distill properties of the activation functions that preserve the PD. Another challenge is the correlation introduced by $U^*$ when it is non-orthogonal. So we first study the minimal eigenvalue for orthogonal $U^*$ and orthogonal $V^*$ and then link the non-orthogonal case to the orthogonal case. 

\subsection{Warm up: orthogonal case}
In this section, we consider the case when $U^*,V^*$ are unitary matrices, i.e., $U^{*\top}U^* = U^*U^{*\top} = I_d$. ($d=k$). This case is easier to analyze because the dependency between different elements of $x$ or $y$ can be disentangled. And we are able to provide lower bound for the Hessian. Before we introduce the lower bound, let's first define the following quantities for an activation function $\phi$. 
\begin{equation}\label{eq:moments_def}
\begin{aligned}
 \alpha_{i,j} := &\E_{z\sim \mathcal{N}(0,1)} [(\phi(z))^i z^j],\\
    \;\beta_{i,j} := & \E_{z\sim \mathcal{N}(0,1)}[(\phi'(z))^i z^j] , \\
 \gamma := &\E_{z\sim \mathcal{N}(0,1)} [\phi(z)\phi'(z) z], \\
\rho := &  \min\{ ( \alpha_{2,0}\beta_{2,0} - \alpha_{1,0}^2 \beta_{1,0}^2 - \beta_{1,0}^2 \alpha_{1,1}^2),~  (\alpha_{2,0}\beta_{2,2} - \alpha_{1,0}^2 \beta_{1,2}^2 -  \gamma^2)\}.
\end{aligned}
\end{equation}

We now present a lower bound for general activation functions including sigmoid and tanh. 
\begin{lemma}\label{lemma:ortho_min_eig_informal} 
Let $(a,b) \in \mathbb{B}$ denote that $\sum_{i=1}^k \|a_i\|^2 + \|b_i\|^2 = 1$.
Assume  $d=k$ and $U^*,V^*$ are unitary matrices, i.e., $U^{*\top}U^* = U^*U^{*\top} = V^*V^{*\top} =V^{*\top}V^{*} =I_d$, then the minimal eigenvalue of the population Hessian in Eq.~\eqref{eq:hessian_min_eig} can be simplified as,{
\begin{equation*}
\min_{(a,b)\in \mathbb{B}}  \E_{x,y} \left[ \left( \sum_{i=1}^k \phi'(x_i ) \phi(y_i ) x^\top a_i + \phi'(y_i )\phi(x_i ) y^\top b_i \right)^2 \right].
\end{equation*}}
Let $\beta,\rho$ be defined as in Eq.~\eqref{eq:moments_def}. If the activation function $\phi$ satisfies $\beta_{1,1} = 0$, then
$\lambda_{\min}(H^*) \geq \rho.$
\end{lemma}
Since sigmoid and tanh have symmetric derivatives w.r.t. $0$, they satisfy $\beta_{1,1} = 0$. Specifically, we have $\rho \approx 0.000658$ for sigmoid and $\rho \approx 0.0095$ for tanh. Also for ReLU, $\beta_{1,1} = 1/2$, so ReLU does not fit in this lemma. The full proof of Lemma~\ref{lemma:ortho_min_eig_informal}, the lower bound of the population Hessian for ReLU and the extension to non-orthogonal cases can be found in  Appendix~\ref{app:hessian}.

\subsection{Error bound for the empirical Hessian near the ground truth}\label{sec:emp_hessian}
In the previous section, we have shown PD for the population Hessian at the ground truth for the orthogonal cases. Based on that, we can characterize the landscape around the ground truth for the empirical risk. In particular, we bound the difference between the empirical Hessian near the ground truth and the population Hessian at the ground truth. The theorem below provides the error bound w.r.t. the number of samples $(n1,n2)$ and the number of observations $|\Omega|$ for both sigmoid and ReLU activation functions.

\begin{theorem}\label{thm:empirical_error_bound}
For any $\epsilon>0$, if 
\begin{align*}
 n_1 \gtrsim \epsilon^{-2} t d \log^2 d, n_2 \gtrsim \epsilon^{-2} t d\log^2 d, |\Omega| \gtrsim \epsilon^{-2} t d \log^2 d,
 \end{align*}
then with probability at least $1-d^{-t}$,
for sigmoid/tanh,
{
\begin{align*}
  \| \nabla^2f_\Omega(U,V) - \nabla^2f_{\mathcal{D}}(U^*,V^*)\| \lesssim   \epsilon +  \|U - U^*\| + \|V - V^*\|;
\end{align*}
}
for ReLU,
\begin{align*}
 \| \nabla^2f_\Omega(U,V) - \nabla^2f_{\mathcal{D}}(U^*,V^*)\| 
\lesssim  \left(\|V - V^*\|^{1/2}+ \|U - U^*\|^{1/2}  + \epsilon\right) (\|U^*\| + \|V^*\|)^2.
\end{align*}
\end{theorem}

The key idea to prove this theorem is to use the population Hessian at $(U,V)$ as a bridge. 

On one side, we bound the population Hessian at the ground truth and the population Hessian at $(U,V)$. This would be easy if the second derivative of the activation function is Lipschitz, which is the case of sigmoid and tanh. But ReLU doesn't have this property. However, we can utilize the condition that the parameters are close enough to the ground truth and the piece-wise linearity of ReLU to bound this term. 

On the other side, we bound the empirical Hessian and the population Hessian. A natural idea is to apply matrix Bernstein inequality. However, there are two obstacles. First the Gaussian variables are not uniformly bounded. Therefore, we instead use Lemma B.7 in \cite{zsjbd17}, which is a loosely-bounded version of matrix Bernstein inequality. The second obstacle is that each individual Hessian calculated from one observation $(x,y)$ is not independent from another observation $(x',y')$, since they may share the same feature $x$ or $y$. The analyses for vanilla IMC and MC assume all the items(users) are given and the observed entries are independently sampled from the whole matrix. However, our observations are sampled from the joint distribution of $\mathcal{X}$ and $\mathcal{Y}$. 

To handle the dependency, our model assumes the following two-stage sampling rule. First, the items/users are sampled from their distributions independently, then given the items and users, the observations $\Omega$ are sampled uniformly with replacement. The key question here is how to combine the error bounds from these two stages. 
Fortunately, we found special structures in the blocks of Hessian which enables us to separate $x,y$ for each block, and bound the errors in stage separately. See Appendix~\ref{sec:emp_pop_hessian} for details.


\section{Positive Definiteness of Population Hessian}\label{app:hessian}
\subsection{Orthogonal case}
We first study the orthogonal case, where $d=k$ and $U^*,V^*$ are unitary matrices, i.e., $U^{*\top}U^* = U^*U^{*\top} = V^*V^{*\top} =V^{*\top}V^{*} =I_d$. 
\subsubsection{Lower bound on minimum eigenvalue}

\begin{lemma}[Restatement of Lemma~\ref{lemma:ortho_min_eig_informal}]\label{lemma:ortho_min_eig_formal}
Let $(a,b) \in \mathbb{B}$ denote that $\sum_{i=1}^k \|a_i\|^2 + \|b_i\|^2 = 1$.
Assume  $d=k$ and $U^*,V^*$ are unitary matrices, i.e., $U^{*\top}U^* = U^*U^{*\top} = V^*V^{*\top} =V^{*\top}V^{*} =I_d$, then the minimal eigenvalue of the population Hessian in Eq.~\eqref{eq:hessian_min_eig} can be simplified as,
\begin{equation}\label{eq:pop_ortho_min_eig}
\lambda_{\min}(H^*) = \min_{(a,b)\in \mathbb{B}}  \E_{x,y} \left[ \left( \sum_{i=1}^k \phi'(x_i ) \phi(y_i ) x^\top a_i + \phi'(y_i )\phi(x_i ) y^\top b_i \right)^2 \right].
\end{equation}
Let $\beta,\rho$ be defined as in Eq.~\eqref{eq:moments_def}. If the activation function $\phi$ satisfies $\beta_{1,1} = 0$, then
$\lambda_{\min}(H^*) \geq \rho.$
\end{lemma}
\begin{proof}
In the orthogonal case, we can easily transform Eq.~\eqref{eq:hessian_min_eig}  to Eq.~\eqref{eq:pop_ortho_min_eig} since $x,y$ are normal distribution. Now we can decompose Eq.~\eqref{eq:pop_ortho_min_eig} into the following three terms. 
\begin{align*}
& ~ \E_{x,y} \left[ \left( \sum_{i=1}^k \phi'(x_i ) \phi(y_i ) x^\top a_i + \phi'(y_i )\phi(x_i ) y^\top b_i \right)^2 \right] \\ 
=  & ~ \underbrace{\E_{x,y} \left[ \left( \sum_{i=1}^k \phi'(x_i ) \phi(y_i ) x^\top a_i  \right)^2 \right] }_{C}+  \E_{x,y} \left[ \left(\sum_{i=1}^k \phi'(y_i )\phi(x_i ) y^\top b_i \right)^2 \right] \\
& ~  +  \underbrace{2\E_{x,y} \left[ \sum_{i,j}  \phi'(x_i ) \phi(y_i ) x^\top a_i \phi'(y_j )\phi(x_j ) y^\top b_j  \right]}_{D}.
\end{align*}
Note that the first term is similar to the second term, so we just lower bound the first term and the third term. 
Define 
$A = [a_1, a_2,\cdots ,a_k], B = [b_1,b_2,\cdots ,b_k]$. Let $A_o$ be the off-diagonal part of $A$ and $A_d$ be the diagonal part of $A$, i.e., $A_o + A_d = A$. And let $g_A = \diag(A)$ be the vector of the diagonal elements of A. We will bound $C$ and $D$ in the following. 

For $C$, we have 
\begin{align*}
& ~ \E_{x,y} \left[ \left( \sum_{i=1}^k \phi'(x_i ) \phi(y_i ) x^\top a_i  \right)^2 \right] \\
= & ~ \sum_{i=1}^k  \E_{x,y} \left[ \left(\phi'(x_i ) \phi(y_i ) x^\top a_i  \right)^2 \right] +   \sum_{i\neq j}  \E_{x,y} \left[\phi'(x_i ) \phi(y_i ) x^\top a_i \cdot \phi'(x_j ) \phi(y_j ) x^\top a_j  \right] \\
= & ~ \sum_{i=1}^k \alpha_{2,0} \left[ (a_i^\top e_i)^2(\beta_{2,2} - \beta_{2,0}) + \beta_{2,0} \|a_i\|^2 \right] \\
& ~ + \sum_{i\neq j} \alpha_{1,0}^2 \left[ \beta_{1,0}^2 a_i^\top a_j +(\beta_{1,2} \beta_{1,0} - \beta_{1,0}^2)(a_i^\top e_i e_i^\top a_j + a_i^\top e_j a_j^\top e_j) + \beta_{1,1}^2(a_i^\top e_i a_j^\top e_j +  a_i^\top e_j a_j^\top e_i)  \right] \\
= & ~ C_1 + C_2.
\end{align*}
where the last step follows by 
\begin{align*}
C_1 = & ~ \sum_{i=1}^k \alpha_{2,0} \left[ (a_i^\top e_i)^2(\beta_{2,2} - \beta_{2,0}) + \beta_{2,0} \|a_i\|^2 \right] \\
C_2 = & ~ \sum_{i\neq j} \alpha_{1,0}^2 \left[ \beta_{1,0}^2 a_i^\top a_j +(\beta_{1,2} \beta_{1,0} - \beta_{1,0}^2)(a_i^\top e_i e_i^\top a_j + a_i^\top e_j a_j^\top e_j) + \beta_{1,1}^2(a_i^\top e_i a_j^\top e_j +  a_i^\top e_j a_j^\top e_i)  \right]
\end{align*}
First we can simplify $C_1$ in the following sense,
\begin{align*}
C_1 = & ~ \alpha_{2,0} (\beta_{2,2} - \beta_{2,0}) \sum_{i=1}^k (a_i^\top e_i)^2  + \alpha_{2,0} \beta_{2,0} \sum_{i=1}^k \| a_i \|_2^2 \\
= & ~ \alpha_{2,0} (\beta_{2,2} - \beta_{2,0}) \| \diag(A) \|_2^2 + \alpha_{2,0} \beta_{2,0} \| A \|_F^2,
\end{align*}
where the last step follows by Fact~\ref{fac:basic_1}.

We can rewrite $C_2$ in the following sense
\begin{align*}
C_2 = \alpha_{1,0}^2 ( \beta_{1,0}^2 C_{2,1} +  (\beta_{1,2}\beta_{1,0} - \beta_{1,0}^2) \cdot ( C_{2,2} + C_{2,3} ) + \beta_{1,1}^2  ( C_{2,4} + C_{2,5} )).
\end{align*}
where 
\begin{align*}
C_{2,1} = & ~ \sum_{i \neq j} a_i^\top a_j \\
C_{2,2} = & ~ \sum_{i \neq j} a_i^\top e_i e_i^\top a_j \\
C_{2,3} = & ~ \sum_{i \neq j} a_i^\top e_j e_j^\top a_j \\
C_{2,4} = & ~ \sum_{i \neq j} a_i^\top e_i a_j^\top e_j \\
C_{2,5} = & ~ \sum_{i \neq j} a_i^\top e_j a_j^\top e_i
\end{align*}
Using Fact~\ref{fac:basic_1}, we have
\begin{align*}
C_{2,1} = \| A \cdot {\bf 1} \|_2^2 - \| A \|_F^2.
\end{align*}
Using Fact~\ref{fac:basic_2}, we have
\begin{align*}
C_{2,2} = & ~ (\diag(A)^\top \cdot (A \cdot {\bf 1}) ) -  \| \diag(A) \|_2^2, \\
C_{2,3} = & ~ (\diag(A)^\top \cdot (A \cdot {\bf 1}) ) -  \| \diag(A) \|_2^2, \\
C_{2,4} = & ~ (\diag(A)^\top \cdot {\bf 1} )^2 - \| \diag(A) \|_2^2, \\
C_{2,5} = & ~ \langle A^\top, A \rangle - \| \diag(A) \|_2^2.
\end{align*}
Thus, 
\begin{align*}
C_2 
= & ~ \alpha_{1,0}^2 ( \beta_{1,0}^2 ( \| A \cdot {\bf 1} \|_2^2 - \| A\|_F^2 ) \\
& ~ +  (\beta_{1,2}\beta_{1,0} - \beta_{1,0}^2) 2 \cdot ( \diag(A)^\top \cdot (A \cdot {\bf 1}) - \| \diag(A) \|_2^2  ) \\
& ~ + \beta_{1,1}^2 ( (\diag(A)^\top \cdot {\bf 1})^2 + \langle A^\top, A \rangle - 2 \| \diag(A) \|_2^2 ) ).
\end{align*}

We consider $C_1+C_2$ by focusing different terms, for the $\| A\|_F^2$(from $C_1$ and $C_2$), 
we have
\begin{align*}
(\alpha_{2,0} \beta_{2,0} -\alpha_{1,0}^2 \beta_{1,0}^2)  \| A \|_F^2 .
\end{align*}
For the term $\langle A, A^\top \rangle$ (from $C_{2,5}$), we have
\begin{align*}
\alpha_{1,0}^2 \beta_{1,1}^2 \langle A, A^\top \rangle.
\end{align*}
For the term $\| \diag(A) \|_2^2$ (from $C_1$ and $C_2$),  we have
\begin{align*}
( \alpha_{2,0} (\beta_{2,2} - \beta_{2,0})  - 2 \alpha_{1,0}^2 (\beta_{1,2} \beta_{1,0} -\beta_{1,0}^2) - 2 \alpha_{1,0} \beta_{1,1}^2 ) \| \diag(A) \|_2^2
\end{align*}
For the term $\| A \cdot {\bf 1} \|_2^2$ (from $C_{2,1}$), we have
\begin{align*}
\alpha_{1,0}^2 \beta_{1,0}^2 \| A \cdot {\bf 1} \|_2^2.
\end{align*}
For the term $\diag(A)^\top \cdot A \cdot {\bf 1}$ (from $C_{2,2}$ and $C_{2,3}$), we have
\begin{align*}
2 \alpha_{1,0}^2 ( \beta_{1,2} \beta_{1,0} - \beta_{1,0}^2 ) \diag(A)^\top \cdot A \cdot {\bf 1}.
\end{align*}
For the term $ (\diag(A)^\top \cdot {\bf 1})^2 $ (from $C_{2,4}$), we have
\begin{align*}
\alpha_{1,0}^2 \beta_{1,1}^2 (\diag(A)^\top \cdot {\bf 1})^2.
\end{align*}

Putting it all together, we have
\begin{align*}
C_1 + C_2 = & ~ (\alpha_{2,0}\beta_{2,0} - \alpha_{1,0}^2 \beta_{1,0}^2) \|A\|_F^2 + \alpha_{1,0}^2 \beta_{1,1}^2 \langle A, A^\top \rangle \\
& ~ + ( \alpha_{2,0}(\beta_{2,2} - \beta_{2,0}) - 2\alpha_{1,0}^2(\beta_{1,2}\beta_{1,0} - \beta_{1,0}^2) - 2 \alpha_{1,0}^2 \beta_{1,1}^2) \cdot  \|\diag(A)\|^2 \\
& ~ + \alpha_{1,0}^2 \beta_{1,0}^2 \|A \cdot {\bf 1}\|^2 + 2\alpha_{1,0}^2(\beta_{1,2}\beta_{1,0} - \beta_{1,0}^2) ( \diag(A)^\top \cdot A \cdot {\bf 1} ) + 
\alpha_{1,0}^2 \beta_{1,1}^2(\diag(A)^\top \cdot {\bf 1})^2 \\
 = & ~ (\alpha_{2,0}\beta_{2,0} - \alpha_{1,0}^2 \beta_{1,0}^2) (\|A_o\|_F^2 + \|g_A\|^2) + \alpha_{1,0}^2 \beta_{1,1}^2 (\langle A_o, A_o^\top \rangle + \|g_A\|^2) \\
& ~ + (\alpha_{2,0}\beta_{2,2} -\alpha_{2,0} \beta_{2,0} - 2\alpha_{1,0}^2\beta_{1,2}\beta_{1,0} +2\alpha_{1,0}^2 \beta_{1,0}^2 - 2 \alpha_{1,0}^2 \beta_{1,1}^2) \cdot  \|g_A\|^2 \\
& ~ + \alpha_{1,0}^2 \beta_{1,0}^2 (\|g_A \|^2 + \|A_o\cdot {\bf 1} \|^2 + 2 g_A^\top \cdot A_o\cdot {\bf 1}) \\
& ~ + 2\alpha_{1,0}^2(\beta_{1,2}\beta_{1,0} - \beta_{1,0}^2) ( g_A^\top \cdot A_o\cdot {\bf 1} + \|g_A\|^2)  + 
\alpha_{1,0}^2 \beta_{1,1}^2(g_A^\top \cdot {\bf 1} )^2 \\
 = & ~ (\alpha_{2,0}\beta_{2,0} - \alpha_{1,0}^2 \beta_{1,0}^2) \|A_o\|_F^2  + \alpha_{1,0}^2 \beta_{1,1}^2 \langle A_o, A_o^\top \rangle   + (\alpha_{2,0}\beta_{2,2}    -  \alpha_{1,0}^2 \beta_{1,1}^2) \cdot  \|g_A\|^2 \\
& ~ + \alpha_{1,0}^2 \beta_{1,0}^2 ( \|A_o\cdot {\bf 1} \|^2 ) + 2\alpha_{1,0}^2\beta_{1,2}\beta_{1,0}  ( g_A^\top \cdot A_o \cdot {\bf 1} ) + \alpha_{1,0}^2 \beta_{1,1}^2(g_A^\top \cdot {\bf 1} )^2 .
\end{align*}
By doing a series of equivalent transformations, we have removed the expectation and the formula $C$ becomes a form of $A$ and the moments of $\phi$. 
These equivalent transforms are mainly based on the fact that $x_i,x_j,y_i,y_j$ for any $i\neq j$ are independent on each other. 

Similarly we can reformulate $D$,
\begin{align*}
& ~ \E_{x,y} \left[ \sum_{i,j}  \phi'(x_i ) \phi(y_i ) x^\top a_i \phi'(y_j )\phi(x_j ) y^\top b_j  \right] \\
= & ~ \sum_{i}  \E_{x,y} \left[ \phi'(x_i ) \phi(y_i ) x^\top a_i \phi'(y_ji )\phi(x_i ) y^\top b_i  \right]  + \sum_{i\neq j} \E_{x,y} \left[ \phi'(x_i ) \phi(y_i ) x^\top a_i \phi'(y_j )\phi(x_j ) y^\top b_j  \right] \\
= & ~ \sum_i \gamma^2 a_i^\top e_i b_i^\top e_i + \sum_{i\neq j} \alpha_{1,1}^2 a_i^\top e_j b_j^\top e_i +\alpha_{1,1}\beta_{1,1}(a_i^\top e_j b_j^\top e_j + a_i^\top e_i b_j^\top e_i) + \beta_{1,1}^2 a_i^\top e_i b_j^\top e_j \\
= & ~ (\gamma^2 - \beta_{1,0}^2 \alpha_{1,1}^2 - 2\alpha_{1,0}\alpha_{1,1}\beta_{1,0}\beta_{1,1} - \alpha_{1,0}^2 \beta_{1,1}^2)  g_A^\top g_B \\
& ~ + \beta_{1,0}^2 \alpha_{1,1}^2 \langle A, B^\top \rangle + \alpha_{1,0}^2 \beta_{1,1}^2 (g_A^\top 1)(g_B^\top 1)  \\
& ~ + \alpha_{1,0}\alpha_{1,1}\beta_{1,0}\beta_{1,1} [(A1)^\top g_B+ (B 1)^\top g_A] \\
= & ~ (\gamma^2  - \alpha_{1,0}^2 \beta_{1,1}^2)  g_A^\top g_B  + \beta_{1,0}^2 \alpha_{1,1}^2 \langle A_o, B_o^\top \rangle + \alpha_{1,0}^2 \beta_{1,1}^2 (g_A^\top 1)(g_B^\top 1) \\
& ~ + \alpha_{1,0}\alpha_{1,1}\beta_{1,0}\beta_{1,1} [(A_o1)^\top g_B+ (B_o 1)^\top g_A ].
\end{align*}

Combining the above results, we have
\begin{equation}\label{eq:orthognal_min_eig}
\begin{aligned}
\lambda_{\min}(H^*) 
 = & ~\min_{ \|A\|_F^2 + \|B\|_F^2 = 1}  \bigg( \beta_{1,0}^2 \alpha_{1,1}^2 \| A_o +  B_o^\top \|_F^2 \\
 & ~ + \| \alpha_{1,0} \beta_{1,0}A_o1 + \alpha_{1,0} \beta_{1,2} g_A + \alpha_{1,1} \beta_{1,1} g_B  \|^2 \\
 & ~ + \| \alpha_{1,0} \beta_{1,0}B_o1 + \alpha_{1,0} \beta_{1,2} g_B + \alpha_{1,1} \beta_{1,1} g_A  \|^2\\
& ~ + (\alpha_{2,0}\beta_{2,0} - \alpha_{1,0}^2 \beta_{1,0}^2 -  \beta_{1,0}^2 \alpha_{1,1}^2 - \alpha_{1,0}^2 \beta_{1,1}^2) (\|A_o\|_F^2+\|B_o\|_F^2)  \\
& ~ + 1/2\cdot \alpha_{1,0}^2 \beta_{1,1}^2 (\|A_o + A_o^\top \|_F^2 +\| B_o+ B_o^\top \|_F^2 )  \\
& ~ + [\alpha_{2,0}\beta_{2,2}    -  \alpha_{1,0}^2 \beta_{1,1}^2 -\alpha_{1,0}^2 \beta_{1,2}^2 - \alpha_{1,1}^2 \beta_{1,1}^2 ] \cdot ( \|g_A\|^2+ \|g_B\|^2) \\
& ~ + 2(\gamma^2  - \alpha_{1,0}^2 \beta_{1,1}^2 - 2 \alpha_{1,0} \alpha_{1,1} \beta_{1,1}\beta_{1,2}  )  g_A^\top g_B \\
& ~ + \alpha_{1,0}^2 \beta_{1,1}^2(g_A^\top 1+ g_B^\top 1)^2 \bigg). \\
\end{aligned}
\end{equation}

The final output of the above formula has a clear form: most non-negative terms are extracted. $A,B$ are separated into the off-diagonal elements and off-diagonal elements and these two terms can be dealt with independently. 
Now we consider the activation functions that satisfy $\beta_{1,1} = 0$, which further simplifies the equation. Note that Sigmoid and $\tanh$ satisfy this condition. 

Finally, for $\beta_{1,1} = 0$, we obtain
\begin{align*}
\lambda_{\min}(H^*) = & ~\min_{\sum_{i=1}^k \|a_i\|^2 + \|b_i\|^2 = 1}  \E_{x,y} \left[ \left( \sum_{i=1}^k \phi'(x_i ) \phi(y_i ) x^\top a_i + \phi'(y_i )\phi(x_i ) y^\top b_i \right)^2 \right]  \\
= & ~\min_{\|A\|_F^2+\|B\|_F^2 = 1}  (\alpha_{2,0}\beta_{2,0} - \alpha_{1,0}^2 \beta_{1,0}^2 - \beta_{1,0}^2 \alpha_{1,1}^2) (\|A_o\|_F^2 +\|B_o\|_F^2 )\\
& ~ +   (\alpha_{2,0}\beta_{2,2} - \alpha_{1,0}^2 \beta_{1,2}^2 -  \gamma^2) ( \|g_A\|^2+\|g_B\|^2)  \\
& ~ +  \beta_{1,0}^2 \alpha_{1,1}^2 \|A_o +  B_o^\top \|_F^2 + \gamma^2 \| g_A+g_B \|^2 \\
& ~ + \alpha_{1,0}^2 (\| \beta_{1,0} g_A + \beta_{1,2}A_o1\|^2+\alpha_{1,0}^2 \| \beta_{1,0} g_A + \beta_{1,2}B_o1\|^2) \\
\geq & ~ \underbrace{ \min\{(\alpha_{2,0}\beta_{2,0} - \alpha_{1,0}^2 \beta_{1,0}^2 - \beta_{1,0}^2 \alpha_{1,1}^2), (\alpha_{2,0}\beta_{2,2} - \alpha_{1,0}^2 \beta_{1,2}^2 -  \gamma^2)\} }_{:=\rho}.
\end{align*}

For sigmoid, we have $\rho = 0.000658$; for tanh, we have $\rho = 0.0095$.

\end{proof}

The following lemma will be used when transforming non-orthogonal cases to orthogonal cases. 
\begin{lemma}\label{lemma:ortho_min_eig_part}
For any $A = [a_1,a_2,\cdots, a_k] \in \R^{d \times k}$, we have,
\begin{align*}
 \E_{x,y\sim \mathcal{D}_k}  \left[ \left\| \sum_{i=1}^k \phi'(x_i ) \phi(y_i )  a_i \right\|^2 \right] 
  \geq ~ 
(\alpha_{2,0}\beta_{2,0} - \alpha_{1,0}^2\beta_{1,0}^2)\|A\|_F^2.
\end{align*}
\end{lemma}
\begin{proof}
Recall ${\bf 1} \in \R^d$ denote the all ones vector.
\begin{align*}
& ~ \E_{x,y\sim \mathcal{D}_k}  \left[ \left\| \sum_{i=1}^k \phi'(x_i ) \phi(y_i )  a_i  \right\|^2 \right] \\
 = & ~\E_{x,y\sim \mathcal{D}_k}  \left[ \sum_{i=1}^k(\phi'(x_i ) \phi(y_i ))^2  \| a_i  \|^2 \right]  + \E_{x,y\sim \mathcal{D}_k}  \left[ \sum_{i\neq j} \phi'(x_i ) \phi(y_i ) \phi'(x_j ) \phi(y_j )  a_i^\top a_j \right] \\
  = & ~ (\alpha_{2,0}\beta_{2,0} - \alpha_{1,0}^2\beta_{1,0}^2)\|A\|_F^2  + \alpha_{1,0}^2\beta_{1,0}^2 \|A \cdot 1\|^2 \\
  \geq & ~ (\alpha_{2,0}\beta_{2,0} - \alpha_{1,0}^2\beta_{1,0}^2)\|A\|_F^2 .
\end{align*}
Thus, we complete the proof.
\end{proof}

Now let's show the PD of the population Hessian of Eq.~\eqref{eq:emp_risk_relu} for the ReLU case. 
\define{lemma:relu_ortho}{Lemma}{
Consider the activation function to be ReLU. Assume $k=d$, $U^*,V^*$ are unitary matrices and $u^*_{1,i} \neq 0, \forall i \in [k]$. 
Then the minimal eigenvalue of the corresponding population Hessian of Eq.~\eqref{eq:emp_risk_relu} is lower bounded,
$$\lambda_{\min}(\nabla^2 f_{\mathcal{D}}^{\mathrm{ReLU}}(W^*,V^*)) \gtrsim  \min_{i \in [k]} \{u_{1,i}^{*2}\}, $$
where $W^* = U^*_{2:d,:}$ is the last $d-1$ rows of $U^*$ and 
\begin{align}\label{eq:pop_risk_relu}
 f_{\mathcal{D}}^{\mathrm{ReLU}}(W,V) 
:=  \E_{x,y} \left[ ( \phi( W^\top x_{2:d} + x_1 (u^{*(1)})^\top)^\top \phi (V^\top y) - A(x,y) )^2 \right],
\end{align}
}
where $u^{*(1)}$ is the first row of $U^*$ and $W \in \dR^{(d-1)\times k}$.

\state{lemma:relu_ortho}

\begin{proof}
By fixing $u_{i,1} = u^*_{i,1} ,\forall i \in [k]$, we can rewrite the minimal eigenvalue of the Hessian as follows. For simplicity, we denote $\lambda_{\min}(H) := \lambda_{\min}(\nabla^2 f_{\mathcal{D}}^{\mathrm{ReLU}}(W^*,V^*))$. First we observe that
\begin{equation}\label{eq:constrained_min}
\lambda_{\min}(H)   = \min_{ \substack{ \sum_{i=1}^k \|a_i\|^2 + \|b_i\|^2 = 1 \\ a_{i,1} = 0, \forall i \in [k] } }  \E_{x,y} \left[ \left( \sum_{i=1}^k \phi'(u_i^{*\top} x ) \phi(v_i^{*\top} y ) x^\top a_i + \phi'(v_i^{*\top} y )\phi(u_i^{*\top} x ) y^\top b_i \right)^2 \right].
\end{equation}
Without loss of generality, we assume $V^* = I$. Set $x = U^*s$, then we have 
\begin{align*}
\lambda_{\min}(H)  & = \min_{ \substack{ \sum_{i=1}^k \|a_i\|^2 + \|b_i\|^2 = 1 \\ a_{i,1} = 0, \forall i\in [k] } }  \E_{x,y} \left[ \left( \sum_{i=1}^k \phi'(s_i ) \phi(y_i ) s^\top U^{*\top}a_i + \phi'(y_i )\phi(x_i ) y^\top b_i \right)^2 \right] \\
& = \min_{ \substack{ \sum_{i=1}^k \|a_i\|^2 + \|b_i\|^2 = 1 \\ u^{*(1)}a_{i} = 0, \forall i \in [k] }}  \E_{x,y} \left[ \left( \sum_{i=1}^k \phi'(s_i ) \phi(y_i ) s^\top a_i + \phi'(y_i )\phi(x_i ) y^\top b_i \right)^2 \right],
\end{align*}
where $u^{*(1)}$ is the first row of $U^*$ and the second equality is because we replace $ U^{*\top}a_i$ by $a_i$.
In the ReLU case, we have 
$$\alpha_{1,0} = \alpha_{1,1} = \alpha_{2,0} = \beta_{1,0}= \beta_{1,1}= \beta_{1,1}= \beta_{2,0}= \beta_{2,2} = \gamma = 1/2.$$ 
According to Eq.~\eqref{eq:orthognal_min_eig}, we have
\begin{align*}
\lambda_{\min}(H) \geq \min_{\|A\|_F^2 + \|B\|_F^2 =1 ,u^{*(1)}A = 0}  C_0 ( & \|A_o\|_F^2 + \|B_o\|_F^2 + \|A_o + A_o^\top\|_F^2/2 + \|B_o + B_o^\top\|_F^2/2 \\
& + \|A_o+B_o^\top\|_F^2 +  \|g_A+g_B\|^2 \\
& + \|A_o 1 + g_A+g_B\|^2 + \|B_o 1 + g_A+g_B\|^2 + (g_A^\top 1 + g_B^\top 1)^2),
\end{align*}
where $C_0$ is a universal constant.
Now we show that there exists a positive number $c_0$ such that $\lambda_{\min}(H) \geq c_0$. If there is no such number, i.e., $\lambda_{\min}(H) = 0$, then we have
$A_o = B_o = 0$, $g_A = -g_B$. By the assumption that $u_{1,i}^* \neq 0$ and the condition $u_{}^{*(1)} A = 0$, we have $g_A = g_B = 0$, which violates $\|A\|_F^2 + \|B\|_F^2 = 1$. So $\lambda_{\min}(H) > 0$. An exact value for $c_0$ is postponed to Theorem~\ref{thm:relu}, which gives the lower bound for the non-orthogonal case.
\end{proof}

\subsection{Non-orthogonal Case}
The restriction of orthogonality on $U,V$ is too strong. We need to consider general non-orthogonal cases. With Gaussian assumption, the non-orthogonal case can be transformed to the orthogonal case according to the following relationship. 
\begin{lemma}\label{lemma:transform_ortho}
Let $U \in \dR^{d\times k}$ be a full-column rank matrix. Let $g: \dR^k \rightarrow [0,\infty)$. Define $\lambda(U) = \sigma_1^k(U)/(\prod_{i=1}^k \sigma_i(U))$. Let ${\cal D}$ denote the normal distribution. Then
\begin{equation}\label{eq:ortho_nonortho}
\E_{x \sim \mathcal{D}_d} \left[ g(U^\top x) \right] \geq \frac{1}{\lambda(U)} \E_{z\sim \mathcal{D}_k} \left[ g(\sigma_k(U) z) \right] .
\end{equation}
\end{lemma}
{\bf Remark} This lemma transforms $U^\top x$, where the elements of $x$ are mixed, to $\sigma_k(U)z$, where all the elements are independently fed into $g$ with the sacrifices of a condition number of $U$.
Using Lemma~\ref{lemma:transform_ortho}, we are able to show the PD for non-orthogonal $U^*,V^*$.

\begin{proof}
Let $P \in \dR^{d\times k} $ be the orthonormal basis of $U$, and let $W=[w_1,w_2,\cdots, w_k] = P^\top U \in \dR^{k\times k}$.
\begin{align*}
& ~ \E_{x \sim \mathcal{D}_d}[g(U^\top x)] \\
= & ~ \E_{z \sim \mathcal{D}_k}[g(W^\top z)] \\
= & ~ \int (2\pi)^{-k/2} g(W^\top z) e^{-\|z\|^2/2} \mathrm{d} z \\
= & ~ \int (2\pi)^{-k/2} g(s) e^{-\| W^{\dagger\top} s \|^2/2} |\det(W^\dagger)| \mathrm{d} s \\
\geq & ~ \int (2\pi)^{-k/2} g(s) e^{-\sigma_1^2( W^{\dagger} ) \|s \|^2/2} |\det(W^\dagger)| \mathrm{d} s \\
= & ~ \int (2\pi)^{-k/2} g \left(\frac{1}{\sigma_1(W^\dagger)}t \right) e^{-\| t \|^2/2} |\det(W^\dagger)|/\sigma_1^k(W^\dagger) \mathrm{d} t \\
= & ~ \frac{1}{\lambda(W)}\int (2\pi)^{-k/2} g(\sigma_k(W)t) e^{-\| t \|^2/2} \mathrm{d} t \\
= & ~ \frac{1}{\lambda(U)} \E_{z \sim \mathcal{D}_k}[g(\sigma_k(U) z)],
\end{align*}
where the third step follows by replacing $z$ by $z = W^{\dagger\top}s$, the fourth step follows by the fact that $\| W^{\dagger\top} s \| \leq \sigma_1(W^\dagger) \|s\|$, and the fifth step follows replacing $s$ by $s = \frac{1}{\sigma_1(W^\dagger)}t$.
\end{proof}

Using Lemma~\ref{lemma:transform_ortho}, we are able to provide the lower bound for the minimal eigenvalue for sigmoid and tanh. 
\begin{theorem}\label{thm:non_ortho_sig}
Assume $\sigma_k(U^*) = \sigma_k(V^*)=1$. Assume $\beta_{1,1}$ defined in Eq.~\eqref{eq:moments_def} is $0$. Then the minimal eigenvalue of Hessian defined in Eq.~\eqref{eq:hessian_min_eig} can be lower bounded by,
\begin{align*}
\lambda_{\min}(H^*) \geq  \frac{\rho}{\lambda(U^*)\lambda(V^*)\max\{\kappa(U^*),\kappa(V^*)\}}
\end{align*}
where 
\begin{align*}
\lambda(U) = \sigma_1^k(U)/(\Pi_{i=1}^k \sigma_i(U)),\kappa(U) = \sigma_1(U)/\sigma_k(U).
\end{align*}
\end{theorem}

\begin{proof}
Let $P \in \dR^{d\times k} ,Q \in \dR^{d\times k} $ be the orthonormal basis of $U^*,V^*$ respectively. Let $R \in \R^{k \times k}, S\in \R^{k \times k}$ satisfy that $U^* = P \cdot R$ and $V^* = Q \cdot S$. Let $P_{\perp}\in \dR^{d\times (d-k)} ,Q_{\perp}\in \dR^{d\times (d-k)}$ be the orthogonal complement of $P,Q$ respectively. 
Set $a_i = P \cdot s_i + P_\perp \cdot t_i$ and $b_i = Q \cdot p_i + Q_\perp \cdot q_i$. Then we can decompose the minimal eigenvalue problem into three terms. 

\begin{align*}
&  \E_{x,y} \left[ \left( \sum_{i=1}^k \phi'(u_i^{*\top} x ) \phi(v_i^{*\top} y ) x^\top a_i + \phi'(v_i^{*\top} y )\phi(u_i^{*\top} x ) y^\top b_i \right)^2 \right] \\
 = &   \E_{x,y} \left[ \left( \sum_{i=1}^k \phi'(u_i^{*\top} x ) \phi(v_i^{*\top} y ) x^\top (Ps_i + P_\perp t_i) + \phi'(v_i^{*\top} y )\phi(u_i^{*\top} x ) y^\top  (Qp_i + Q_\perp q_i)  \right)^2 \right] \\
 = &  \underbrace{ \E_{x,y} \left[ \left( \sum_{i=1}^k \phi'(u_i^{*\top} x ) \phi(v_i^{*\top} y ) x^\top Ps_i  + \phi'(v_i^{*\top} y )\phi(u_i^{*\top} x ) y^\top  Qp_i \right)^2 \right]}_{C_1} \\
 & + \underbrace{\E_{x,y} \left[ \left( \sum_{i=1}^k \phi'(u_i^{*\top} x ) \phi(v_i^{*\top} y ) x^\top  P_\perp t_i  \right)^2 \right]}_{C_2} + \E_{x,y} \left[ \left( \sum_{i=1}^k \phi'(v_i^{*\top} y )\phi(u_i^{*\top} x ) y^\top  Q_\perp q_i \right)^2 \right] ,
\end{align*}
where we omit the terms containing a single independent Gaussian variable, whose expectation is zero. 
Using Lemma~\ref{lemma:transform_ortho}, we can lower bound the term $C_1$ as follows,
\begin{align*}
C_1 = & ~ \E_{x,y} \left[ \left( \sum_{i=1}^k \phi'(u_i^{*\top} x ) \phi(v_i^{*\top} y ) x^\top U^*R^{-1}s_i  + \phi'(v_i^{*\top} y )\phi(u_i^{*\top} x ) y^\top  V^*S^{-1}p_i \right)^2 \right] \\
\geq & ~  \frac{1}{\lambda(U^*)\lambda(V^*)} \cdot \E_{x,y\sim \mathcal{D}_k} \left[ \left( \sum_{i=1}^k \phi'(\sigma_k(U^*) x_i) ) \phi(y_i ) x^\top R^{-1}s_i\sigma_k(U^*) \right.\right. \\
& ~ \left.\left.  + \phi'(\sigma_k(V^*) y_i )\phi(\sigma_k(U^*) x_i ) y^\top S^{-1}p_i  \sigma_k(V^*) \right)^2 \right].
\end{align*}

And 
\begin{align*}
C_2 \geq & \E_{x,y} \left[ \left\| \sum_{i=1}^k \phi'(u_i^{*\top} x ) \phi(v_i^{*\top} y )  t_i  \right\|^2 \right] \\
\geq & \frac{1}{\lambda(U^*)\lambda(V^*)} \E_{x,y\sim \mathcal{D}_k}  \left[ \left\| \sum_{i=1}^k \phi'(\sigma_k(U^*)x_i ) \phi(\sigma_k(V^*)y_i )  t_i  \right\|^2 \right].
\end{align*}

Without loss of generality, we assume $\sigma_k(U^*) = \sigma_k(V^*)=1$.
Then according to Lemma~\ref{lemma:ortho_min_eig_formal} and Lemma~\ref{lemma:ortho_min_eig_part}, we have 
\begin{align*}
\lambda_{\min}(H) \geq & ~ \frac{1}{\lambda(U^*)\lambda(V^*)\max\{\kappa(U^*),\kappa(V^*)\}} \\
& ~ \cdot \min\{(\alpha_{2,0}\beta_{2,0} - \alpha_{1,0}^2 \beta_{1,0}^2 - \beta_{1,0}^2 \alpha_{1,1}^2), (\alpha_{2,0}\beta_{2,2} - \alpha_{1,0}^2 \beta_{1,2}^2 -  \gamma^2)\}.
\end{align*}
Considering the definition of $\rho$ in Eq.~\eqref{eq:moments_def}, we complete the proof.
\end{proof}

For the ReLU case, we lower bound the minimal eigenvalue of the Hessian for non-orthogonal cases. 
\begin{theorem}\label{thm:relu}
Consider the activation to be ReLU. Assume $U^*,V^*$ are full-column-rank matrices and $u^*_{1,i} \neq 0, \forall i \in [k]$. 
Then the minimal eigenvalue of the Hessian of Eq.~\eqref{eq:pop_risk_relu} is lower bounded,
\begin{align*}
& \lambda_{\min}(\nabla^2 f_{\mathcal{D}}^{\mathrm{ReLU}}(W^*,V^*))  
\gtrsim \frac{1}{\lambda(U^*)\lambda(V^*)} \left( \frac{\min_{i \in [k]}\{|u_{1,i}^{*}|\}}{(1 + \|u^{*(1)} \|)\max\{\|U^*\|,\|V^*\|\}} \right)^2,
\end{align*}
where $u^{*(1)}$ is the first row of $U^*$. 
\end{theorem}

\begin{proof}

Let $P \in \dR^{d\times k} ,Q \in \dR^{d\times k} $ be the orthonormal basis of $U^*,V^*$ respectively. Let $R \in \R^{k \times k}, S\in \R^{k \times k}$ satisfy that $U^* = P \cdot R$ and $V^* = Q \cdot S$. Let $P_{\perp}\in \dR^{d\times (d-k)} ,Q_{\perp}\in \dR^{d\times (d-k)}$ be the orthogonal complement of $P,Q$ respectively. 
Set $a_i = P \cdot s_i + P_\perp \cdot t_i$ and $b_i = Q \cdot p_i + Q_\perp \cdot q_i$. 
Similar to the proof of Theorem~\ref{thm:non_ortho_sig}, Lemma~\ref{lemma:ortho_min_eig_part} and Lemma~\ref{lemma:relu_ortho}, we have the following.

\begin{align*}
& ~ \E_{x,y} \left[ \left( \sum_{i=1}^k \phi'(u_i^{*\top} x ) \phi(v_i^{*\top} y ) x^\top a_i + \phi'(v_i^{*\top} y )\phi(u_i^{*\top} x ) y^\top b_i \right)^2 \right] \\
\geq & ~ \frac{1}{\lambda(U^*)\lambda(V^*)} \E_{x,y\sim \mathcal{D}_k} \left[ \left( \sum_{i=1}^k \phi'(\sigma_k(U^*) x_i) ) \phi(y_i ) x^\top R^{-1}s_i\sigma_k(U^*) \right. \right.\\
& ~ \left. \left. + \phi'(\sigma_k(V^*) y_i )\phi(\sigma_k(U^*) x_i ) y^\top S^{-1}p_i  \sigma_k(V^*) \right)^2 \right] \\
& ~ + \frac{1}{\lambda(U^*)\lambda(V^*)} \E_{x,y\sim \mathcal{D}_k}  \left[ \left\| \sum_{i=1}^k \phi'(\sigma_k(U^*)x_i ) \phi(\sigma_k(V^*)y_i )  t_i  \right\|^2 \right] \\
& ~ + \frac{1}{\lambda(U^*)\lambda(V^*)} \E_{x,y\sim \mathcal{D}_k}  \left[ \left\| \sum_{i=1}^k \phi'(\sigma_k(U^*)x_i ) \phi(\sigma_k(V^*)y_i )  q_i  \right\|^2 \right] \\
\geq & ~ \frac{1}{16 \lambda(U^*)\lambda(V^*)} ( \|\hat  A_o\|_F^2 + \|\hat  B_o\|_F^2 +  \|g_{\hat A}+g_{\hat B}\|^2  + 3(\|\hat T\|_F^2 + \|\hat Q\|_F^2)),
\end{align*}

where $\hat A = [R^{-1}s_1,R^{-1}s_2,\cdots,R^{-1}s_k]$, $\hat B = [S^{-1}p_1,S^{-1}p_2,\cdots,S^{-1}p_k]$, $\hat T = [t_1,t_2,\cdots, t_k]$, $\hat Q = [q_1,q_2,\cdots,q_k]$. 

Similar to Eq.~\eqref{eq:constrained_min}, we can find the minimal eigenvalue of the Hessian by the following constrained minimization problem. 

\begin{equation*}
\begin{aligned}
\lambda_{\min}(H) & = \min_{ \substack{ \sum_{i=1}^k \|a_i\|^2 + \|b_i\|^2 = 1 \\ a_{i,1} = 0, \forall i \in [k]} }  \E_{x,y} \left[ \left( \sum_{i=1}^k \phi'(u_i^{*\top} x ) \phi(v_i^{*\top} y ) x^\top a_i + \phi'(v_i^{*\top} y )\phi(u_i^{*\top} x ) y^\top b_i \right)^2 \right], \\
\end{aligned}
\end{equation*}

which is lower bounded by the following formula.
\begin{equation}\label{eq:relu_min_eig_opt}
\begin{aligned}
 \underset{\hat A,\hat B, \hat T, \hat P}{\min} & ~ \frac{1}{16 \lambda(U^*)\lambda(V^*)} ( \|\hat  A_o\|_F^2 + \|\hat  B_o\|_F^2 +  \|g_{\hat A}+g_{\hat B}\|^2  + 3(\|\hat T\|_F^2 + \|\hat Q\|_F^2)) \\
 \text{s.t. } & ~ \|R\hat A\|_F^2 + \|S\hat B\|_F^2 + \|\hat T\|_F^2 + \|\hat Q\|_F^2  = 1 \\ & ~  e_1^\top P R\hat A + e_1^\top P_{\perp} \hat T= 0 
\end{aligned}
\end{equation}
If we assume the minimum of the above formula is $c_1$. We show that $c_1>0$ by contradiction. If $c_1=0$, then $\hat T = \hat Q = 0$, $\hat A_o = \hat B_o = 0$, $g_{\hat A} = -g_{\hat B}$. Since $\hat T=0$, we have $e_1^\top PR \hat A =e_1^\top U^*\hat A = 0$. Assuming $(e_1^\top U^*)_i \neq 0, \forall i$, we have $g_{\hat A} = g_{\hat B} = 0$. This violates the condition that $ \|R\hat A\|_F^2 + \|S\hat B\|_F^2 + \|\hat T\|_F^2 + \|\hat Q\|_F^2  = 1$. 

Now we give a lower bound for $c_1$. First we note,
\begin{align*}
 \|R\hat A\|_F^2 + \|S\hat B\|_F^2 + \|\hat T\|_F^2 + \|\hat Q\|_F^2  \leq \|R\|^2 \|\hat A\|^2_F + \|S\|^2 \|\hat B\|_F^2 + \|\hat T\|_F^2 + \|\hat Q\|_F^2 .
 \end{align*}
 Therefore, 
\begin{align*}
\|\hat A\|^2_F + \|\hat B\|_F^2 + \|\hat T\|_F^2 + \|\hat Q\|_F^2  \geq \frac{1}{\max\{\|U^*\|^2,\|V^*\|^2\}}.
 \end{align*}
  
Also, as $e_1^\top U^* \hat A_o + (e_1^\top U^*) \odot g_{\hat A}^\top + e_1^\top P_\perp \hat T = 0$, where $\odot$ is the element-wise product, we have
\begin{align*}
\|g_{\hat A}\|^2 & \leq (\frac{1}{\min\{|u_{1,i}^{*}|\}} (\|u^{*(1)} \|\|\hat A_o \| + \| \hat T\|)^2 \\
&  \leq \left( \frac{1 + \|u^{*(1)} \|}{\min\{|u_{1,i}^{*}|\}} \right)^2 2(\|\hat A_o \|_F^2 + \| \hat T\|_F^2) .
\end{align*}

Note that $\|g_{\hat A}\|^2 + \|g_{\hat A} + g_{\hat B}\|^2 \geq \frac{1}{2} \|g_{\hat B}\|^2$. Now let's return to the main part of objective function Eq.~\eqref{eq:relu_min_eig_opt}. 
\begin{align*}
& ~ \|\hat  A_o\|_F^2 + \|\hat  B_o\|_F^2 +  \|g_{\hat A}+g_{\hat B}\|^2  + 3(\|\hat T\|_F^2 + \|\hat Q\|_F^2) \\
 \geq & ~ \frac{2}{3} (\|\hat  A_o\|_F^2 +\|\hat T\|_F^2)  +  \frac{1}{3}\|\hat  A_o\|_F^2 + \|\hat  B_o\|_F^2 +  \|g_{\hat A}+g_{\hat B}\|^2  + \|\hat T\|_F^2 + \|\hat Q\|_F^2 \\
 \geq & ~ \frac{1}{3} \left( \frac{\min\{|u_{1,i}^{*}|\}}{1 + \|u^{*(1)} \|} \right)^2 \|g_{\hat A}\|^2  +  \frac{1}{3}\|\hat  A_o\|_F^2 + \|\hat  B_o\|_F^2 +  \|g_{\hat A}+g_{\hat B}\|^2  + \|\hat T\|_F^2 + \|\hat Q\|_F^2 \\
  \geq & ~ \frac{1}{12} \left( \frac{\min\{|u_{1,i}^{*}|\}}{1 + \|u^{*(1)} \|} \right)^2 (\|g_{\hat A}\|^2 + \|g_{\hat B}\|^2)  +  \frac{1}{3}\|\hat  A_o\|_F^2 + \|\hat  B_o\|_F^2  + \|\hat T\|_F^2 + \|\hat Q\|_F^2 \\
    \geq & ~ \frac{1}{12} \left( \frac{\min\{|u_{1,i}^{*}|\}}{1 + \|u^{*(1)} \|} \right)^2 \left(\|g_{\hat A}\|^2 + \|g_{\hat B}\|^2 + \|\hat  A_o\|_F^2 + \|\hat  B_o\|_F^2  + \|\hat T\|_F^2 + \|\hat Q\|_F^2 \right) \\
        \geq & ~ \frac{1}{12} \left( \frac{\min\{|u_{1,i}^{*}|\}}{(1 + \|u^{*(1)} \|)\max\{\|U^*\|,\|V^*\|\}} \right)^2 .
\end{align*}
Therefore, 
\begin{align*}
c_1        \geq &\frac{1}{200\lambda(U^*)\lambda(V^*)} \left( \frac{\min\{|u_{1,i}^{*}|\}}{(1 + \|u^{*(1)} \|)\max\{\|U^*\|,\|V^*\|\}} \right)^2 .
\end{align*}

\end{proof}

\section{Positive Definiteness of the Empirical Hessian}\label{sec:emp_pop_hessian}
For any $(U, V)$, the population Hessian can be decomposed into the following $2k\times 2k$ blocks ($i\in[k], j\in[k]$), 

\begin{equation}\label{eq:Hessian_blocks}
\begin{aligned}
\frac{ \partial^2 f_\mathcal{D} (U,V) }{\partial u_i \partial u_j}  = & ~  \E_{x,y} \left[ \phi'(u_i^{\top} x) \phi'(u_j^{\top} x) xx^\top \phi(v_i^{\top} y )\phi(v_j^{\top} y) \right] \\
& + \delta_{ij}\E_{x,y} \left[ \left( \phi(U^\top x)^\top \phi(V^\top y) - \phi(U^{*\top} x)^\top \phi(V^{*\top} y)  \right) \phi''(u_i^\top x) \phi(v_i^\top y)xx^\top   \right] \\
\frac{ \partial^2 f_\mathcal{D} (U,V) }{\partial u_i \partial v_j}  = & ~  \E_{x,y} \left[ \phi'(u_i^{\top} x) \phi'(v_j^{\top} y) xy^\top \phi(v_i^{\top} y )\phi(u_j^{\top} x) \right] \\
& + \delta_{ij}\E_{x,y} \left[ \left( \phi(U^\top x)^\top \phi(V^\top y) - \phi(U^{*\top} x)^\top \phi(V^{*\top} y)  \right) \phi'(u_i^\top x) \phi'(v_i^\top y)xy^\top   \right] ,
\end{aligned}
\end{equation}
where $\delta_{ij}=1$ if $i=j$, otherwise $\delta_{ij}=0$. Similarly we can write the formula for $\frac{ \partial^2 f_\mathcal{D} (U,V) }{\partial v_i \partial v_j}  $ and $\frac{ \partial^2 f_\mathcal{D} (U,V) }{\partial v_i \partial u_j}  $. 

Replacing $\E_{x,y}$ by $\frac{1}{|\Omega|}\sum_{(x,y) \in \Omega}$ in the above formula, we can obtain the formula for the corresponding empirical Hessian, $\nabla^2f_\Omega(U,V)$.

We now bound the difference between $\nabla^2f_\Omega(U,V)$ and $\nabla^2f_{\mathcal{D}}(U^*,V^*)$. 

\begin{theorem}[Restatement of Theorem~\ref{thm:empirical_error_bound}]
For any $\epsilon>0$, if 
\begin{align*}
 n_1 \gtrsim \epsilon^{-2} t d \log^2 d ,  \quad n_2 \gtrsim \epsilon^{-2} t \log d , \quad |\Omega| \gtrsim \epsilon^{-2} t d \log^2 d ,
\end{align*}
then with probability $1-d^{-t}$,
for sigmoid/tanh,
\begin{equation*}
\| \nabla^2f_\Omega(U,V) - \nabla^2f_{\mathcal{D}}(U^*,V^*)\|  \lesssim \epsilon +  \|U - U^*\| + \|V - V^*\|,  
 \end{equation*}
for ReLU,
\begin{equation*}
\| \nabla^2f_\Omega(U,V) - \nabla^2f_{\mathcal{D}}(U^*,V^*)\|  \lesssim  \left(\left(\frac{\|V - V^*\|}{\sigma_k(V^*)}\right)^{1/2}+ \left(\frac{\|U - U^*\|}{\sigma_k(U^*)}\right)^{1/2}  + \epsilon\right) (\|U^*\| + \|V^*\|)^2.
 \end{equation*}

\end{theorem}
\begin{proof}
Define $H(U,V) \in \dR^{(2kd)\times (2kd)}$ as a symmetric matrix, whose blocks are represented as
\begin{equation}\label{eq:H_def}
\begin{aligned}
H_{u_i, u_j}  = & ~  \E_{x,y} \left[ \phi'(u_i^{\top} x) \phi'(u_j^{\top} x) xx^\top \phi(v_i^{\top} y )\phi(v_j^{\top} y) \right], \\
H_{ u_i, v_j}  = & ~  \E_{x,y} \left[ \phi'(u_i^{\top} x) \phi'(v_j^{\top} y) xy^\top \phi(v_i^{\top} y )\phi(u_j^{\top} x) \right].
\end{aligned}
\end{equation}
where $H_{u_i,u_j} \in \dR^{d\times d},H_{u_i,v_j}\in \dR^{d\times d}$ correspond to $\frac{ \partial^2 f_\mathcal{D} (U,V) }{\partial u_i \partial u_j}, \frac{ \partial^2 f_\mathcal{D} (U,V) }{\partial u_i \partial v_j}$ respectively.

We decompose the difference into
\begin{align*}
\| \nabla^2f_\Omega(U,V) - \nabla^2f_{\mathcal{D}}(U^*,V^*)\| & \leq \| \nabla^2 f_\Omega(U,V) - H(U,V)\| + \|H(U,V) - \nabla^2f_{\mathcal{D}}(U^*,V^*)\|.
\end{align*}

Combining Lemma~\ref{lemma:near_ground_emp_pop}, \ref{lemma:near_ground_pop_pop}, we complete the proof.
\end{proof}

%
%
%
%
\begin{lemma}\label{lemma:near_ground_emp_pop}
For any $\epsilon>0$, if 
\begin{align*}
 n_1 \gtrsim \epsilon^{-2} t d \log^2 d ,  \quad n_2 \gtrsim \epsilon^{-2} t \log d , \quad |\Omega| \gtrsim \epsilon^{-2} t d \log^2 d,
\end{align*}
then with probability $1-d^{-t}$,
for sigmoid/tanh,
$$ \| \nabla^2 f_\Omega(U,V) - H(U,V)\| \lesssim  \epsilon +  \|U - U^*\| + \|V - V^*\|  ,$$
for ReLU,
$$ \| \nabla^2 f_\Omega(U,V) - H(U,V)\| \lesssim  \epsilon \|U^*\|\|V^*\| .$$
\end{lemma}
\begin{proof}
We can bound $ \| \nabla^2 f_\Omega(U,V) - H(U,V)\| $ if we bound each block.

We can show that if 
\begin{align*} n_1 \gtrsim \epsilon^{-2} t d \log^2 d ,  \quad n_2 \gtrsim \epsilon^{-2} t \log d , \quad |\Omega| \gtrsim \epsilon^{-2} t d \log^2 d,
\end{align*}
then with probability $1-d^{-t}$,
\begin{align*}
& ~ \left\| \left( \E_{x,y} - \frac{1}{|\Omega|} \sum_{(x,y) \in \Omega}\right) \left[ \phi'(u_i^{\top} x) \phi'(u_j^{\top} x) xx^\top \phi(v_i^{\top} y )\phi(v_j^{\top} y) \right] \right\| \\
\lesssim & ~ \epsilon \|U^*\|^p\|V^*\|^p&\text{Lemma~\ref{lemma:diag_emp_pop_sym}}\\
& ~ \left\| \frac{1}{|\Omega|} \sum_{(x,y) \in \Omega}  \left[ \left( \phi(U^\top x)^\top \phi(V^\top y) - \phi(U^{*\top} x)^\top \phi(V^{*\top} y)  \right) \phi''(u_i^\top x) \phi(v_i^\top y)xx^\top   \right] \right\| \\
 \lesssim & ~ \|U - U^*\| + \|V - V^*\|& \text{Lemma~\ref{lemma:offdiag_emp_pop_sym}}\\
& ~ \left\| \left( \E_{x,y} - \frac{1}{|\Omega|} \sum_{(x,y) \in \Omega} \right) \left[ \phi'(u_i^{\top} x) \phi'(v_j^{\top} y) xy^\top \phi(v_i^{\top} y )\phi(u_j^{\top} x) \right] \right\| \\
 \lesssim & ~ \epsilon \|U^*\|^p\|V^*\|^p& \text{Lemma~\ref{lemma:diag_emp_pop_asym}} \\
&\left\| \frac{1}{|\Omega|} \sum_{(x,y) \in \Omega}  \left[ \left( \phi(U^\top x)^\top \phi(V^\top y) - \phi(U^{*\top} x)^\top \phi(V^{*\top} y)  \right) \phi'(u_i^\top x) \phi'(v_i^\top y)xy^\top   \right] \right\| \\
 \lesssim & ~ \|U - U^*\| + \|V - V^*\|, & \text{Lemma~\ref{lemma:offdiag_emp_pop_asym}}
\end{align*}
where $p=1$ if $\phi$ is ReLU, $p=0$ if $\phi$ is sigmoid/tanh. 

Note that for ReLU activation, for any given $U,V$, the second term is $0$ because $\phi''(z) = 0$ almost everywhere. 
\end{proof}

\begin{lemma}\label{lemma:diag_emp_pop_sym}
If 
\begin{align*}
 n_1 \gtrsim \epsilon^{-2} t d \log^2 d ,  \quad n_2 \gtrsim \epsilon^{-2} t \log d , \quad |\Omega| \gtrsim \epsilon^{-2} t d \log^2 d,
 \end{align*}
then with probability at least $1-d^{-t}$,
\begin{align*}
\left\| \left( \E_{x,y} - \frac{1}{|\Omega|} \sum_{(x,y) \in \Omega}\right) \left[ \phi'(u_i^{\top} x) \phi'(u_j^{\top} x) xx^\top \phi(v_i^{\top} y )\phi(v_j^{\top} y) \right] \right\| \leq \epsilon \|v_i\|^{p} \|v_j\|^p
\end{align*}
where $p=1$ if $\phi$ is $\mathrm{ReLU}$, $p=0$ if $\phi$ is sigmoid/tanh. 
\end{lemma}

\begin{proof}
Let $B(x,y) = \phi'(u_i^{\top} x) \phi'(u_j^{\top} x) xx^\top \phi(v_i^{\top} y )\phi(v_j^{\top} y)$. By applying Lemma~\ref{lemma:two_bernstein_E_Omega} and Property $\mathrm{(\RN{1})} - \mathrm{(\RN{3})},  \mathrm{(\RN{6})}$ in Lemma~\ref{lemma:zzt_prop} and Lemma~\ref{lemma:yy_prop}, we have for any $\epsilon >0$ if 
\begin{align*}
n_1 \gtrsim \epsilon^{-2} t d \log^2 d ,\quad n_2 \gtrsim \epsilon^{-2} t \log d,
\end{align*} 
then with probability at least $1-d^{-2t}$,

\begin{align}\label{eq:E_S_diff}
\left\|\E_{x,y} [ B(x,y) ] - \frac{1}{|S| } \sum_{(x,y) \in S} B(x,y) \right\|  \leq \epsilon \|v_i\|^p \|v_j\|^p.
 \end{align}

By applying Lemma~\ref{lemma:second_order_bound} and Property $\mathrm{(\RN{1})}, \mathrm{(\RN{3})}- \mathrm{(\RN{5})} $ in Lemma~\ref{lemma:zzt_prop} and Lemma~\ref{lemma:yy_prop}, we have for any $\epsilon >0$ if 
\begin{align*}
 n_1 \gtrsim \epsilon^{-1} t d \log^2 d,  \quad n_2 \gtrsim \epsilon^{-2}  t \log d ,
 \end{align*}
then
\begin{align*}
   \left\|\frac{1}{n_1}\sum_{l \in [n_1]} (\phi'(u_i^{\top} x_l) \phi'(u_j^{\top} x_l))^2 \|x_l\|^2 x_lx_l^\top \right\| \lesssim d ,
\end{align*}
and 
\begin{align*}
\left\| \frac{1}{n_2}\sum_{l\in [n_2]} (\phi(v_i^{\top} y_l )\phi(v_j^{\top} y_l))^2 \right\| \lesssim \|v_i\|^{2p}\|v_j\|^{2p}.
\end{align*}
Therefore,
\begin{equation}\label{eq:second_order_bound_eq}
\max \left( \left\|\frac{1}{|S| } \sum_{(x,y) \in S} B(x,y) B(x,y)^\top \right\| , \left\|\frac{1}{|S| } \sum_{(x,y) \in S} B(x,y)^\top B(x,y) \right\|  \right) \lesssim \epsilon d \|v_i\|^{2p}\|v_j\|^{2p}.
\end{equation}
We can apply Lemma~\ref{lemma:two_bernstein} and use Eq.~\eqref{eq:second_order_bound_eq} and Property  $\mathrm{(\RN{1})}$ in Lemma~\ref{lemma:zzt_prop} and Lemma~\ref{lemma:yy_prop} to obtain the following result. 
If 
\begin{align*}
|\Omega| \gtrsim \epsilon^{-2} t d \log^2 d ,
\end{align*}
then with probability at least $1-d^{-2t}$,
\begin{equation}\label{eq:O_S_diff}
\left\| \frac{1}{|S| } \sum_{(x,y) \in S} B(x,y) - \frac{1}{|\Omega| } \sum_{(x,y) \in \Omega} B(x,y) \right\|  \lesssim \epsilon \|v_i\|^{p} \|v_j\|^p.
 \end{equation}
 Combining Eq.~\eqref{eq:E_S_diff} and \eqref{eq:O_S_diff}, we finish the proof.
\end{proof}

\begin{lemma}\label{lemma:zzt_prop}
Define $T(z) = \phi'(u_i^{\top} z) \phi'(u_j^{\top} z) zz^\top $. If $z\sim \mathcal{Z}$, $ \mathcal{Z} = \mathcal{N}(0,I_d)$ and $\phi$ is ReLU or sigmoid/tanh, the following holds for $T(z)$ and any $t>1$,
\begin{align*} 
\mathrm{(\RN{1})} \quad & \quad   \underset{z \sim {\cal Z}}{\Pr}\left[ \left\| T(z) \right\| \leq 5td \log n \right] \geq 1 -n^{-1}d^{-t} ; \\ 
\mathrm{(\RN{2})} \quad & \quad \max_{\| a\|=\| b\|=1} \left( \underset{z \sim {\cal Z}}{\mathbb{E}} \left[ \left( a^\top T(z)  b \right)^2 \right]  \right)^{1/2} \lesssim 1;\\
\mathrm{(\RN{3})} \quad & \quad \max \left( \left\| \underset{z \sim {\cal Z}}{\mathbb{E}} [ T(z) T(z)^\top ] \right\|, \left\| \underset{z \sim {\cal Z}}{\mathbb{E}} [ T(z)^\top T(z) ] \right\| \right) \lesssim d ; \\ 
\mathrm{(\RN{4})} \quad & \quad \max_{\| a\|=1} \left( \underset{z \sim {\cal Z}}{\mathbb{E}} \left[ \left( a^\top T(z) T(z)^\top a \right)^2 \right]  \right)^{1/2} \lesssim d;\\
\mathrm{(\RN{5})} \quad & \quad   \left\| \underset{z \sim {\cal Z}}{\mathbb{E}} [ T(z) T(z)^\top T(z) T(z)^\top ] \right\| \lesssim d^3; \\
\mathrm{(\RN{6})} \quad & \quad \left\| \underset{z \sim {\cal Z}}{\mathbb{E}} [T(z)] \right\|  \lesssim 1.
\end{align*}
\end{lemma}

\begin{proof}
Note that $0\leq \phi'(z) \leq 1$, therefore $\mathrm{(\RN{1})}$ can be proved by Proposition 1 of \cite{hsu2012tail}.
$\mathrm{(\RN{2})} - \mathrm{(\RN{6})} $ can be proved by H\"{o}lder's inequality.
\end{proof}

\begin{lemma}\label{lemma:yy_prop}
Define $T(z) =\phi(v_i^{\top} z )\phi(v_j^{\top} z) $.  If $z\sim \mathcal{Z}$, $ \mathcal{Z} = \mathcal{N}(0,I_d)$ and $\phi$ is ReLU or sigmoid/tanh, the following holds for $T(z)$ and any $t>1$,
\begin{align*} \mathrm{(\RN{1})} \quad & \quad \underset{z \sim {\cal Z}}{\Pr}\left[ \left\| T(z) \right\|\leq 5t \|v_i\|^p \|v_j\|^p \log n \right] \geq 1 -n^{-1}d^{-t} ; \\ 
\mathrm{(\RN{2})} \quad & \quad \max_{\| a\|=\| b\|=1} \left( \underset{z \sim {\cal Z}}{\mathbb{E}} \left[ \left( a^\top T(z)  b \right)^2 \right]  \right)^{1/2} \lesssim \|v_i\|^p \|v_j\|^p;\\
\mathrm{(\RN{3})} \quad & \quad \max \left( \left\| \underset{z \sim {\cal Z}}{\mathbb{E}} [ T(z) T(z)^\top ] \right\|, \left\| \underset{z \sim {\cal Z}}{\mathbb{E}} [ T(z)^\top T(z) ] \right\| \right)  \lesssim \|v_i\|^{2p} \|v_j\|^{2p} ; \\ 
\mathrm{(\RN{4})} \quad & \quad \max_{\| a\|=1} \left( \underset{z \sim {\cal Z}}{\mathbb{E}} \left[ \left( a^\top T(z) T(z)^\top a \right)^2 \right]  \right)^{1/2}  \lesssim \|v_i\|^{2p} \|v_j\|^{2p};\\
\mathrm{(\RN{5})} \quad & \quad   \left\| \underset{z \sim {\cal Z}}{\mathbb{E}} [ T(z) T(z)^\top T(z) T(z)^\top ] \right\|  \lesssim \|v_i\|^{4p} \|v_j\|^{4p}; \\
\mathrm{(\RN{6})} \quad & \quad \left\| \underset{z \sim {\cal Z}}{\mathbb{E}} [T(z)] \right\|  \lesssim \|v_j\|^{p}\|v_i\|^{p}.
\end{align*}
where $p=1$ if $\phi$ is ReLU, $p=0$ if $\phi$ is sigmoid/tanh.
\end{lemma}

\begin{proof}
Note that $|\phi(z)| \leq |z|^p$, therefore $\mathrm{(\RN{1})}$ can be proved by Proposition 1 of \cite{hsu2012tail}.
$\mathrm{(\RN{2})} - \mathrm{(\RN{6})} $ can be proved by H\"{o}lder's inequality
\end{proof}

\begin{lemma}\label{lemma:offdiag_emp_pop_sym}
If 
\begin{align*}
 n_1 \gtrsim \epsilon^{-2} t d \log^2 d,  \quad n_2 \gtrsim \epsilon^{-2} t \log d , \quad |\Omega| \gtrsim \epsilon^{-2} t d \log^2 d,
\end{align*}
then with probability at least $1-d^{-t}$,
\begin{align*}
& ~ \left\| \frac{1}{|\Omega|} \sum_{(x,y) \in \Omega}  \left[ \left( \phi(U^\top x)^\top \phi(V^\top y) - \phi(U^{*\top} x)^\top \phi(V^{*\top} y)  \right) \phi''(u_i^\top x) \phi(v_i^\top y)xx^\top   \right] \right\| \\
\lesssim & ~ (\|U - U^*\| + \|V - V^*\|).
\end{align*}
\end{lemma}
\begin{proof}
We consider the following formula first,
\begin{align*}
& \left\| \frac{1}{|\Omega|} \sum_{(x,y) \in \Omega}  \left[ \left( (\phi(u_j^\top x) - \phi(u_j^{*\top} x)) \phi(v_j^{*\top} y)  \right) \phi''(u_i^\top x) \phi(v_i^\top y)xx^\top   \right] \right\|\\
 \leq &\left\| \frac{1}{|\Omega|} \sum_{(x,y) \in \Omega}  \left[ \left| (u_j - u_j^*)^\top x \right| xx^\top  \phi(v_j^{*\top} y)\phi(v_i^\top y) \right] \right\|.
\end{align*}
Similar to Lemma~\ref{lemma:diag_emp_pop_sym}, we are able to show 
\begin{align*}
& ~ \left\| \frac{1}{|\Omega|} \sum_{(x,y) \in \Omega}  \left[ \left| (u_j - u_j^*)^\top x \right| xx^\top  \phi(v_j^{*\top} y)\phi(v_i^\top y) \right] -\E_{(x,y)}  \left[ \left| (u_j - u_j^*)^\top x \right| xx^\top  \phi(v_j^{*\top} y)\phi(v_i^\top y) \right]  \right\| \\
 \lesssim & ~ \|U - U^*\|.
\end{align*}
Note that by H\"{o}lder's inequality, we have,
\begin{align*}
\left\| \E_{(x,y)}  \left[ \left| (u_j - u_j^*)^\top x \right| xx^\top  \phi(v_j^{*\top} y)\phi(v_i^\top y) \right]  \right\| \lesssim \|U - U^*\| .
\end{align*}
So we complete the proof.
\end{proof}

\begin{lemma}\label{lemma:diag_emp_pop_asym}
If 
\begin{align*}
 n_1 \gtrsim \epsilon^{-2} t d \log^2 d ,  \quad n_2 \gtrsim \epsilon^{-2}  t \log d, \quad |\Omega| \gtrsim \epsilon^{-2} t d \log^2 d,
\end{align*}
then with probability at least $1-d^{-t}$,
\begin{align*}
 \left\| \left( \E_{x,y} - \frac{1}{|\Omega|} \sum_{(x,y) \in \Omega} \right) \left[ \phi'(u_i^{\top} x) \phi'(v_j^{\top} y) xy^\top \phi(v_i^{\top} y )\phi(u_j^{\top} x) \right] \right\|  \lesssim  \epsilon \|v_i\|^{p} \|u_j\|^p.
\end{align*}
\end{lemma}

\begin{proof}
Let $B(x,y) = M(x)N(y)$, where $M(x) = \phi'(u_i^{\top} x)  \phi(u_j^{\top} x) x$ and $N(y)  =\phi'(v_j^{\top} y)  \phi(v_i^{\top} y ) y^\top$. By applying Lemma~\ref{lemma:two_bernstein_E_Omega} and Property $\mathrm{(\RN{1})} - \mathrm{(\RN{3})},  \mathrm{(\RN{6})}$ in Lemma~\ref{lemma:xy_prop} , we have for any $\epsilon >0$ if 
\begin{align*}
n_1 \gtrsim \epsilon^{-2} t d \log^2 d ,\quad n_2 \gtrsim \epsilon^{-2} t d \log^2 d,
\end{align*} 
then with probability  at least $1-d^{-2t}$,

\begin{equation}\label{eq:E_S_diff_xy}
\left\| \E_{x,y}B(x,y) - \frac{1}{|S| } \sum_{(x,y) \in S} B(x,y) \right\|  \lesssim \epsilon \|u_j\|^p \|v_i\|^p.
 \end{equation}

By applying Lemma~\ref{lemma:second_order_bound} and Property $\mathrm{(\RN{1})}, \mathrm{(\RN{4})}- \mathrm{(\RN{6})} $ in Lemma~\ref{lemma:xy_prop}, we have for any $\epsilon >0$ if 
\begin{align*}
 n_1 \gtrsim \epsilon^{-2} t d \log^2 d , n_2 \gtrsim \epsilon^{-2} t d \log^2 d,
\end{align*}
then
\begin{align*}
\left\| \frac{1}{n_1}\sum_{l \in [n_1]} M(x_l)M(x_l)^\top \right\| \lesssim \|u_j\|^{2p},\quad
\left\| \frac{1}{n_2}\sum_{l\in [n_2]} N(y_l)^\top N(y_l) \right\| \lesssim  \|v_i\|^{2p}.
\end{align*}

By applying Lemma~\ref{lemma:second_order_bound} and Property $\mathrm{(\RN{1})}, \mathrm{(\RN{4})}, \mathrm{(\RN{7})}, \mathrm{(\RN{8})} $ in Lemma~\ref{lemma:xy_prop}, we have for any $\epsilon >0$ if 
\begin{align*}
n_1 \gtrsim \epsilon^{-2} t d \log^2 d, n_2 \gtrsim \epsilon^{-2} t d \log^2 d,
\end{align*}
then
\begin{align*}
\left\| \frac{1}{n_1}\sum_{l \in [n_1]} M(x_l)^\top M(x_l) \right\| \lesssim d \|u_j\|^{2p},\quad
\left\| \frac{1}{n_2}\sum_{l\in [n_2]} N(y_l) N(y_l)^\top \right\| \lesssim d \|v_i\|^{2p}. 
\end{align*}

Therefore,
\begin{equation}\label{eq:second_order_bound_eq_xy}
\max \left( \left\| \frac{1}{|S| } \sum_{(x,y) \in S} B(x,y) B(x,y)^\top \right\| ,\left\|\frac{1}{|S| } \sum_{(x,y) \in S} B(x,y)^\top B(x,y) \right\| \right) \lesssim \epsilon d \|v_i\|^{2p}\|u_j\|^{2p}
\end{equation}
We can apply Lemma~\ref{lemma:two_bernstein} and Eq.~\eqref{eq:second_order_bound_eq_xy} and Property  $\mathrm{(\RN{1})}$ in Lemma~\ref{lemma:xy_prop} to obtain the following result. 
If 
\begin{align*}
|\Omega| \gtrsim \epsilon^{-2} t d \log^2 d,
\end{align*}
then with probability at least $1-d^{-2t}$,
\begin{equation}\label{eq:O_S_diff_xy}
\left\| \frac{1}{|S| } \sum_{(x,y) \in S} B(x,y)- \frac{1}{|\Omega| } \sum_{(x,y) \in \Omega} B(x,y) \right\|  \leq \epsilon \|v_i\|^{p} \|u_j\|^p.
 \end{equation}
 Combining Eq.~\eqref{eq:E_S_diff_xy} and \eqref{eq:O_S_diff_xy}, we finish the proof.
\end{proof}

\begin{lemma}\label{lemma:xy_prop}
Define $T(z) = \phi'(u_i^{\top} z) \phi(u_j^{\top} z) z $. If $z\sim \mathcal{Z}$, $ \mathcal{Z} = \mathcal{N}(0,I_d)$ and $\phi$ is ReLU or sigmoid/tanh, the following holds for $T(z)$ and any $t>1$,
\begin{align*} 
\mathrm{(\RN{1})} \quad & \quad   \underset{z \sim {\cal Z}}{\Pr}\left[ \left\| T(z) \right\| \leq 5t d^{1/2} \|u_j\|^p \log n \right] \geq 1 -n^{-1}d^{-t} ; \\ 
\mathrm{(\RN{2})} \quad & \quad \left\| \underset{z \sim {\cal Z}}{\mathbb{E}} [T(z)] \right\|  \lesssim \|u_j\|^{p};\\
\mathrm{(\RN{3})} \quad & \quad \max_{\| a\|=\| b\|=1} \left( \underset{z \sim {\cal Z}}{\mathbb{E}} \left[ \left( a^\top T(z)  b \right)^2 \right]  \right)^{1/2} \lesssim  \|u_j\|^{p};\\
\mathrm{(\RN{4})} \quad & \quad \max \left\{ \left\| \underset{z \sim {\cal Z}}{\mathbb{E}} [ T(z) T(z)^\top ] \right\|, \left\| \underset{z \sim {\cal Z}}{\mathbb{E}} [ T(z)^\top T(z) ] \right\| \right\} \lesssim d\|u_j\|^{2p} ; \\ 
\mathrm{(\RN{5})} \quad & \quad \max_{\| a\|=1} \left( \underset{z \sim {\cal Z}}{\mathbb{E}} \left[ \left( a^\top T(z) T(z)^\top a \right)^2 \right]  \right)^{1/2} \lesssim \|u_j\|^{2p}  ;\\
\mathrm{(\RN{6})} \quad & \quad   \left\| \underset{z \sim {\cal Z}}{\mathbb{E}} [ T(z) T(z)^\top T(z) T(z)^\top ] \right\| \lesssim d \|u_j\|^{4p}; \\
\mathrm{(\RN{7})} \quad & \quad \max_{\| a\|=1} \left( \underset{z \sim {\cal Z}}{\mathbb{E}} \left[ \left( a^\top T(z)^\top T(z) a \right)^2 \right]  \right)^{1/2} \lesssim d \|u_j\|^{2p}  ;\\
\mathrm{(\RN{8})} \quad & \quad   \left\| \underset{z \sim {\cal Z}}{\mathbb{E}} [ T(z)^\top T(z) T(z)^\top T(z) ] \right\| \lesssim d^2 \|u_j\|^{4p}.
\end{align*}
\end{lemma}

\begin{proof}
Note that $0\leq \phi'(z) \leq 1$,$|\phi(z)| \leq |z|^p$, therefore $\mathrm{(\RN{1})}$ can be proved by Proposition 1 of \cite{hsu2012tail}.
$\mathrm{(\RN{2})} - \mathrm{(\RN{8})} $ can be proved by H\"{o}lder's inequality.
\end{proof}

\begin{lemma}\label{lemma:offdiag_emp_pop_asym}
If 
\begin{align*} n_1 \gtrsim t d \log^2 d,  \quad n_2 \gtrsim  t \log d, \quad |\Omega| \gtrsim t d \log^2 d,
\end{align*}
then with probability at least $1-d^{-t}$,
\begin{align*}
& ~ \left\| \frac{1}{|\Omega|} \sum_{(x,y) \in \Omega}  \left[ \left( \phi(U^\top x)^\top \phi(V^\top y) - \phi(U^{*\top} x)^\top \phi(V^{*\top} y)  \right) \phi'(u_i^\top x) \phi'(v_i^\top y)xy^\top   \right] \right\| \\
  \lesssim & ~ \|U - U^*\| + \|V - V^*\| .
\end{align*}
\end{lemma}
\begin{proof}
We consider the following formula first,
\begin{align*}
& ~ \left\| \frac{1}{|\Omega|} \sum_{(x,y) \in \Omega}  \left[ \left( (\phi(u_j^\top x) - \phi(u_j^{*\top} x)) \phi(v_j^{*\top} y)  \right) \phi'(u_i^\top x) \phi'(v_i^\top y)xy^\top   \right] \right\|\\
\end{align*}

Set $M(x) = (\phi(u_j^\top x) - \phi(u_j^{*\top} x)) \phi'(u_i^\top x) x  $ and $N(y)  = \phi(v_j^{*\top} y) \phi'(v_i^\top y)y^\top $ and follow the proof for Lemma~\ref{lemma:diag_emp_pop_asym}. Also note that $\phi$ is Lipschitz, i.e., $|\phi(u_j^\top x) - \phi(u_j^{*\top} x) | \leq | u_j^\top x - u_j^{*\top} x |$. We can show the following. If 
\begin{align*}
n_1 \gtrsim t d \log^2 d,  \quad n_2 \gtrsim  t \log d, \quad |\Omega| \gtrsim t d \log^2 d,
\end{align*}
then with probability at least $1-d^{-t}$,
\begin{align*}
\left\| \left( \frac{1}{|\Omega|} \sum_{(x,y) \in \Omega} -\E_{x,y}\right)\left[M(x)N(y) \right] \right\| \lesssim \|u_j - u_j^*\| .
\end{align*}
Note that by H\"{o}lder's inequality, we have,
\begin{align*}
\| \E_{x,y} \left[M(x)N(y) \right] \| \lesssim \|u_j - u_j^*\| .
\end{align*}
So we complete the proof.
 
\end{proof}

We provide a variation of Lemma B.7 in \cite{zsjbd17}. Note that the Lemma B.7 \cite{zsjbd17} requires four properties, we simplify it into three properties.
\begin{lemma}[Matrix Bernstein for unbounded case (A modified version of bounded case, Theorem 6.1 in \cite{t11}, A variation of Lemma B.7 in \cite{zsjbd17})]
\label{lem:modified_bernstein_non_zero}
Let ${\cal B}$ denote a distribution over $\mathbb{R}^{d_1 \times d_2}$. Let $d = d_1 +d_2$. Let $B_1, B_2, \cdots B_n$ be i.i.d. random matrices sampled from ${\cal B}$. Let $\overline{B} = \mathbb{E}_{B\sim {\cal B}} [B]$ and $\wh{B}  = \frac{1}{n} \sum_{i=1}^n B_i$. For parameters $m\geq 0, \gamma \in (0,1),\nu >0 ,L>0$, if the distribution ${\cal B}$ satisfies the following four properties,
 \begin{align*}
\mathrm{(\RN{1})} \quad & \quad \underset{B \sim {\cal B}}{\Pr}\left[ \left\| B \right\| \leq  m \right] \geq 1 - \gamma; \\ 
\mathrm{(\RN{2})} \quad & \quad \max \left( \left\| \underset{B \sim {\cal B}}{\mathbb{E}} [ B B^\top ] \right\|, \left\| \underset{B \sim {\cal B}}{\mathbb{E}} [ B^\top B ] \right\| \right) \leq \nu ;\\ 
\mathrm{(\RN{3})} \quad & \quad \max_{\| a\|=\| b\|=1} \left( \underset{B \sim {\cal B}}{\mathbb{E}} \left[ \left( a^\top B  b \right)^2 \right]  \right)^{1/2} \leq L.
\end{align*}
Then we have for any $\epsilon > 0$ and $t\geq 1$, if
\begin{equation*}
n \geq ( 18 t \log d)  \cdot ( (\epsilon + \|\ov{B}\|)^2+ m \epsilon +\nu) / \epsilon^2  \quad \text{~and~} \quad \gamma \leq (\epsilon/ (2L))^2
\end{equation*} 
with probability  at least $1-d^{-2t}-n\gamma$,
\begin{equation*}
\left\| \frac{1}{n} \sum_{i=1}^n B_i - \underset{B \sim {\cal B} }{\mathbb{E}}[B] \right\| \leq \epsilon.
\end{equation*}
\end{lemma}

\begin{proof}
Define the event \[\xi_i = \{ \|B_i\| \leq m \}, \forall i\in [n]. \]
Define $M_i = \bone_{\|B_i\| \leq m}B_i$. Let $\ov{M} = \mathbb{E}_{B\sim {\cal B}} [ \bone_{\|B \| \leq m}B ] $ and $\wh{M} = \frac{1}{n}\sum_{i=1}^n M_i$. By triangle inequality, we have 
\begin{align}\label{eq:whB_minus_ovB}
 \| \wh{B} - \ov{B} \| \leq \|\wh{B} - \wh{M}\| +\|\wh{M} - \ov{M}\| + \| \ov{M} - \ov{B} \|.
\end{align}
  In the next a few paragraphs, we will upper bound the above three terms. 

{\bf The first term in Eq.~\eqref{eq:whB_minus_ovB}}. For each $i$, let $\ov{\xi}_i$ denote the complementary set of $\xi_i$, i.e. $\ov{\xi}_i = [n] \backslash \xi_i$. Thus $\Pr[ \ov{\xi}_i] \leq \gamma$.
By a union bound over $i\in [n]$, with probability $1- n \gamma$, $\|B_i\| \leq m$ for all $i\in [n]$. Thus $\wh{M} = \wh{B}$. 

{\bf The second term in Eq.~\eqref{eq:whB_minus_ovB}}. For a matrix $B$ sampled from ${\cal B}$, we use $\xi$ to denote the event that $\xi = \{ \| B \| \leq m\}$. Then, we can upper bound $\|\ov{M}-\ov{B}\|$ in the following way,
\begin{align}
& ~ \| \ov{M}- \ov{B} \| \notag \\
= & ~ \left\| \underset{B \sim {\cal B}}{\mathbb{E}} [\bone_{\| B \| \leq m} \cdot B] - \underset{B\sim {\cal B}}{\mathbb{E}} [B] \right\| \notag \\
=& ~ \left\| \underset{B \sim {\cal B}}{\mathbb{E}} \left[ B \cdot \bone_{ \ov{\xi} } \right] \right\| \notag \\
=& ~ \max_{\| a\|=\| b\| = 1}  \underset{B \sim {\cal B}}{\mathbb{E}} \left[ a^\top B b \bone_{\ov{\xi}} \right] \notag \\
\leq & ~  \max_{\| a\|=\| b\|=1}  \underset{B \sim {\cal B}}{\mathbb{E}} [ ( a^\top B b)^2 ]^{1/2} \cdot \underset{B \sim {\cal B}}{\mathbb{E}} \left[ \bone_{ \ov{\xi} } \right]^{1/2} &\text{~by~H\"{o}lder's inequality} \notag \\
\leq & ~ L \underset{B \sim {\cal B}}{\mathbb{E}} \left[\bone_{ \ov{\xi} } \right]^{1/2}  & \text{~by~Property~(\RN{4})} \notag \\
\leq & ~  L \gamma^{1/2}, &\text{~by~}\Pr [ \ov{\xi} ] \leq \gamma \notag \\
\leq & ~ \frac{1}{2} \epsilon, & \text{~by~} \gamma \leq (\epsilon / (2L))^2, \notag 
\end{align}
which is
\begin{align*}
\| \ov{M}- \ov{B} \|\leq \frac{\epsilon}{2}.
\end{align*}
Therefore, $\|\ov{M}\|\leq \| \ov{B} \| + \frac{\epsilon}{2}$.

{\bf The third term in Eq.~\eqref{eq:whB_minus_ovB}}. We can bound $\|\wh{M} - \ov{M}\|$ by Matrix Bernstein's inequality \cite{t11}. 

We define $Z_i = M_i - \ov{M}$. Thus we have $ \underset{B_i\sim {\cal B}}{\mathbb{E}} [Z_i] = 0$, $\|Z_i\| \leq 2m$,  and 
\begin{align*}
 \left\| \underset{B_i\sim {\cal B}}{\mathbb{E}} [ Z_iZ_i^\top ] \right\| = \left\| \underset{B_i\sim {\cal B}}{\mathbb{E}} [ M_iM_i^\top] - \ov{M}\cdot \ov{M}^\top \right\| \leq \nu + \| \ov{M} \|^2 \leq  \nu + \| \ov{B} \|^2 + \epsilon^2 + \epsilon \|\ov{B}\|.
\end{align*}
  Similarly, we have $\left\|\underset{B_i\sim {\cal B}}{\mathbb{E}} [Z_i^\top Z_i] \right\| \leq   \nu+ \| \ov{B} \|^2 + \epsilon^2 + \epsilon \|\ov{B}\| $. 
Using matrix Bernstein's inequality, for any $\epsilon>0$, 
\begin{align*}
& \underset{B_1, \cdots, B_n \sim {\cal B} }{\Pr} \left[ \frac{1}{n} \left\| \sum_{i=1}^{n}Z_i \right\|\geq \epsilon  \right] 
\leq  d \exp\left(-\frac{\epsilon^2 n /2}{ \nu+ \| \ov{B} \|^2 + \epsilon^2 + \epsilon \|\ov{B}\| + 2m \epsilon /3} \right) .
\end{align*}
By choosing
\begin{align*}
n \geq ( 3t \log d ) \cdot  \frac{\nu+\| \ov{B} \|^2 + \epsilon^2 + \epsilon \|\ov{B}\| + 2m \epsilon /3}{ \epsilon^2 /2 } ,
\end{align*}
for $t\geq 1$, we have with probability at least $1-d^{-2t}$,
\begin{align*}
\left\| \frac{1}{n}\sum_{i=1}^n M_i - \ov{M} \right\| \leq \frac{\epsilon}{2}.
\end{align*}

Putting it all together, we have for $\epsilon >0 $, if
\begin{equation*}
n \geq ( 18 t \log d)  \cdot ( (\epsilon + \|\ov{B}\|)^2+ m \epsilon +\nu) / (\epsilon^2)  \quad \text{~and~} \quad \gamma \leq (\epsilon/ (2L))^2
\end{equation*} 
with probability  at least $1-d^{-2t}-n\gamma$,
\begin{equation*}
\left\| \frac{1}{n} \sum_{i=1}^n B_i - \underset{B \sim {\cal B} }{\mathbb{E}}[B] \right\| \leq \epsilon.
\end{equation*}

\end{proof}

\begin{lemma}[Tail Bound for fully-observed rating matrix]\label{lemma:two_bernstein_E_Omega}
Let $\{x_i\}_{i \in [n_1]}$ be independent samples from distribution $\mathcal{X}$ and $\{y_j\}_{j \in [ n_2] }$ be independent samples from distribution $\mathcal{Y}$. Denote $S := \{(x_i, y_j)\}_{i\in[n_1], j\in[n_2]}$ as the collection of all the $(x_i,y_j)$ pairs. Let $B(x,y)$ be a random matrix of $x,y$, which can be represented as the product of two matrices $M(x), N(y)$, i.e., $B(x,y) = M(x) N(y)$.  Let $\ov M =  \E_{x}M(x) $ and $\ov N = \E_y N(y)$. Let $d_x$ be the sum of the two dimensions of $M(x)$ and $d_y$ be the sum of the two dimensions of $N(y)$. Suppose both $M(x)$ and $N(y)$ satisfy the following properties ($z$ is a representative for $x,y$, and $T(z)$ is a representative for $M(x),N(y)$),
\begin{align*}
\mathrm{(\RN{1})} \quad & \quad \underset{z \sim {\cal Z}}{\Pr}\left[ \left\| T(z) \right\| \leq  m_z \right] \geq 1 - \gamma_z; \\ 
\mathrm{(\RN{2})} \quad & \quad \max_{\| a\|=\| b\|=1} \left( \underset{z \sim {\cal Z}}{\mathbb{E}} \left[ \left( a^\top T(z)  b \right)^2 \right]  \right)^{1/2} \leq L_z;\\
\mathrm{(\RN{3})} \quad & \quad \max \left( \left\| \underset{z \sim {\cal Z}}{\mathbb{E}} [ T(z) T(z)^\top ] \right\|, \left\| \underset{z \sim {\cal Z}}{\mathbb{E}} [ T(z)^\top T(z) ] \right\| \right) \leq \nu_z. \\ 
\end{align*}
Then for any $\epsilon_1>0, \epsilon_2>0$ if
\begin{align*}
& n_1 \geq  ( 18 t \log d_x ) \cdot ( \nu_x + (\| \ov{M} \| + \epsilon_1)^2 + m_x \epsilon_1 )  / \epsilon_1^2   \quad  \text{~and~} \quad \gamma_x \leq (\epsilon_1  /(2L_x) )^2 \\
& n_2 \geq  ( 18 t \log d_y  ) \cdot ( \nu_y + (\epsilon_2 +\| \ov{N} \|)^2+ m_y \epsilon_2 )  / \epsilon_2^2  \quad  \text{~and~} \quad \gamma_y \leq (\epsilon_2  /(2L_y) )^2 \\
\end{align*} 
with probability  at least $1-d_x^{-2t} - d_y^{-2t} - n_1\gamma_x - n_2\gamma_y$,

\begin{equation}
\left\| \E_{x,y}B(x,y) - \frac{1}{|S| } \sum_{(x,y) \in S} B(x,y) \right\|  \leq \epsilon_2 \|\ov M\| + \epsilon_1 \|\ov N\| + \epsilon_1\epsilon_2.
 \end{equation}
\end{lemma}

\begin{proof}
First we note that,
\begin{equation*}
 \frac{1}{|S| } \sum_{(x,y) \in S} B(x,y) =  \frac{1}{n_1n_2 } \left( \sum_{i \in [n_1]} M(x_i) \right) \cdot \left( \sum_{j \in [n_2]} N(y_j) \right) ,
 \end{equation*}
and 
\begin{equation*}
\E_{x,y} [ B(x,y) ] = \left( \E_x [M(x)] \right) \left( \E_y[ N(y) ] \right).
 \end{equation*}
Therefore, if we can bound $\| \E_x [ M(x) ]  - \frac{1}{n_1}\sum_{i \in [n_1]} M(x_i) \|$ and the corresponding term for $y$, we are able to prove this lemma. 

 By the conditions of $M(x)$, the three conditions in Lemma~\ref{lem:modified_bernstein_non_zero} are satisfied, which completes the proof.

\end{proof}

\begin{lemma}[Upper bound for the second-order moment]\label{lemma:second_order_bound}
Let $\{z_i\}_{i \in [n]}$ be independent samples from distribution $\mathcal{Z}$. Let $T(z)$ be a matrix of $z$.
 Let $d$ be the sum of the two dimensions of $T(z)$ and $\ov T :=  \underset{z \sim {\cal Z}}{\mathbb{E}}[T(z)T(z)^\top] $. Suppose $T(z)$ satisfies the following properties.
\begin{align*}
\mathrm{(\RN{1})} \quad & \quad \underset{z \sim {\cal Z}}{\Pr}\left[ \left\| T(z) \right\| \leq  m_z \right] \geq 1 - \gamma_z; \\ 
\mathrm{(\RN{2})} \quad & \quad \max_{\| a\|=1} \left( \underset{z \sim {\cal Z}}{\mathbb{E}} \left[ \left( a^\top T(z) T(z)^\top a \right)^2 \right]  \right)^{1/2} \leq L_z;\\
\mathrm{(\RN{3})} \quad & \quad  \left\| \underset{z \sim {\cal Z}}{\mathbb{E}} [ T(z) T(z)^\top T(z) T(z)^\top ] \right\| \leq \nu_z,
\end{align*}
Then for any $t>1$, if
\begin{align*}
& n \geq  ( 18 t \log d  ) \cdot ( \nu_z + (\| \ov{T} \|+\epsilon)^2 + m^2_z  )  / \epsilon^2  \quad  \text{~and~} \quad \gamma_z \leq ( \epsilon /(2L_z) )^2, \\
\end{align*} 
we have with probability at least $1-d^{-2t} - n\gamma_z$,
\begin{align*}
\left\| \frac{1}{n}\sum_{i \in [n]} T(z_i)T(z_i)^\top \right\| \leq  \left\| \underset{z \sim {\cal Z}}{\mathbb{E}}[T(z)T(z)^\top]  \right\|  + \epsilon. 
 \end{align*}
\end{lemma}

\begin{proof}
The proof directly follows by applying Lemma~\ref{lem:modified_bernstein_non_zero}.
\end{proof}

  \begin{lemma}[Tail Bound for partially-observed rating matrix]\label{lemma:two_bernstein}
Given $\{x_i\}_{i \in [n_1]}$ and $\{y_j\}_{j \in [n_2]}$, let's denote $S := \{(x_i, y_j)\}_{i\in[n_1], j\in[n_2]}$ as the collection of all the $(x_i,y_j)$ pairs. Let $\Omega$ also be a collection of $(x_i,y_j)$ pairs, where each pair is sampled from $S$ independently and uniformly. Let $B(x,y)$ be a matrix of $x,y$. Let $d_B$ be the sum of the two dimensions of $B(x,y)$. Define $\ov{B}_S = \frac{1}{|S| } \sum_{(x,y) \in S} B(x,y)$. Assume the following,
\begin{align*}
\mathrm{(\RN{1})} \quad & \quad \| B(x,y) \| \leq m_B,  \forall (x,y) \in S,\\
\mathrm{(\RN{2})} \quad & \quad \max \left( \left\|\frac{1}{|S|}\sum_{(x,y)\in S} B(x,y)B(x,y)^\top \right\|,\; \left\|\frac{1}{|S|}\sum_{(x,y)\in S} B(x,y)^\top B(x,y) \right\| \right) \leq \nu_B.
\end{align*}

Then we have for any $\epsilon >0$ and $t\geq 1$, if
\begin{align*}
|\Omega| \geq  ( 18 t \log d_B  ) \cdot ( \nu_B  + \|\ov{B}_S\|^2+ m_B  \epsilon )  / \epsilon^2 ,
\end{align*} 
with probability  at least $1-d_B^{-2t}$,
\begin{align*}
\left\| \ov{B}_S - \frac{1}{|\Omega| } \sum_{(x,y) \in \Omega} B(x,y) \right\|  \leq \epsilon.
 \end{align*}
\end{lemma}
\begin{proof}

Since each entry in $\Omega$ is sampled from $S$ uniformly and independently, we have
\begin{align*}
\E_{\Omega} \left[ \frac{1}{|\Omega| } \sum_{(x,y) \in \Omega} B(x,y) \right] = \frac{1}{|S| } \sum_{(x,y) \in S} B(x,y) .
 \end{align*}
Applying the matrix Bernstein inequality Theorem 6.1 in \cite{t11}, we prove this lemma.   
\end{proof}

\begin{lemma}\label{lemma:near_ground_pop_pop}
For sigmoid/tanh activation function,
\begin{align*}
 \|H(U,V)  - \nabla^2f_{\mathcal{D}}(U^*,V^*)\|  \lesssim (\|V - V^*\| + \|U - U^*\|),
\end{align*}
where $H(U,V) $ is defined as in Eq.~\eqref{eq:H_def}. 

For ReLU activation function,
\begin{align*} 
\|H(U,V)  - \nabla^2f_{\mathcal{D}}(U^*,V^*)\| 
  \lesssim  \left(\left(\frac{\|V - V^*\|}{\sigma_k(V^*)}\right)^{1/2} \|U^*\|+ \left(\frac{\|U - U^*\|}{\sigma_k(U^*)}\right)^{1/2} \|V^*\|\right) (\|U^*\| + \|V^*\|).
\end{align*}
\end{lemma}

\begin{proof}
We can bound each block, i.e.,
\begin{align}\label{eq:xx_expect} 
 & \E_{x,y} \left[ \phi'(u_i^{\top} x) \phi'(u_j^{\top} x) xx^\top \phi(v_i^{\top} y )\phi(v_j^{\top} y) -  \phi'(u_i^{*\top} x) \phi'(u_j^{*\top} x) xx^\top \phi(v_i^{*\top} y )\phi(v_j^{*\top} y) \right]  .
  \end{align}
\begin{align}\label{eq:xy_expect}
\E_{x,y} \left[ \phi'(u_i^{\top} x) \phi'(v_j^{\top} y) xy^\top \phi(v_i^{\top} y )\phi(u_j^{\top} x) -  \phi'(u_i^{*\top} x) \phi'(v_j^{*\top} y) xy^\top \phi(v_i^{*\top} y )\phi(u_j^{*\top} x) \right]  .
\end{align}

For smooth activations, the bound for Eq.~\eqref{eq:xx_expect} follows by combining Lemma~\ref{lemma:phiphi_expect} and  Lemma~\ref{lemma:smooth_expect_bound} and the bound for Eq.~\eqref{eq:xy_expect} follows
Lemma~\ref{lemma:phiphiprime_expect} and Lemma~\ref{lemma:phiphiprime_expect_x_sigmoid}.
For ReLU activation, the bound for Eq.~\eqref{eq:xx_expect} follows by combining Lemma~\ref{lemma:phiphi_expect}, Lemma~\ref{lemma:pwl_expect_bound} and the bound for Eq.~\eqref{eq:xy_expect} follows
Lemma~\ref{lemma:phiphiprime_expect} and Lemma~\ref{lemma:phiphiprime_expect_x_relu}. 
\end{proof}

\begin{lemma}\label{lemma:phiphi_expect}
\begin{align*}
\left\| \E_{y \sim \D_d} \left[(\phi(v_i^\top y) - \phi(v_i^{*\top}y))\phi(v_j^\top y) \right] \right\| 
 \lesssim \|V^*\|^p \|V - V^*\|.
\end{align*}
\end{lemma}
\begin{proof}
The proof follows the property of the activation function ($\phi(z)\leq |z|^p$) and H\"{o}lder's inequality. 
\end{proof}

\begin{lemma}\label{lemma:smooth_expect_bound}
When the activation function is smooth, we have 
\begin{align*}
\left\| \E_{x \sim \D_d} \left[ (\phi'(u_i^{\top} x) - \phi'(u_i^{*\top} x) ) \phi'(u_l^{\top} x) xx^\top \right] \right\| 
 \lesssim\|U - U^*\|.
\end{align*}
\end{lemma}

\begin{proof}
The proof directly follows Eq.~{(12)} in Lemma D.10 in \cite{zsjbd17}.
\end{proof}

\begin{lemma}\label{lemma:pwl_expect_bound}
When the activation function is piece-wise linear with $e$ turning points, we have 
\begin{align*}
\left\| \E_{x \sim \D_d} \left[ (\phi'(u_i^{\top} x) - \phi'(u_i^{*\top} x) ) \phi'(u_l^{\top} x) xx^\top \right] \right\| 
 \lesssim  (e \|U-U^*\|/\sigma_k(U^*))^{1/2 }.
\end{align*}
\end{lemma}

\begin{proof}
\begin{align*}
 \left\| \E_{x,y} \left[ (\phi'(u_i^{\top} x) - \phi'(u_i^{*\top} x) ) \phi'(u_l^{\top} x) xx^\top \right] \right\| \leq ~  \max_{\| a\|=1} \left( \E_{x\sim \D_d} \left[  |\phi'(u_i^{\top} x) - \phi'(u_i^{*\top} x) | \phi'(u_l^{\top} x) (x^\top a)^2 \right] \right) .
\end{align*}

Let $P$ be the orthogonal basis of $\text{span}( u_i, u_i^*, u_l)$. Without loss of generality, we assume $ u_i, u_i^*, u_l$ are independent, so $P = \text{span}( u_i, u_i^*, u_l)$ is $d$-by-3. Let $[ q_i\;  q_i^*\; q_l] = P^\top [ u_i \;  u_i^* \;  u_l] \in \mathbb{R}^{3 \times 3}$. Let $ a = P b+P_\perp  c$, where $P_\perp \in \R^{d\times (d-3)}$ is the complementary matrix of $P$. 
\begin{align}\label{eq:reduce_non_smooth}
& ~\E_{x\sim \D_d} \left[ |\phi'( u_i^\top x)-\phi'( u_i^{*\top} x)| |\phi'( u_l^\top x)|  (x^\top a)^2  \right] \notag \\
= & ~\E_{x\sim \D_d} \left[ |\phi'( u_i^\top x)-\phi'( u_i^{*\top} x)| |\phi'( u_l^\top x)|  (x^\top  (Pb + P_{\bot} c) )^2 \right] \notag \\
\lesssim & ~ \E_{x\sim \D_d} \left[   |\phi'( u_i^\top x)-\phi'( u_i^{*\top} x)| |\phi'( u_l^\top x)| \left(  (x^\top P  b)^2 +(x^\top P_\perp  c )^2 \right) \right] \notag \\
= & ~ \E_{x\sim \D_d} \left[   |\phi'( u_i^\top x)-\phi'( u_i^{*\top} x)| |\phi'( u_l^\top x)|    (x^\top P  b)^2  \right] \notag \\
+ & ~ \E_{x\sim \D_d} \left[   |\phi'( u_i^\top x)-\phi'( u_i^{*\top} x)| |\phi'( u_l^\top x)|   (x^\top P_\perp  c )^2  \right] \notag \\
= & ~ \E_{z\sim \D_3} \left[   |\phi'( q_i^\top z)-\phi'( q_i^{*\top} z)| |\phi'( q_l^\top z)|   (z^\top  b)^2  \right] \notag \\
+ & ~ \E_{z\sim \D_3, y\sim \D_{d-3}} \left[    |\phi'( q_i^\top z)-\phi'( q_i^{*\top} z)| |\phi'( q_l^\top z)|    (y^\top  c )^2  \right],
\end{align}
where the first step follows by $a = Pb + P_{\bot} c$, the last step follows by $(a+b)^2 \leq 2a^2 + 2b^2$.


We have $e$ exceptional points which have $\phi''(z) \neq 0$. Let these $e$ points be $p_1,p_2,\cdots,p_e$. Note that if $ q_i^\top  z$ and $ q_i^{*\top}  z$ are not separated by any of these exceptional points, i.e., there exists no $j\in[e]$ such that $ q_i^\top  z \leq p_j \leq  q_i^{*\top}  z$ or $ q_i^{*\top}  z \leq p_j \leq  q_i^\top  z $, then we have $\phi'( q_i^\top  z) = \phi'( q_i^{*\top}  z)$ since $\phi''(s)$ are zeros except for $\{p_j\}_{j=1,2,\cdots,e}$. So we consider the probability that $ q_i^\top  z, q_i^{*\top}  z$ are separated by any exception point. We use $\xi_j$ to denote the event that $ q_i^\top  z, q_i^{*\top}  z$ are separated by an exceptional point $p_j$. By union bound, $1- \sum_{j=1}^e\Pr [ \xi_j]$ is the probability that $ q_i^\top  z, q_i^{*\top}  z$ are not separated by any exceptional point. 
The first term of Equation~\eqref{eq:reduce_non_smooth} can be bounded as,
\begin{align*}
& ~ \E_{z\sim \D_3} \left[   |\phi'( q_i^\top  z)-\phi'( q_i^{*\top}  z)| |\phi'( q_l^\top  z)|  ( z^\top  b)^2 \right] \\
= & ~ \E_{z\sim \D_3} \left[\bone_{\cup_{j=1}^e \xi_j}|\phi'( q_i^\top  z) + \phi'( q_i^{*\top}  z)| |\phi'( q_l^\top  z)|  ( z^\top  b)^2 \right] \\
\leq & ~ \left( \E_{z\sim \D_3} \left[\bone_{\cup_{j=1}^e \xi_j} \right] \right)^{1/2 } \left(\E_{z\sim \D_3} \left[ (\phi'( q_i^\top  z) + \phi'( q_i^{*\top}  z))^2 \phi'( q_l^\top  z)^2  ( z^\top  b)^4 \right] \right)^{1/2} \\
\leq & ~ \left(\sum_{j=1}^e \Pr_{z\sim \D_3} [ \xi_j ] \right)^{1/2 } \left(\E_{z\sim \D_3} \left[(\phi'( q_i^\top  z)  + \phi'( q_i^{*\top}  z))^2 \phi'( q_l^\top  z)^2  ( z^\top  b)^4 \right] \right)^{1/2} \\
\lesssim & ~ \left(\sum_{j=1}^e \Pr_{z \sim \D_3}[\xi_j] \right)^{1/2 } \| b\|^2,
\end{align*}
where the first step follows by if $ q_i^\top  z, q_i^{*\top}  z$ are not separated by any exceptional point then $\phi'( q_i^\top  z) = \phi'( q_i^{*\top} z)$ and the last step follows by H\"{o}lder's inequality.

It remains to upper bound $\Pr_{z\sim \D_3}[\xi_j]$. First note that if $ q_i^\top  z, q_i^{*\top}  z$ are separated by an exceptional point, $p_j$, then $ | q_i^{*\top}  z - p_j| \leq | q_i^\top  z- q_i^{*\top}  z|  \leq  \| q_i- q_i^{*}\| \|  z\| $. Therefore, 
\begin{align*}
\Pr_{z\sim \D_3}[\xi_j] \leq \Pr_{z\sim \D_3} \left[ \frac{| q_i^\top  z-p_j|}{\| z\|} \leq \| q_i- q_i^*\| \right].
\end{align*}

 Note that $(\frac{ q_i^{*\top}  z}{\| z\| \| q_i^*\|}+1)/2$ follows Beta(1,1) distribution which is uniform distribution on $[0,1]$. 
\begin{align*}
& ~\Pr_{z\sim \D_3} \left[\frac{| q_i^{*\top}  z - p_j|}{\| z\|\| q_i^*\| }\leq \frac{\| q_i- q_i^*\|}{\| q_i^*\|} \right] 
\leq \Pr_{z\sim \D_3} \left[ \frac{| q_i^{*\top}  z|}{\| z\|\| q_i^*\| }\leq \frac{\| q_i- q_i^*\|}{\| q_i^*\|} \right]
\lesssim \frac{\| q_i- q_i^*\|}{\| q_i^*\|} 
\lesssim  \frac{ \|U-U^*\| }{\sigma_k(U^*)},
\end{align*}
where the first step is because we can view $\frac{ q_i^{*\top}  z}{\| z\|}$ and $\frac{p_j}{\| z\|}$ as two independent random variables: the former is about the direction of $ z$ and the later is related to the magnitude of $ z$. 
Thus, we have 
\begin{align}
\E_{z\in \D_3} [   |\phi'( q_i^\top  z)-\phi'( q_i^{*\top}  z)| |\phi'( q_l^\top  z)|  ( z^\top  b)^2] 
\lesssim  (e \|U-U^*\|/\sigma_k(U^*))^{1/2 } \| b\|^2 . \label{eq:decomp_off_first_part}
\end{align}

Similarly we have 
\begin{align} \label{eq:decomp_off_second_part}
\E_{z\sim \D_3, y\sim \D_{d-3}} \left[    |\phi'( q_i^\top z)-\phi'( q_i^{*\top} z)| |\phi'( q_l^\top z)|    (y^\top  c )^2  \right]
\lesssim  (e \|U-U^*\|/\sigma_k(U^*))^{1/2 }\| c\|^2. 
\end{align}
Finally combining Eq.~\eqref{eq:decomp_off_first_part} and Eq.~\eqref{eq:decomp_off_second_part} completes the proof.

\end{proof}

\begin{lemma}\label{lemma:phiphiprime_expect}
\begin{align*}
\left\| \E_{x \sim \D_d} \left[(\phi(u_j^\top x) - \phi(u_j^{*\top}x))\phi'(u_i^\top x) x \right] \right\| \lesssim  \|U - U^*\|.
\end{align*}
\end{lemma}
\begin{proof}
First, we can use the Lipschitz continuity of the activation function,
\begin{align*}
\left\| \E_{x \sim \D_d} \left[\phi(u_j^\top x) - \phi(u_j^{*\top}x)\phi'(u_i^\top x) x \right] \right\| 
& \leq  \max_{\|a\| = 1} \left\| \E_{x \sim \D_d} \left[|\phi(u_j^\top x) - \phi(u_j^{*\top}x)|\phi'(u_i^\top x) |x^\top a| \right] \right\|\\
 \leq  \max_{\|a\| = 1} L_{\phi} \left\| \E_{x \sim \D_d} \left[|u_j^\top x - u_j^{*\top}x|\phi'(u_i^\top x) |x^\top a| \right] \right\|,
\end{align*}
where $L_{\phi} \leq 1$ is the Lipschitz constant of $\phi$. 
Then the proof follows H\"{o}lder's inequality.
\end{proof}

\begin{lemma}\label{lemma:phiphiprime_expect_x_relu}
When the activation function is ReLU,
\begin{align*}
\left\| \E_{x \sim \D_d} \left[ \phi(u_j^{*\top}x)(\phi'(u_i^\top x)-\phi'(u_i^{*\top} x)) x \right] \right\| 
 \lesssim (\|U-U^*\|/\sigma_k(U^*))^{1/2 } \|u_j\|.
\end{align*}
\end{lemma}

\begin{proof}
\begin{align*}
\left\| \E_{x \sim \D_d} \left[ \phi(u_j^{*\top}x)(\phi'(u_i^\top x)-\phi'(u_i^{*\top} x)) x \right] \right\| 
 \leq  \max_{\|a\| = 1}  \E_{x \sim \D_d} \left[| \phi(u_j^{*\top}x)(\phi'(u_i^\top x)-\phi'(u_i^{*\top} x)) x^\top a| \right] .
\end{align*}
Similar to Lemma~\ref{lemma:pwl_expect_bound}, we can show that
\begin{align*}
 \max_{\|a\| = 1} \E_{x \sim \D_d} \left[| \phi(u_j^{*\top}x)(\phi'(u_i^\top x)-\phi'(u_i^{*\top} x)) x^\top a| \right] 
 \lesssim  (\|U-U^*\|/\sigma_k(U^*))^{1/2 } \|u_j\|.
\end{align*}
\end{proof}

\begin{lemma}\label{lemma:phiphiprime_expect_x_sigmoid}
When the activation function is sigmoid/tanh,
\begin{align*}
\left\| \E_{x \sim \D_d} \left[ \phi(u_j^{*\top}x)(\phi'(u_i^\top x)-\phi'(u_i^{*\top} x)) x \right] \right\| 
 \lesssim \|U-U^*\|.
\end{align*}
\end{lemma}
\begin{proof}
\begin{align*}
& ~ \left\| \E_{x \sim \D_d} \left[ \phi(u_j^{*\top}x)(\phi'(u_i^\top x)-\phi'(u_i^{*\top} x)) x \right] \right\|  \\
 \leq & ~ \max_{\|a\| = 1} \E_{x \sim \D_d} \left[| \phi(u_j^{*\top}x)(\phi'(u_i^\top x)-\phi'(u_i^{*\top} x)) x^\top a| \right]  \\
 \lesssim & ~ \max_{\|a\| = 1} \E_{x \sim \D_d} \left[| (u_i^\top x - u_i^{*\top} x) x^\top a| \right] \\
 \lesssim & ~ \|U-U^*\|.
\end{align*}
\end{proof}

\subsection{Local Linear Convergence}\label{app:linear_gd}

Given Theorem~\ref{thm:sigmoid_main}, we are now able to show local linear convergence of gradient descent for sigmoid and tanh activation function. 

\begin{theorem}[Restatement of Theorem~\ref{thm:grad_converge}]
Let $[U^c,V^c]$ be the parameters in the $c$-th iteration. Assuming $ \|U^c - U^*\| + \|V^c - V^*\| \lesssim  1 / ( \lambda^2\kappa )$,
then given a fresh sample set, $\Omega$, that is independent of $[U^c,V^c]$ and satisfies the conditions in Theorem~\ref{thm:sigmoid_main}, the next iterate using one step of gradient descent, i.e.,
$ [U^{c+1},V^{c+1}] = [U^c,V^c] - \eta \nabla f_{\Omega} (U^c,V^c),$
satisfies
\begin{align*}
\|U^{c+1} - U^*\|_F^2 + \|V^{c+1}-V^*\|_F^2 \leq (1-M_l/M_u) (\|U^c - U^*\|_F^2 + \|V^c-V^*\|_F^2)
\end{align*}
with probability $1-d^{-t}$, where $\eta = \Theta(1/M_u)$ is the step size and $M_l \gtrsim 1 / ( \lambda^2\kappa ) $ is the lower bound on the eigenvalues of the Hessian and $M_u \lesssim 1$ is the upper bound on the eigenvalues of the Hessian. 
\end{theorem}

\begin{proof}
In order to show the linear convergence of gradient descent, we first show that the Hessian along the line between $[U^c,V^c]$ and $[U^*,V^*]$ are positive definite w.h.p..  

The idea is essentially building a $d^{-1/2} \lambda^{-2} \kappa^{-1}$-net for the line between the current iterate and the optimal. In particular, we set $d^{1/2}$ points $\{[U^a,V^a]\}_{a=1,2,\cdots,d^{1/2}}$ that are equally distributed between $[U^c,V^c]$ and  $[U^*,V^*]$. Therefore, $\|U^{a+1} - U^a\|  + \|V^{a+1} - V^a\| \lesssim d^{-1/2} \lambda^{-2} \kappa^{-1}$

Using Lemma~\ref{lemma:local_uniform_bound}, we can show that for any $[U,V]$,  if there exists a value of $a$ such that  $\|U - U^a\|  + \|V - V^a\| \lesssim d^{-1/2} \lambda^{-2} \kappa^{-1}$,  then
$$ \|\nabla^2 f_{\Omega}(U,V) - \nabla^2 f_{\Omega} (U^a,V^a) \| \lesssim  \lambda^{-2} \kappa^{-1} .$$
Therefore, for every point $[U,V]$ in the line between  $[U^c,V^c]$ and  $[U^*,V^*]$, we can find a fixed point in $\{[U^a,V^a]\}_{a=1,2,\cdots,d^{1/2}}$, such that $\|U - U^a\|  + \|V - V^a\| \lesssim d^{-1/2} \lambda^{-2} \kappa^{-1}$. Now applying union bound for all $a$, we have that w.p.
$1-d^{-t}$, for every point $[U,V]$ in the line between  $[U^c,V^c]$ and  $[U^*,V^*]$, 
$$ M_l I \preceq \nabla^2 f_{\Omega}(U,V) \preceq  M_u,$$
where  $M_l = \Omega(\lambda^{-2} \kappa^{-1})$ and $M_u = O(1)$. Note that the upper bound of the Hessian is due to the fact that $\phi$ and $\phi'$ are bounded.

Given the positive definiteness of the Hessian along the line between the current iterate and the optimal, we are ready to show the linear convergence. First we set the stepsize for the gradient descent update as $\eta = 1/M_u$ and use notation $W := [U,V]$ to simplify the writing.
\begin{align*}
& ~\| W^{c+1} - W^*\|_F^2 \\
= &~ \|W^c  - \eta \nabla f_\Omega(W^c) - W^* \|_F^2 \\
= &~ \| W^c - W^*\|_F^2 - 2\eta \langle \nabla {f}_\Omega(W^c), (W^c-W^*) \rangle + \eta^2 \|\nabla {f}_{\Omega}(W^c)\|_F^2 
\end{align*}
Note that
\begin{align*}
\nabla f_\Omega(W^c) =  \left( \int_{0}^1 \nabla^2 f_\Omega( W^* + \xi (W^c-W^*) ) d\xi \right) \text{vec}(W^c - W^*).
\end{align*}
Define $H \in \dR^{(2kd)\times (2kd)}$,
\begin{align*}
H = \left( \int_{0}^1 \nabla^2 f_\Omega( W^* + \xi (W^c-W^*) ) d\xi \right).
\end{align*}
By the result provided above, we have
\begin{equation}\label{eq:smooth_sc_line}
M_l I \preceq H \preceq M_u I.
\end{equation} 
Now we upper bound the norm of the gradient,
\begin{align*}
\|\nabla f_\Omega(W^c)\|_F^2 =  \langle H \text{vec}(W^c-W^*), H  \text{vec}(W^c-W^*)\rangle 
\leq  M_u \langle  \text{vec}(W^c-W^*), H \text{vec}(W^c-W^*) \rangle .
\end{align*}

Therefore, 
\begin{align*}
& ~\|{W}^{c+1} - W^*\|_F^2  \\
\leq & ~ \| W^c - W^*\|_F^2 - (-\eta^2 M_u + 2\eta)\langle  \text{vec}(W^c-W^*), H  \text{vec}(W^c-W^*) \rangle\\
 \leq  & ~ \| W^c - W^*\|_F^2 - (-\eta^2 M_u + 2\eta)M_l \| W^c-W^*\|_F^2 \\
  =  & ~ \| W^c - W^*\|_F^2 - \frac{M_l}{M_u} \| W^c-W^*\|_F^2 \\
    \leq  & ~ (1- \frac{M_u}{M_l} ) \| W^c-W^*\|_F^2 
\end{align*}
\end{proof}

\begin{lemma}\label{lemma:local_uniform_bound}
Let the activation function be tan/sigmoid. For given $U^a,V^a$ and $r>0 $,  if 
\begin{align*} n_1 \gtrsim \epsilon^{-2} t d \log^2 d ,  \quad n_2 \gtrsim \epsilon^{-2} t \log d , \quad |\Omega| \gtrsim \epsilon^{-2} t d \log^2 d,
\end{align*}
then with probability $1-d^{-t}$,
\begin{align*}
\sup_{ \|U - U^a\| + \|V-V^a\|  \leq r }\|\nabla^2 f_{\Omega}(U,V) - \nabla^2 f_{\Omega} (U^a,V^a) \| \lesssim d^{1/2}\cdot r
\end{align*}
\end{lemma}
\begin{proof}
We consider each block of the Hessian as defined in Eq~\eqref{eq:Hessian_blocks}. 
In particular, we show that if 
\begin{align*} n_1 \gtrsim \epsilon^{-2} t d \log^2 d ,  \quad n_2 \gtrsim \epsilon^{-2} t \log d , \quad |\Omega| \gtrsim \epsilon^{-2} t d \log^2 d,
\end{align*}
then with probability $1-d^{-t}$,
\begin{align*}
& ~ \bigg\|  \frac{1}{|\Omega|} \sum_{(x,y) \in \Omega} \left[( \phi'(u_i^{\top} x) \phi'(u_j^{\top} x)  \phi(v_i^{\top} y )\phi(v_j^{\top} y) -  \phi'(u_i^{a\top} x) \phi'(u_j^{a\top} x)  \phi(v_i^{a\top} y )\phi(v_j^{a\top} y) )xx^\top\right] \bigg\| \\
\lesssim & ~(\|u_i - u_i^a\| + \|u_j - u_j^a\|+\|v_i - v_i^a\|+\|v_j - v_j^a\| ) d^{1/2}  ; \\
& ~ \text{by~Lemma~\ref{lemma:diag_emp_pop_sym_dep}}\\
& ~ \bigg\| \frac{1}{|\Omega|} \sum_{(x,y) \in \Omega}  \left[ \left( \phi(U^\top x)^\top \phi(V^\top y) - \phi(U^{*\top} x)^\top \phi(V^{*\top} y)  \right) \phi''(u_i^\top x) \phi(v_i^\top y)xx^\top \right. \\
& \left.  - \left( \phi(U^{a\top} x)^\top \phi(V^{a\top} y) - \phi(U^{*\top} x)^\top \phi(V^{*\top} y)  \right) \phi''(u_i^{a\top} x) \phi(v_i^{a\top} y)xx^\top
    \right] \bigg\| \\
 \lesssim & ~(\| U - U^{a} \|+\| V - V^{a} \|) d^{1/2} \\
 & ~ \text{by~Lemma~\ref{lemma:offdiag_emp_pop_sym_dep}}\\
& ~ \bigg\|  \frac{1}{|\Omega|} \sum_{(x,y) \in \Omega}  \left[ \left( \phi'(u_i^{\top} x) \phi'(v_j^{\top} y)  \phi(v_i^{\top} y )\phi(u_j^{\top} x)  -  \phi'(u_i^{a\top} x) \phi'(v_j^{a\top} y)  \phi(v_i^{a\top} y )\phi(u_j^{a\top} x)\right) xy^\top \right] \bigg\| \\
 \lesssim & ~ (\|u_i - u_i^a\| + \|u_j - u_j^a\|+\|v_i - v_i^a\|+\|v_j - v_j^a\| ) d^{1/2} \\
 & ~ \text{by~Lemma~\ref{lemma:diag_emp_pop_asym_dep}} \\
&\bigg\| \frac{1}{|\Omega|} \sum_{(x,y) \in \Omega}  \left[ \left( \phi(U^\top x)^\top \phi(V^\top y) - \phi(U^{*\top} x)^\top \phi(V^{*\top} y)  \right) \phi'(u_i^\top x) \phi'(v_i^\top y)xy^\top  \right.\\
&  -\left( \phi(U^{a\top} x) ^\top \phi(V^{a\top} y) - \phi(U^{*\top} x)^\top \phi(V^{*\top} y)  \right) \phi'(u_i^{a\top} x) \phi'(v_i^{a \top }y)xy^\top
 ] \bigg\| \\
 \lesssim & ~(\| U - U^{a} \|+\| V - V^{a} \|) d^{1/2} \\
 & ~ \text{by~Lemma~\ref{lemma:offdiag_emp_pop_asym_dep}}
\end{align*}
\end{proof}

\begin{lemma}\label{lemma:diag_emp_pop_sym_dep}
If 
\begin{align*}
 n_1 \gtrsim \epsilon^{-2} t d \log^2 d ,  \quad n_2 \gtrsim \epsilon^{-2} t \log d , \quad |\Omega| \gtrsim \epsilon^{-2} t d \log^2 d,
 \end{align*}
then with probability at least $1-d^{-t}$,
\begin{align*}
& ~ \left\|  \frac{1}{|\Omega|} \sum_{(x,y) \in \Omega} \left[( \phi'(u_i^{\top} x) \phi'(u_j^{\top} x)  \phi(v_i^{\top} y )\phi(v_j^{\top} y) -  \phi'(u_i^{a\top} x) \phi'(u_j^{a\top} x)  \phi(v_i^{a\top} y )\phi(v_j^{a\top} y) )xx^\top\right] \right\| \\
\lesssim & ~(\|u_i - u_i^a\| + \|u_j - u_j^a\|+\|v_i - v_i^a\|+\|v_j - v_j^a\| ) d^{1/2}
\end{align*}
\end{lemma}

\begin{proof}
Note that 
\begin{align}
 & \phi'(u_i^{\top} x) \phi'(u_j^{\top} x)  \phi(v_i^{\top} y )\phi(v_j^{\top} y) -  \phi'(u_i^{a\top} x) \phi'(u_j^{a\top} x)  \phi(v_i^{a\top} y )\phi(v_j^{a\top} y) \notag \\
= &  \phi'(u_i^{\top} x) \phi'(u_j^{\top} x)  \phi(v_i^{\top} y )\phi(v_j^{\top} y) -  \phi'(u_i^{a\top} x) \phi'(u_j^{\top} x)  \phi(v_i^{\top} y )\phi(v_j^{\top} y) \notag \\
 & + \phi'(u_i^{a\top} x) \phi'(u_j^{\top} x)  \phi(v_i^{\top} y )\phi(v_j^{\top} y)  - \phi'(u_i^{a\top} x) \phi'(u_j^{a\top} x)  \phi(v_i^{\top} y )\phi(v_j^{\top} y) \notag \\
 & + \phi'(u_i^{a\top} x) \phi'(u_j^{a\top} x)  \phi(v_i^{\top} y )\phi(v_j^{\top} y) -  \phi'(u_i^{a\top} x) \phi'(u_j^{a\top} x)  \phi(v_i^{a\top} y )\phi(v_j^{\top} y) \notag \\
 & + \phi'(u_i^{a\top} x) \phi'(u_j^{a\top} x)  \phi(v_i^{a\top} y )\phi(v_j^{\top} y) - \phi'(u_i^{a\top} x) \phi'(u_j^{a\top} x)  \phi(v_i^{a\top} y )\phi(v_j^{a\top} y) 
 \label{eq:split_eq}
\end{align}
Let's consider the first term in the above formula. The other terms are similar.
\begin{align*}
& ~ \left\|  \frac{1}{|\Omega|} \sum_{(x,y) \in \Omega} \left[( \phi'(u_i^{\top} x)-  \phi'(u_i^{a\top} x)) \phi'(u_j^{\top} x)  \phi(v_i^{\top} y )\phi(v_j^{\top} y) xx^\top\right] \right\| \\
 \leq & ~ \left\|  \frac{1}{|\Omega|} \sum_{(x,y) \in \Omega} \left[\| u_i - u_i^{a} \| \|x\|  xx^\top\right] \right\| \\
\end{align*}
which is because both $\phi'(\cdot)$ and $\phi(\cdot)$ are bounded and Lipschitz continuous. 
Applying the unbounded matrix Bernstein Inequality Lemma~\ref{lem:modified_bernstein_non_zero}, we can bound 
\begin{align*}
 \left\|  \frac{1}{|\Omega|} \sum_{(x,y) \in \Omega} \left[\| u_i - u_i^{a} \| \|x\|  xx^\top\right] \right\|  \lesssim \| u_i - u_i^{a} \| d^{1/2}
\end{align*}
Since both $\phi'(\cdot)$ and $\phi(\cdot)$ are bounded and Lipschitz continuous, we can easily extend the above inequality to other cases and finish the proof. 
\end{proof}

\begin{lemma}\label{lemma:offdiag_emp_pop_sym_dep}
If 
\begin{align*}
 n_1 \gtrsim \epsilon^{-2} t d \log^2 d ,  \quad n_2 \gtrsim \epsilon^{-2} t \log d , \quad |\Omega| \gtrsim \epsilon^{-2} t d \log^2 d,
 \end{align*}
then with probability at least $1-d^{-t}$,
\begin{align*}
& \left\| \frac{1}{|\Omega|} \sum_{(x,y) \in \Omega}  \left[ \left( \phi(U^\top x)^\top \phi(V^\top y) - \phi(U^{*\top} x)^\top \phi(V^{*\top} y)  \right) \phi''(u_i^\top x) \phi(v_i^\top y)xx^\top \right. \right. \\
& \left.\left. - \left( \phi(U^{a\top} x)^\top \phi(V^{a\top} y) - \phi(U^{*\top} x)^\top \phi(V^{*\top} y)  \right) \phi''(u_i^{a\top} x) \phi(v_i^{a\top} y)xx^\top
    \right] \right\| \\
 \lesssim  & ~ (\| U - U^{a} \|+\| V - V^{a} \|) d^{1/2}
\end{align*}
\end{lemma}
\begin{proof}
Since for sigmoid/tanh, $\phi, \phi', \phi''$ are all Lipschitz continuous and bounded, the proof of this lemma resembles the proof for Lemma~\ref{lemma:diag_emp_pop_sym_dep}.
\end{proof}

\begin{lemma}\label{lemma:diag_emp_pop_asym_dep}
If 
\begin{align*}
 n_1 \gtrsim \epsilon^{-2} t d \log^2 d ,  \quad n_2 \gtrsim \epsilon^{-2} t \log d , \quad |\Omega| \gtrsim \epsilon^{-2} t d \log^2 d,
 \end{align*}
then with probability at least $1-d^{-t}$,
\begin{align*}
& ~ \left\|  \frac{1}{|\Omega|} \sum_{(x,y) \in \Omega}  \left[ \left( \phi'(u_i^{\top} x) \phi'(v_j^{\top} y)  \phi(v_i^{\top} y )\phi(u_j^{\top} x)  -  \phi'(u_i^{a\top} x) \phi'(v_j^{a\top} y)  \phi(v_i^{a\top} y )\phi(u_j^{a\top} x)\right) xy^\top \right] \right\|  \\
 \lesssim  & ~ (\|u_i - u_i^a\| + \|u_j - u_j^a\|+\|v_i - v_i^a\|+\|v_j - v_j^a\| )d^{1/2}
\end{align*}
\end{lemma}
\begin{proof}
Do the similar splits as Eq.~\eqref{eq:split_eq} and let's consider the following case, 
\begin{align*}
 \left\|  \frac{1}{|\Omega|} \sum_{(x,y) \in \Omega}  \left[ \left( \phi'(u_i^{\top} x) - \phi'(u_i^{a\top} x) \right) \phi'(v_j^{\top} y)  \phi(v_i^{\top} y )\phi(u_j^{\top} x)   xy^\top \right] \right\|.
 \end{align*}
Setting $M(x) = \left( \phi'(u_i^{\top} x) - \phi'(u_i^{a\top} x) \right)\phi(u_j^{\top} x)   x $, $N(y) =  \phi'(v_j^{\top} y)  \phi(v_i^{\top} y ) y^\top$ and using the fact that $\| \phi'(u_i^{\top} x) - \phi'(u_i^{a\top} x) \| \leq \| u_i - u_i^a\| \|x\|$, we can follow the proof of Lemma~\ref{lemma:diag_emp_pop_asym} to show if 
\begin{align*}
 n_1 \gtrsim \epsilon^{-2} t d \log^2 d ,  \quad n_2 \gtrsim \epsilon^{-2} t \log d , \quad |\Omega| \gtrsim \epsilon^{-2} t d \log^2 d,
 \end{align*}
then with probability at least $1-d^{-t}$,
\begin{align*}
 ~ \left\|  \frac{1}{|\Omega|} \sum_{(x,y) \in \Omega}  \left[ \left( \phi'(u_i^{\top} x) - \phi'(u_i^{a\top} x) \right) \phi'(v_j^{\top} y)  \phi(v_i^{\top} y )\phi(u_j^{\top} x)   xy^\top \right] \right\|  \leq \|u_i - u_i^a\| d^{1/2}
\end{align*}
\end{proof}

\begin{lemma}\label{lemma:offdiag_emp_pop_asym_dep}
If 
\begin{align*}
 n_1 \gtrsim \epsilon^{-2} t d \log^2 d ,  \quad n_2 \gtrsim \epsilon^{-2} t \log d , \quad |\Omega| \gtrsim \epsilon^{-2} t d \log^2 d,
 \end{align*}
then with probability at least $1-d^{-t}$,
\begin{align*}
&\left\| \frac{1}{|\Omega|} \sum_{(x,y) \in \Omega}  \left[ \left( \phi(U^\top x)^\top \phi(V^\top y) - \phi(U^{*\top} x)^\top \phi(V^{*\top} y)  \right) \phi'(u_i^\top x) \phi'(v_i^\top y)xy^\top  \right.\right.\\
& \left. \left. -\left( \phi(U^{a\top} x) ^\top \phi(V^{a\top} y) - \phi(U^{*\top} x)^\top \phi(V^{*\top} y)  \right) \phi'(u_i^{a\top} x) \phi'(v_i^{a \top }y)xy^\top
 \right] \right\| \\
 \lesssim  & ~ (\| U - U^{a} \|+\| V - V^{a} \|) d^{1/2}
\end{align*}
\end{lemma}
\begin{proof}
Since for sigmoid/tanh, $\phi, \phi', \phi''$ are all Lipschitz continuous and bounded, the proof of this lemma resembles the proof for Lemma~\ref{lemma:diag_emp_pop_asym_dep}.
\end{proof}









\end{document}